\documentclass[11pt]{article}
\usepackage[preprint]{acl}
\usepackage[utf8]{inputenc}
\usepackage[T1]{fontenc}
\usepackage{hyperref}
\usepackage{times}
\usepackage{url}
\usepackage{booktabs}
\usepackage{amsfonts}
\usepackage{amsmath}
\usepackage{amssymb}
\usepackage{amsthm}
\usepackage{nicefrac}
\usepackage{xcolor}
\usepackage{graphicx}
\usepackage{algorithm}
\usepackage{algpseudocode}
\usepackage{multirow}
\usepackage{stfloats}
\usepackage{placeins}

\theoremstyle{plain}
\newtheorem{theorem}{Theorem}
\newtheorem{proposition}[theorem]{Proposition}
\newtheorem{lemma}[theorem]{Lemma}

\theoremstyle{definition}

\newtheorem{assumption}[theorem]{Assumption}

\theoremstyle{remark}
\newtheorem{remark}[theorem]{Remark}

\newcommand{\E}{\mathbb{E}}
\newcommand{\PP}{\mathbb{P}}
\newcommand{\D}{\mathcal{D}}
\newcommand{\X}{\mathcal{X}}
\newcommand{\Y}{\mathcal{Y}}

\newcommand{\Cov}{\mathrm{Cov}}
\newcommand{\mask}{\mathsf{m}}
\newcommand{\G}{\mathcal{G}}
\newcommand{\GradRef}{\textsc{GradRef}}

\newsavebox{\CSPtablebox}
\newcommand{\CSPfitTable}[2]{%
  \sbox{\CSPtablebox}{#2}%
  \ifdim\wd\CSPtablebox>#1\relax
    \resizebox{#1}{!}{\usebox{\CSPtablebox}}%
  \else
    \usebox{\CSPtablebox}%
  \fi
}

\title{Candidate-Label Saliency for Efficient Conformal Pruning of Text Classifiers}

\author{
 \textbf{Ibne Farabi Shihab\textsuperscript{1 ,}\thanks{Equal contribution.}},
 \textbf{Joyanta Jyoti Mondal\textsuperscript{2, }\footnotemark[1]},
 \textbf{Adria Binte Habib\textsuperscript{3, }\footnotemark[1]}
\\
\\
 \textsuperscript{1}Department of Computer Science, Iowa State University, USA\\
 \textsuperscript{2}Department of Computer and Information Sciences, University of Delaware, USA\\
 \textsuperscript{3}Department of Computer Science and Engineering, BRAC University, Bangladesh
\\
 \small{
   \textbf{Correspondence:} \href{mailto:ibnfarabishihab@gmail.com}{ibnfarabishihab@gmail.com}
 }
}

\begin{document}

\maketitle

\begin{abstract}
Recalibration can preserve conformal coverage after pruning while enlarging prediction sets and leaving more alternatives for downstream decisions. Directly searching for deletions that keep sets small is costly. We study whether a small candidate-gradient budget can approximate this search at lower cost. Candidate-Saliency Pruning (CSP) estimates the squared first-order change in two plausible labels' scores when each parameter is deleted. CSP averages these changes, combines the resulting importance with a base pruner, and selects the mixing weight under validation accuracy and coverage requirements. Independent split-conformal calibration supplies marginal coverage under exchangeability. We evaluate CSP on three classification tasks, two base pruners, and three sparsity levels against a squared true-label control that separates candidate choice from gradient squaring, under fixed recovery updates, equal total time, and restricted access to the training corpus. Before recovery, CSP reduces mean set size by 13.2\% to 19.2\% against the reference pruner and by 7.7\% to 13.2\% against the control in all six encoder cells at 50\% sparsity. The advantage over the control shrinks with recovery to practical equivalence in all six cells at 1,000 updates, and at equal total time the cheaper reference pruner is practically equivalent in 18 of 18 settings. When recovery is limited to the retained pruning pool, CSP keeps a meaningful advantage in all six cells. Its value therefore depends on the time and data available for recovery.
\end{abstract}

\section{Introduction}
\label{sec:intro}

A compressed classifier can retain the correct top prediction yet leave too many plausible alternatives for a downstream decision. Conformal prediction makes this distinction explicit: \emph{coverage} measures whether a prediction set contains the true label, and \emph{set size} measures how many alternatives remain. In triage or assisted classification, a useful classifier must preserve task accuracy while producing sets with adequate coverage and a manageable number of candidates.

Independent split-conformal calibration provides marginal coverage for any fixed pruned model under exchangeability \citep{vovk2005algorithmic,lei2018distribution,angelopoulos2023gentle} but does not ensure small prediction sets. \citet{pmlr-v179-zhao22a} target set efficiency by evaluating hidden-unit deletions directly, but these evaluations become costly at transformer scale. In this paper, we study whether a small candidate-gradient budget can approximate useful deletion decisions at lower cost.

We introduce \emph{Candidate-Saliency Pruning} (CSP). For each pruning example, CSP averages the squared coordinate-level score displacements of the two labels most compatible under the dense model, then combines this local sensitivity measure with Wanda \citep{sun2024wanda} or SparseGPT \citep{frantar2023sparsegpt}. We evaluate the approximation by comparing its final set size with direct search and measuring the full cost of producing the selected sparse classifier.

Recovery complicates this comparison. A cheaper pruning method can leave time for additional training updates, so an advantage at a fixed update count may disappear at equal total time. We compare absolute true-label gradients, squared true-label gradients, and candidate gradients at 0, 100, 250, 500, and 1,000 recovery updates. In this manner, we separate candidate choice from gradient squaring. We also test checkpoint transfer without the original task-training corpus, allowing recovery only on the retained pruning examples.

Our primary evaluation uses fine-tuned RoBERTa~\citep{liu2019roberta} on DBpedia-14~\citep{zhang2015textclassification}, Banking77~\citep{casanueva2020banking77}, and CLINC150~\citep{larson2019clinc150}, and fine-tuned and zero-shot Qwen classifiers~\citep{yang2024qwen25} extend it to decoder classification and larger prediction sets. Top-2 is the default candidate rule; we evaluate Threshold-2 separately.

We make three contributions:
\begin{itemize}
\item A fixed-budget candidate-saliency rule that cuts mean set size by 13.2\% to 19.2\% against the reference at 50\% sparsity before recovery (Section~\ref{sec:method}).
\item Sufficient conditions for set-size stability (Section~\ref{sec:theory}).
\item Recovery comparisons showing practical equivalence with the squared true-label control after 1,000 updates and with the reference at equal time, but a retained gain in 6/6 cells when recovery uses only the pruning pool (Section~\ref{sec:experiments}).
\end{itemize}

Figure~\ref{fig:overview} summarizes the CSP pipeline.

\begin{figure*}[t]
\centering
\includegraphics[width=0.85\textwidth]{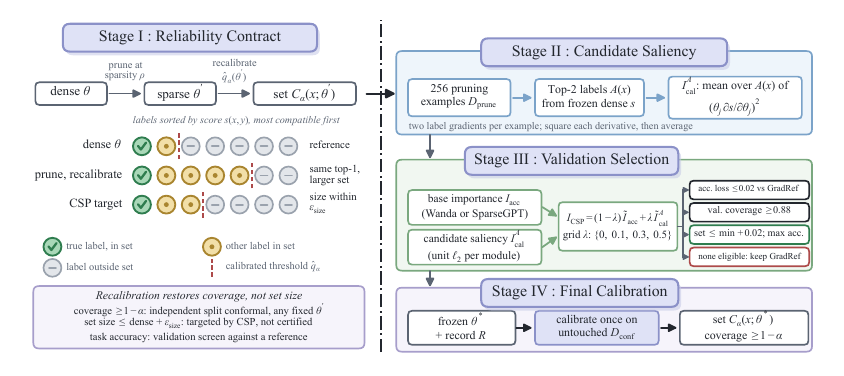}
\caption{Candidate-Saliency Pruning (CSP). I: independent recalibration restores coverage after pruning but can enlarge prediction sets (schematic). CSP targets set size, and validation screens for accuracy. II: candidate saliency averages squared coordinate displacements over each pruning example's Top-2 dense labels. III: validation selection mixes candidate and base-pruner importance over a coefficient grid, applies accuracy, coverage, and set-size rules, and falls back to the reference if needed. IV: the selected model and scorer are frozen before a single calibration on untouched data.}
\label{fig:overview}
\end{figure*}

\section{Related Work}
\label{sec:related}

\paragraph{Direct conformal-efficiency optimization.}
The closest predecessor prunes neural networks for efficiency under inductive conformal prediction \citep{pmlr-v179-zhao22a,zhao2025pruningcnn}, evaluating deletion effects directly and calibrating independently. CSP approximates the same objective with coordinate-level saliency under a fixed gradient budget; Appendix~\ref{app:related_extended} describes our matched adaptation of the predecessor's criterion.

\paragraph{Pruning and parameter importance.}
One-shot methods include magnitude pruning \citep{han2015learning}, SparseGPT \citep{frantar2023sparsegpt}, Wanda \citep{sun2024wanda}, and structured LLM-Pruner \citep{ma2023llm}, and scaling analyses examine sparse-model performance \citep{frantar2025scaling}. Our comparisons hold pruning support and task-data access fixed and account for recovery separately. Appendix~\ref{app:related_extended} relates our importance controls to continual learning.

\paragraph{Calibration and conformal guarantees.}
Temperature and Platt scaling adjust confidence after training \citep{guo2017calibration,platt1999probabilistic}, and sparse-subnetwork studies examine calibration after pruning \citep{bayesian_lth_2024,open_lth_2024}; compression also changes behavior beyond accuracy \citep{hooker2019forget,hong2024decoding}. Conformal methods address generation, factuality, and information-relative certificates \citep{quach2024cp,mohri2024factuality,shihab2025infolift}, as well as score-function and answer-choice optimization \citep{vishwakarma2025prune}. CSP's coverage guarantee follows standard split-conformal theory; Appendix~\ref{app:related_extended} discusses shift-aware and risk-controlling alternatives.

\section{Problem Setting}
\label{sec:setting}

\subsection{Prediction Sets and Task Utility}
\label{sec:reframe}

Let \(\X\) be the input space and \(\Y\) a finite label space with \(|\Y|=K\). The nonconformity score \(s(x,y;\theta)\) evaluates a complete candidate label; smaller scores indicate greater compatibility. For a normalized classifier \(f_\theta:\X\rightarrow\Delta(\Y)\), a common choice is
\begin{equation}
s(x,y;\theta)=-\log p_\theta(y\mid x).
\label{eq:nll_score}
\end{equation}
The primary encoder and fine-tuned decoder use these normalized class probabilities. The zero-shot decoder instead scores each label by the negative teacher-forced log-likelihood of a fixed verbalizer, which need not sum to one across labels (Appendix~\ref{app:decoder_scorer}). Every scorer evaluates complete labels and is frozen before final calibration.

We use five disjoint data splits for pruning, validation selection, final calibration, and testing. Dense task fitting and unrestricted recovery use a separate task-training pool. The pruning split \(\D_{\mathrm{prune}}\) (256 examples) supplies gradients and pruning statistics. Validation-calibration \(\D_{\mathrm{val\text{-}cal}}\) (512 examples) supplies provisional quantiles and the preliminary dense threshold for Threshold-2, and validation-evaluation \(\D_{\mathrm{val\text{-}eval}}\) (2,048 examples) selects \(\lambda\). We access the independent conformal split \(\D_{\mathrm{conf}}=\{(X_i,Y_i)\}_{i=1}^n\) only after freezing the sparse model, and the test split \(\D_{\mathrm{test}}\) only for evaluation.

For a fixed \(\theta\) and target miscoverage \(\alpha\), define
\begin{equation}
k=\left\lceil(n+1)(1-\alpha)\right\rceil.
\label{eq:k_order}
\end{equation}
When \(k\leq n\), \(\hat q_\alpha(\theta)\) is the \(k\)-th smallest calibration score; when \(k>n\), we set it to \(+\infty\). The prediction set is
\begin{equation}
C_\alpha(x;\theta)
=\{y\in\Y:s(x,y;\theta)\leq\hat q_\alpha(\theta)\}.
\label{eq:conformal_set}
\end{equation}

We evaluate coverage, set efficiency, and task utility separately. Marginal coverage, \(\PP\{Y\in C_\alpha(X;\theta)\}\), is the validity constraint; expected set size, \(\bar C_P(\theta)=\E_{\D_{\mathrm{conf}},X}|C_\alpha(X;\theta)|\), is the primary efficiency outcome; accuracy measures task utility, and expected calibration error (ECE) is a complementary diagnostic. For a pruned model \(\theta'=\mathcal P(\theta)\), we express the desired \emph{reliability contract} with efficiency slack \(\varepsilon_{\mathrm{size}}\) as
\begin{equation}
\begin{aligned}
\PP\{Y\in C_\alpha(X;\theta')\}&\geq1-\alpha,\\
\bar C_P(\theta')&\leq\bar C_P(\theta)+\varepsilon_{\mathrm{size}}.
\end{aligned}
\label{eq:contract}
\end{equation}
Independent split conformal supplies marginal coverage for any fixed \(\theta'\). CSP targets set size, while validation screens candidates for task accuracy. The slack specifies a desired outcome rather than an experimental certificate, the validation screen does not guarantee test accuracy, and marginal coverage averages over the calibration sample and a new test example rather than conditioning on one realized sample.

\subsection{Why Recalibration Alone Is Not Enough}
\label{sec:prune_then_recal_fails}

Under exchangeability, independent final calibration makes prune-then-recalibrate valid, yet the resulting sets may still be uninformative: Top-1 accuracy depends on the lowest-scoring label, whereas set size depends on every label whose score lies below the calibration threshold. Pruning can therefore preserve the top prediction while changing many other labels' membership in the set.

Temperature scaling adjusts a fixed model's confidence \citep{guo2017calibration} but cannot reconstruct label-specific information removed by pruning. A strictly increasing transformation applied identically to candidate and calibration scores, such as \(s/T\) with \(T>0\), leaves conformal sets unchanged. Temperature scaling inside normalized class probabilities need not have this invariance, because the normalizer depends on the input; recovering the dense representation requires a different operation.

A useful feasibility check follows directly from set membership. For any prediction set \(C\subseteq\Y\),
\[
\E|C|\leq\PP(Y\in C)+(K-1).
\]
Thus, a method that always returns all \(K\) labels must have coverage one. This relation provides a joint consistency check on set size and coverage.

\section{Method: Candidate-Saliency Pruning}
\label{sec:method}

CSP estimates the local effect of deleting weights from candidate-label score sensitivity, changing only the importance used to select retained coordinates; the base pruner determines eligible modules, sparsity allocation, and reconstruction updates. Let \(I_{\mathrm{acc}}(\theta_j)\) be the base pruner's coordinate importance: magnitude \citep{han2015learning}, Wanda \citep{sun2024wanda}, or SparseGPT reconstruction \citep{frantar2023sparsegpt}.

\subsection{A Fixed Candidate-Gradient Budget}

Top-2 forms \(A(x)\) from the two lowest-scoring labels under the frozen dense model \(s(x,y;\theta)\), breaking ties by fixed label order. We construct these candidate sets once and reuse them across mixing coefficients. An encoder obtains all class scores in one forward pass, whereas a verbalizer decoder must score complete candidate sequences. Each of the 256 pruning examples receives exactly two label gradients:

\begin{equation}
\begin{aligned}
I_{\mathrm{cal}}^{A}(\theta_j)
&=\frac{1}{|\D_{\mathrm{prune}}|}
\sum_{x\in\D_{\mathrm{prune}}}\overline G_j^{A}(x),\\
\overline G_j^{A}(x)
&=\frac{1}{|A(x)|}\sum_{y\in A(x)}
\left(\theta_j\frac{\partial s(x,y;\theta)}{\partial\theta_j}\right)^2.
\end{aligned}
\label{eq:candidate_importance}
\end{equation}
Setting \(\theta_j\) to zero changes a label's score by approximately \(-\theta_j\,\partial s(x,y;\theta)/\partial\theta_j\). Equation~\ref{eq:candidate_importance} squares this estimated change for each candidate label and averages over labels and examples. Squaring before averaging prevents derivatives with opposite signs from cancelling.

We normalize each nonzero base and saliency vector to unit \(\ell_2\) norm within a module (a zero vector stays zero) and combine them as
\begin{equation}
I_{\mathrm{CSP}}(\theta_j)
=(1-\lambda)\widetilde I_{\mathrm{acc}}(\theta_j)
+\lambda\widetilde I_{\mathrm{cal}}^{A}(\theta_j).
\label{eq:cpp_importance}
\end{equation}
We prune attention and multilayer perceptron (MLP) matrices under the base allocation; norm parameters, embeddings, tied output heads, positional mechanisms, and key-value (KV) cache state remain dense.

Two alternatives change the candidate labels but keep two gradients per example. Threshold-2 retains the two labels nearest \(\tilde q\), the dense true-label quantile estimated on validation-calibration, and Random-2 samples two labels without replacement. Because validation-calibration is reused during construction and selection, its provisional quantiles carry no independent-calibration guarantee; only the final independent quantile determines reported prediction sets.

\subsection{Validation Selection and Accuracy Matching}
\label{sec:selection}

We call the selected absolute true-label gradient reference \GradRef{}. Its candidates combine Wanda or SparseGPT importance with \(\E_i|\theta_j g_{ij}|\), inspired by single-shot network pruning (SNIP) \citep{lee2019snip}, where \(g_{ij}\) is the per-example true-label score gradient. CSP and \GradRef{} share \(\Lambda=\{0,0.1,0.3,0.5\}\), eligible weights, sparsity allocation, and data. Both retain SparseGPT reconstruction updates.

We select \GradRef{} by screening its candidates against the unaugmented base pruner (Appendix~\ref{app:experimental}). If no candidate qualifies, \GradRef{} retains the base pruner. The label identifies this validation-selected comparator throughout the paper. For each CSP candidate, the paired validation accuracy loss is \(L_i=\mathbf1\{\hat y_i^{\rm ref}=y_i\}-\mathbf1\{\hat y_i^{\rm CSP}=y_i\}\). On \(m=2048\) validation examples, let \(\hat\ell\) and \(s_L\) be the sample mean and standard deviation (SD) of this loss. A candidate is eligible when
\begin{equation}
\hat\ell+1.645\frac{s_L}{\sqrt m}\leq0.02,
\qquad \widehat{\Cov}_{\rm val}\geq0.88.
\label{eq:validation_buffer}
\end{equation}
The accuracy screen allows a two-percentage-point loss with a buffer for sampling variation; the second condition imposes a coverage floor. Because the reference and the CSP candidate are both selected on these validation data, the buffer is not a post-selection confidence guarantee.

Among eligible candidates within 0.02 labels of the smallest mean set size, we choose the highest accuracy and break remaining ties by the smallest \(\lambda\). If no candidate is eligible, we return the reference and record the fallback; all choices are frozen before final calibration. Algorithm~\ref{alg:cpp} summarizes the procedure for base pruner \(B\) at sparsity \(\rho\); its selection record \(R\) stores the selected mixing coefficient or reference fallback.


\begin{algorithm}[t]
\caption{Candidate-Saliency Pruning (Top-2)}
\label{alg:cpp}
\begin{algorithmic}[1]
\Require \(\theta,s,B,\rho,\alpha,\Lambda\); disjoint data splits
\State Fix \(A(x)\) to the two lowest-score dense labels for each \(x\in\D_{\mathrm{prune}}\); break ties by label order
\State Compute \(I_{\mathrm{cal}}^A\) using Eq.~\ref{eq:candidate_importance}
\State Construct and select \GradRef{} \(\theta^{\mathrm{ref}}\) (Appendix~\ref{app:experimental})
\For{\(\lambda\in\Lambda\)}
    \State Construct \(\theta_\lambda\) using \(B\), \(\rho\), and Eq.~\ref{eq:cpp_importance}
    \State Fit a provisional quantile on \(\D_{\mathrm{val\text{-}cal}}\)
    \State Evaluate accuracy \(\hat a_\lambda\), coverage, and mean set size \(\hat c_\lambda\) on \(\D_{\mathrm{val\text{-}eval}}\)
\EndFor
\State \(E\gets\) coefficients in \(\Lambda\) passing Eq.~\ref{eq:validation_buffer} against \(\theta^{\mathrm{ref}}\)
\If{\(E=\varnothing\)}
    \State \(\theta^*\gets\theta^{\mathrm{ref}}\)
\Else
    \State \(E\gets\{\lambda\in E:\hat c_\lambda\leq\min_{\lambda'\in E}\hat c_{\lambda'}+0.02\}\)
    \State \(\lambda^*\gets\arg\max_{\lambda\in E}\hat a_\lambda\); break ties by smallest \(\lambda\)
    \State \(\theta^*\gets\theta_{\lambda^*}\)
\EndIf
\State Record the selected coefficient or reference fallback in \(R\); freeze \(\theta^*\) and its scorer
\State Fit \(\hat q_\alpha\) on untouched \(\D_{\mathrm{conf}}\) using Eq.~\ref{eq:k_order}
\State \Return \(\theta^*\), \(\hat q_\alpha\), \(R\)
\end{algorithmic}
\end{algorithm}

\subsection{Separating Label Choice from Gradient Squaring}

The squared true-label control (SqTL) replaces candidate labels with the observed label, using displacement importance
\begin{equation}
\begin{aligned}
I_{\mathrm{cal}}(\theta_j)
&=\theta_j^2G_{\mathrm{cal}}(\theta_j)\\
&=\frac1{|\D_{\mathrm{prune}}|}\sum_{(x,y)\in\D_{\mathrm{prune}}}
\left(\theta_j\tfrac{\partial s(x,y;\theta)}{\partial\theta_j}\right)^2,
\end{aligned}
\label{eq:cal_importance}
\end{equation}
where \(G_{\mathrm{cal}}\) (Appendix~\ref{app:sqtl_details}) is the mean squared true-label gradient. SqTL shares Top-2's 256 pruning examples, normalization, mixing grid, validation screen, mask support, and reconstruction, but uses one true-label gradient per example. The comparison tests candidate choice relative to true-label squaring while retaining SqTL's lower gradient cost. The reported TL-CEP aggregates use the same functional score under separate protocols (Appendix~\ref{app:reported_ablations}).

\section{Theory: Coverage and Set-Size Stability}
\label{sec:theory}

We first establish the standard marginal-coverage guarantee for a pruned classifier calibrated on an independent split. We then derive a set-size stability bound under uniform score perturbations and a density condition around the calibration threshold.

\begin{assumption}[Exchangeability and split independence]
\label{assump:exchangeability}
The examples in \(\D_{\mathrm{conf}}\) and the future test example are exchangeable. The final pruned model \(\theta'\), including its mask, saliency construction, sparsity, \(\lambda\), prompt, verbalizers, and temperature, is measurable with respect to data independent of \(\D_{\mathrm{conf}}\) and the test example.
\end{assumption}

\begin{theorem}[Split-conformal coverage after pruning]
\label{thm:coverage}
Under Assumption~\ref{assump:exchangeability}, let \(\theta'\) be any model fixed before \(\D_{\mathrm{conf}}\) is accessed. Then the set in Eq.~\ref{eq:conformal_set} satisfies
\begin{equation}
\PP\!\left\{Y_{n+1}\in C_\alpha(X_{n+1};\theta')\right\}
\geq1-\alpha.
\label{eq:coverage}
\end{equation}
If scores are almost surely distinct and \(k\leq n\), coverage is less than \(1-\alpha+1/(n+1)\).
\end{theorem}

Conditioning on pruning and validation leaves the final-calibration and test scores exchangeable, so the standard rank argument applies. The probability averages over final calibration and a future example; Appendix~\ref{app:proof} separates calibration variability from test variability.

Set-size stability depends on changes in both candidate-label scores and the calibration threshold. The following lemma bounds how much the threshold can move when calibration scores change.

\begin{lemma}[Order-statistic stability]
\label{lem:order_stat_stability}
If \(\max_i|a_i-b_i|\leq\epsilon\), then their \(k\)-th order statistics obey
\(|a_{(k)}-b_{(k)}|\leq\epsilon\).
\end{lemma}

\begin{theorem}[Set-size stability under bounded score perturbation]
\label{thm:set_size_stability}
Fix the dense and pruned models and one common conformal calibration sample, independently of a fresh test input. Suppose \(k\leq n\), \(\Y\) is finite, and
\begin{equation}
|s(x,y;\theta')-s(x,y;\theta)|\leq\epsilon
\label{eq:uniform_score_bound}
\end{equation}
for every calibration pair and, almost surely over the test input, every candidate-label pair \((x,y)\). Then
\begin{equation}
C_\alpha(x;\theta')
\subseteq
\{y:s(x,y;\theta)\leq\hat q_\alpha(\theta)+2\epsilon\}.
\label{eq:set_containment}
\end{equation}
Conditional on the fixed models and calibration sample, if each \(s(X,y;\theta)\) has density at most \(M\) between
\(\hat q_\alpha(\theta)\) and \(\hat q_\alpha(\theta)+2\epsilon\), then
\begin{equation}
\E_X|C_\alpha(X;\theta')|
\leq\E_X|C_\alpha(X;\theta)|+2M|\Y|\epsilon.
\label{eq:set_size_bound}
\end{equation}
\end{theorem}

For \(k\leq n\), only labels entering the \(2\epsilon\) band above the dense threshold can inflate the pruned set; under the infinite-quantile convention, \(k>n\) makes both sets equal to \(\Y\). The uniform bound is a sufficient stability condition, and CSP does not certify that it holds.

For a pure mask removing \(m\) coordinates, Proposition~\ref{prop:first_order} bounds empirical mean score displacement by the diagonal term \(\{m\sum_{j\in P(\mask)}I_{\mathrm{cal},\D_0}(\theta_j)\}^{1/2}\) plus a local second-order remainder \((H/2)\|\theta_\mask-\theta\|_2^2\).

This is a local approximation. Diagonalization ignores cancellation across coordinates, and the Hessian remainder can dominate at high sparsity. An empirical true-label average does not establish the uniform candidate-label bound in Theorem~\ref{thm:set_size_stability}. SparseGPT reconstruction changes retained weights as well, so its analysis must bound the full perturbation. Appendix~\ref{app:proof} gives the smoothness assumption, the full-perturbation bound, and the complete proofs.

\section{Experiments}
\label{sec:experiments}

\subsection{Setup}
\label{sec:setup_eval}

We first measure the cost and set-size trade-off between candidate saliency and direct efficiency search. We then test whether the candidate-label advantage survives recovery at fixed update counts, at equal total time, and with recovery restricted to the pruning pool. Each comparison evaluates set size together with accuracy matching. Appendix~\ref{app:supporting} reports the full comparisons and ablations. Separate descriptive controls evaluate Qwen on AG News~\citep{zhang2015textclassification} and TREC~\citep{li2002trec}, two tasks with small label sets (Table~\ref{tab:pilot_multi_extended}; Appendix~\ref{app:aggregate_controls}).

The encoder evaluation covers three tasks, Wanda and SparseGPT, and 30\%, 50\%, and 70\% sparsity, with 16 paired seeds. Cell names encode the model (R/D/Q for RoBERTa/fine-tuned Qwen/zero-shot Qwen), task (1/2/3 for DBpedia-14/Banking77/CLINC150), base pruner (W/S for Wanda/SparseGPT), and sparsity. Within each paired seed, methods share the dense task checkpoint, data splits, calibration sample, and test cases (Appendix~\ref{app:eval_details}). We evaluate each sparsity level directly because neither accuracy nor set size is assumed to vary monotonically with sparsity.

\GradRef{} denotes the selected absolute true-label gradient reference defined in Section~\ref{sec:selection}: the validation-selected combination of the base pruner and absolute true-label saliency, including its base-pruner fallback. SqTL denotes the squared true-label control; Top and Thr denote the Top-2 and Threshold-2 rules. Two one-sided tests (TOST) assess test accuracy equivalence with a \(\pm0.005\) margin on paired seed differences (Appendix~\ref{app:eval_details}).

\subsection{Candidate Saliency vs Direct Efficiency Search}

We first ask whether Top-2 reduces set size across encoder sparsities while meeting the accuracy-equivalence requirement. Top-2 has a smaller mean set size than the \GradRef{} in all 18 cells, with every paired 95\% set-size interval below zero and mean reductions from 0.216 to 0.881 labels (Table~\ref{tab:csp_top_compact}; accuracy and coverage in Table~\ref{tab:csp_top_grid}, Appendix~\ref{app:supporting}; dense references in Table~\ref{tab:tcep_dense_summary}). The Banking77 70\% rows show why set size alone is insufficient: their paired accuracy intervals do not establish equivalence within \(\pm0.005\).

\begin{table}[t]
\centering\scriptsize
\setlength{\tabcolsep}{3pt}
\caption{Set-size difference between Top-2 and the selected absolute-gradient reference (\GradRef{}) across encoder sparsities. Entries are mean Top-2 minus \GradRef{} differences in labels, with paired 95\% Student-\(t\) confidence intervals over 16 seeds. R1/R2/R3 denote RoBERTa on DBpedia-14/Banking77/CLINC150. Table~\ref{tab:csp_top_grid} reports accuracy and coverage.}
\label{tab:csp_top_compact}
\CSPfitTable{\columnwidth}{%
\begin{tabular}{llcc}
\toprule
Task & Sparsity & Wanda \(\Delta|C|\) [95\%] \(\downarrow\) & SparseGPT \(\Delta|C|\) [95\%] \(\downarrow\)\\
\midrule
R1 & 30\% & \({-0.216\;[-0.262,-0.171]}\) & \({-0.238\;[-0.281,-0.195]}\)\\
R1 & 50\% & \({-0.367\;[-0.466,-0.267]}\) & \({-0.377\;[-0.473,-0.280]}\)\\
R1 & 70\% & \({-0.408\;[-0.531,-0.285]}\) & \({-0.428\;[-0.552,-0.303]}\)\\
R2 & 30\% & \({-0.425\;[-0.512,-0.338]}\) & \({-0.434\;[-0.517,-0.351]}\)\\
R2 & 50\% & \({-0.511\;[-0.671,-0.351]}\) & \({-0.488\;[-0.661,-0.315]}\)\\
R2 & 70\% & \({-0.783\;[-1.057,-0.509]}\) & \({-0.881\;[-1.177,-0.585]}\)\\
R3 & 30\% & \({-0.348\;[-0.415,-0.282]}\) & \({-0.356\;[-0.415,-0.296]}\)\\
R3 & 50\% & \({-0.478\;[-0.629,-0.326]}\) & \({-0.483\;[-0.630,-0.335]}\)\\
R3 & 70\% & \({-0.809\;[-1.036,-0.581]}\) & \({-0.702\;[-0.894,-0.510]}\)\\
\bottomrule
\end{tabular}}
\end{table}

To evaluate the proxy directly, we compare Top-2 block saliency with the signed finite-deletion criterion of Zhao and Bellotti (ZB-sign), using the same class score, 32 coordinate blocks per matrix, 50\% block sparsity, and independent final calibration. This matched transformer adaptation differs from the published hidden-unit benchmark; direct search is not an optimal-mask reference (Appendices~\ref{app:eval_details} and \ref{app:csp_extended} specify the shared mask support and decoder comparison).

Table~\ref{tab:proxy_efficiency_cost} reports final set-size excess alongside complete offline time for the encoder and decoder cells. On the encoder, Top-2's mean relative set-size differences against direct search are \({4.22\%}\), \({3.16\%}\), and \({-1.01\%}\), while direct search costs \({10.3}\) to \({11.8}\) times as much offline time. The paired intervals thus support a cost-efficiency trade-off rather than uniform dominance.


Tables~\ref{tab:threshold_regime} and \ref{tab:tcep_full_grid} extend the evaluation to candidate choice and decoder transfer. Because models and scorers differ across regimes, differences in set-size savings do not establish set size as a causal moderator.

\subsection{Recovery with Fixed Updates}

The powered study uses 200 paired seeds in each of the six encoder task/base cells at 50\% sparsity. All arms share the task checkpoint, data partition, recovery minibatches, optimizer settings, and checkpoint-selection rule. We compare budgets of 0, 100, 250, 500, and 1,000 updates; every scheduled update and checkpoint evaluation counts toward cost, even when validation selects an earlier checkpoint.

For each pair, define the relative set-size difference \(d=(C_A-C_B)/C_B\). Meaningful benefit requires \(\E d<-0.02\); practical equivalence requires \(|\E d|<0.02\); each decision also requires accuracy equivalence within \(\pm0.005\). One 270-claim Holm family covers fixed-update, equal-time, cross-budget, and restricted-data decisions; Appendix~\ref{app:powered_protocol} specifies the estimands, tests, and power calculations.

\begin{table*}[t]
\centering\scriptsize
\setlength{\tabcolsep}{3pt}
\caption{Set-size differences across recovery budgets in six encoder task/base cells, with 200 paired seeds per cell. Top, \GradRef{}, and SqTL denote Top-2, the absolute-gradient reference, and the squared true-label control. Ranges give the minimum and maximum cellwise mean relative differences (\%); they are not confidence intervals or pooled effects. Negative values favor Top. B/E/U count meaningful-benefit, practical-equivalence, and unresolved decisions under the joint Holm rule over 270 claims. Table~\ref{tab:powered_recovery} reports the Top versus SqTL contrasts by cell.}
\label{tab:recovery_path_headline}
\CSPfitTable{\textwidth}{%
\begin{tabular}{rcccc}
\toprule
Updates & Top versus \GradRef{} range (\%) \(\downarrow\) & B/E/U & Top versus SqTL range (\%) \(\downarrow\) & B/E/U\\
\midrule
0 & \({-19.2\ \text{to}\ -13.2}\) & \({6/0/0}\) & \({-13.2\ \text{to}\ -7.7}\) & \({6/0/0}\)\\
100 & \({-5.0\ \text{to}\ -2.7}\) & \({4/0/2}\) & \({-3.5\ \text{to}\ -1.6}\) & \({3/0/3}\)\\
250 & \({-3.7\ \text{to}\ -2.6}\) & \({5/0/1}\) & \({-2.5\ \text{to}\ -1.7}\) & \({0/0/6}\)\\
500 & \({-2.5\ \text{to}\ -1.9}\) & \({0/0/6}\) & \({-2.0\ \text{to}\ -1.4}\) & \({0/0/6}\)\\
1000 & \({-1.3\ \text{to}\ -0.4}\) & \({0/4/2}\) & \({-0.9\ \text{to}\ -0.3}\) & \({0/6/0}\)\\
\bottomrule
\end{tabular}}
\end{table*}

Table~\ref{tab:recovery_path_headline} shows the candidate-label advantage shrinking with recovery. Top holds a meaningful advantage over SqTL in \({6/6}\) cells without updates and \({3/6}\) at 100 updates. At 250 and 500 updates, none of the six Top-versus-SqTL comparisons establishes either a greater-than-2\% benefit or equivalence within \(\pm2\%\) under the Holm-adjusted joint tests; all twelve decisions remain unresolved; at 1,000 updates, all \({6/6}\) comparisons establish practical equivalence (Tables~\ref{tab:powered_zero} and \ref{tab:powered_recovery} report each cell's Top-versus-SqTL contrasts with 95\% intervals).

The squared true-label control thus accounts for part of the gain over absolute gradients, and candidate selection adds an initial advantage that longer recovery removes within the tested equivalence margin.


\subsection{Equal Total-time Budgets}

Fixed-update comparisons leave each method's construction cost unmatched, so we also impose three wall-clock deadlines per task/base cell, fixed before testing. Each standalone arm pays for checkpoint loading, gradients, all mixing-grid candidates, validation selection, optimizer updates, checkpoint evaluations, final calibration, and testing; a shared deadline can thus permit different numbers of recovery updates. Comparisons use completed operating points (Appendix~\ref{app:eval_details}).

\begin{table*}[t]
\centering\scriptsize
\setlength{\tabcolsep}{3pt}
\caption{Encoder results at each cell's lowest total-time deadline, with 200 paired seeds. Triples list \GradRef{}/SqTL/Top: the absolute-gradient reference, squared true-label control, and Top-2. Cap is the deadline in minutes; updates are mean completed counts. Relative size compares Top with SqTL. Joint decisions require accuracy equivalence, the stated size endpoint, and elapsed-time equivalence within 1\% of the cap, with Holm correction over 270 claims. Every run finishes within its cap. Table~\ref{tab:equal_time_recovery} reports all three deadlines. Bold indicates the best value per row; \(\downarrow\) indicates lower is better.}
\label{tab:equal_time_headline}
\CSPfitTable{\textwidth}{%
\begin{tabular}{lrcccccc}
\toprule
Cell & Cap & Completed updates & Set size \(\downarrow\) & Relative size [95\%] \(\downarrow\) & Benefit \(p\) & Equiv. \(p\) & Decision\\
\midrule
R1W50 & 40 & \({683/305/271}\) & \({1.586/1.608/\mathbf{1.569}}\) & \({-0.024\;[-0.029,-0.018]}\) & \({1.0000}\) & \({1.0000}\) & \({\mathrm{unresolved}}\)\\
R1S50 & 58 & \({876/281/247}\) & \({1.516/1.545/\mathbf{1.510}}\) & \({-0.022\;[-0.027,-0.017]}\) & \({1.0000}\) & \({1.0000}\) & \({\mathrm{unresolved}}\)\\
R2W50 & 64 & \({793/340/313}\) & \({\mathbf{3.112}/3.144/3.113}\) & \({-0.009\;[-0.014,-0.004]}\) & \({1.0000}\) & \({3.02\times10^{-4}}\) & \({\mathrm{equivalent}}\)\\
R2S50 & 90 & \({1026/366/339}\) & \({\mathbf{2.788}/2.831/\mathbf{2.788}}\) & \({-0.015\;[-0.020,-0.010]}\) & \({1.0000}\) & \({1.0000}\) & \({\mathrm{unresolved}}\)\\
R3W50 & 84 & \({861/360/335}\) & \({2.654/2.673/\mathbf{2.632}}\) & \({-0.015\;[-0.020,-0.010]}\) & \({1.0000}\) & \({1.0000}\) & \({\mathrm{unresolved}}\)\\
R3S50 & 125 & \({1208/442/418}\) & \({2.325/2.363/\mathbf{2.313}}\) & \({-0.020\;[-0.026,-0.015]}\) & \({1.0000}\) & \({1.0000}\) & \({\mathrm{unresolved}}\)\\
\bottomrule
\end{tabular}}
\end{table*}

At equal total time, Top and \GradRef{} are practically equivalent in all 18 task/base/deadline settings. \GradRef{} completes more recovery updates than Top under each cap (Table~\ref{tab:equal_time_headline}). All reported equal-time runs finish below their assigned deadlines. Top establishes a greater-than-2\% benefit over \GradRef{} in \({0/18}\) settings. Against SqTL, it establishes benefit in \({0/18}\) settings and equivalence in \({11/18}\), leaving the remaining decisions unresolved. Table~\ref{tab:equal_time_recovery} in Appendix~\ref{app:supporting} reports all three deadlines, completed-update counts, and every unresolved comparison. The initial set-size advantage thus does not establish a meaningful equal-time benefit over \GradRef{}. Table~\ref{tab:encoder_saliency_budget} and Figure~\ref{fig:runtime_phases} account for the additional saliency work with 32-seed gradient-budget and phase-runtime ablations.

\subsection{Restricted Training Data}

Equal-time recovery assumes access to training data. To test restricted access, we transfer the fixed dense checkpoint but withhold its original task-fitting corpus; compression retains the 256 labeled pruning examples and the same disjoint validation and calibration roles. \GradRef{} and SqTL may make 1,000 masked recovery updates by cycling the pruning pool and selecting checkpoints on validation; Top uses no recovery.

\begin{table*}[t]
\centering\scriptsize
\setlength{\tabcolsep}{3pt}
\caption{Checkpoint transfer without the original task-training corpus, with 200 paired seeds. All arms retain the same 256 labeled pruning examples, with separate validation and calibration data. Top-0 is Top-2 without recovery; \GradRef{} and SqTL are the absolute-gradient and squared true-label baselines, each allowed 1,000 masked updates on the pruning examples with validation selection. Set-size triples list Top-0/\GradRef{}-recovery/SqTL-recovery. Each contrast requires a size benefit greater than 2\% and accuracy equivalence. Joint \(p\)-values use Holm correction over 270 claims. Bold indicates the best value per row; \(\downarrow\) indicates lower is better.}
\label{tab:restricted_recovery}
\CSPfitTable{\textwidth}{%
\begin{tabular}{lccccc}
\toprule
Cell & Set size \(\downarrow\) & Top-0 vs \GradRef{} [95\%] \(\downarrow\) & Top-0 vs SqTL [95\%] \(\downarrow\) & \GradRef{} \(p\) & SqTL \(p\)\\
\midrule
R1W50 & \({\mathbf{1.684}/2.032/1.896}\) & \({-0.171\;[-0.178,-0.163]}\) & \({-0.109\;[-0.119,-0.100]}\) & \({3.72\times10^{-19}}\) & \({1.10\times10^{-21}}\)\\
R1S50 & \({\mathbf{1.591}/1.932/1.805}\) & \({-0.176\;[-0.183,-0.169]}\) & \({-0.117\;[-0.126,-0.108]}\) & \({2.03\times10^{-27}}\) & \({2.28\times10^{-19}}\)\\
R2W50 & \({\mathbf{3.446}/3.911/3.694}\) & \({-0.119\;[-0.126,-0.112]}\) & \({-0.065\;[-0.074,-0.056]}\) & \({3.64\times10^{-28}}\) & \({1.25\times10^{-16}}\)\\
R2S50 & \({\mathbf{3.081}/3.527/3.295}\) & \({-0.126\;[-0.134,-0.119]}\) & \({-0.063\;[-0.073,-0.054]}\) & \({5.23\times10^{-26}}\) & \({7.41\times10^{-14}}\)\\
R3W50 & \({\mathbf{2.817}/3.248/3.043}\) & \({-0.132\;[-0.140,-0.124]}\) & \({-0.073\;[-0.083,-0.063]}\) & \({1.67\times10^{-30}}\) & \({7.32\times10^{-18}}\)\\
R3S50 & \({\mathbf{2.448}/2.879/2.720}\) & \({-0.149\;[-0.157,-0.142]}\) & \({-0.098\;[-0.108,-0.089]}\) & \({9.57\times10^{-19}}\) & \({3.12\times10^{-16}}\)\\
\bottomrule
\end{tabular}}
\end{table*}

Without recovery, the DBpedia reductions are \({18.2\%}\) to \({19.2\%}\) against \GradRef{} and \({12.1\%}\) to \({13.2\%}\) against SqTL. The smaller contrast against SqTL measures the additional candidate-label effect, which still exceeds the specified benefit margin.

Recovery restricted to the pruning pool does not eliminate this advantage: Top-0 (Top-2 without recovery) retains a meaningful set-size benefit with matched accuracy against both recovered baselines in all six cells. Across cells, the paired-seed mean relative reduction ranges from 11.9\% to 17.6\% against recovered \GradRef{} and from 6.3\% to 11.7\% against recovered SqTL. Table~\ref{tab:restricted_recovery} reports the corresponding Holm-adjusted joint \(p\)-values. (Appendix~\ref{app:supporting}). Table~\ref{tab:restricted_metrics} reports the recovered baselines' accuracy, coverage, set size, and complete cost.

\section{Discussion and Conclusion}

Independent calibration supplies marginal coverage, but the pruning rule determines which label-score information survives. CSP targets set size through this rule with a fixed candidate-gradient budget: direct-search comparisons show the proxy's cost-efficiency trade-off, and the squared true-label control identifies the additional benefit of candidate choice.

Across the tested encoder task/base cells, CSP reduces set size before recovery and when recovery uses only the retained pruning pool. With longer recovery, Top and SqTL are practically equivalent; at equal total time, Top is also practically equivalent to the cheaper reference. These findings support CSP when recovery time or training-data access is limited.

\section*{Limitations}

The powered study covers six encoder task/base cells at 50\% sparsity, and its planning calculations and parametric tests depend on an adequate model of seed differences. The three- and five-seed aggregate controls provide descriptive means and cannot support the powered study's paired inference.

The guarantees also have limited scope. Independent calibration requires exchangeability and does not ensure conditional, subgroup, or shift coverage. Score stability relies on sufficient conditions not verified as uniform certificates for the evaluated models. SparseGPT compensation extends beyond a pure mask, and the block-based direct-search comparison adapts the predecessor's support and scorer.

Decoder classification and smaller prediction sets do not establish free-form generation quality or downstream utility. Incremental pruning time does not establish sparse-kernel speed or energy savings.

\section*{Ethical Considerations}

Population-averaged coverage can conceal failures on rare or costly cases. Deployment assessment should examine subgroup accuracy and coverage, empty-set handling, distribution shift, and the consequences of retaining or excluding candidate labels. The application determines whether these outcomes make compression useful.

\section*{AI Use Statement}
Generative AI tools assist with language editing, literature discovery, mathematical exposition, and source-code preparation. The authors verify the proofs, citations, experimental records, and final claims and remain responsible for the manuscript.

\section*{Reproducibility Statement}
The appendices give the theoretical proofs and assumptions, data allocations and classifier configurations, pruning and calibration procedures, and controlled experimental settings. They document the candidate-label and gradient-budget controls, direct-search comparisons, recovery schedules, and timing boundaries. The statistical protocols specify the replication structure, paired tests, and multiplicity correction. The reported results retain their stated numerical units and uncertainty conventions.

\bibliography{ref}

@article{kirkpatrick2017overcoming,
  title={Overcoming catastrophic forgetting in neural networks},
  author={Kirkpatrick, James and Pascanu, Razvan and Rabinowitz, Neil and Veness, Joel and Desjardins, Guillaume and Rusu, Andrei A and Milan, Kieran and Quan, John and Ramalho, Tiago and Grabska-Barwinska, Agnieszka and others},
  journal={Proceedings of the National Academy of Sciences},
  volume={114},
  number={13},
  pages={3521--3526},
  year={2017}
}

@inproceedings{wang2023orthogonal,
  title={Orthogonal subspace learning for language model continual learning},
  author={Wang, Xiao and Chen, Tianze and Ge, Qiming and Xia, Han and Bao, Rong and Zheng, Rui and Zhang, Qi and Gui, Tao and Huang, Xuanjing},
  booktitle={Findings of EMNLP},
  year={2023}
}

@inproceedings{zenke2017continual,
  title={Continual learning through synaptic intelligence},
  author={Zenke, Friedemann and Poole, Ben and Ganguli, Surya},
  booktitle={International Conference on Machine Learning},
  year={2017}
}

@article{wu2024csur,
  title={Continual learning of large language models: A comprehensive survey},
  author={Shi, Haizhou and Xu, Zihao and Wang, Hengyi and Qin, Weiyi and Wang, Wenyuan and Wang, Yibin and Wang, Zifeng and Ebrahimi, Sayna and Wang, Hao},
  journal={ACM Computing Surveys},
  volume={58},
  number={5},
  pages={1--42},
  year={2026},
  doi={10.1145/3735633}
}

@article{wang2024survey,
  title={A comprehensive survey of continual learning: Theory, method and application},
  author={Wang, Liyuan and Zhang, Xingxing and Su, Hang and Zhu, Jun},
  journal={IEEE Transactions on Pattern Analysis and Machine Intelligence},
  year={2024}
}

@book{vovk2005algorithmic,
  title={Algorithmic learning in a random world},
  author={Vovk, Vladimir and Gammerman, Alexander and Shafer, Glenn},
  year={2005},
  publisher={Springer}
}

@article{angelopoulos2023gentle,
  title={Conformal prediction: A gentle introduction},
  author={Angelopoulos, Anastasios N and Bates, Stephen},
  journal={Foundations and Trends in Machine Learning},
  volume={16},
  number={4},
  pages={494--591},
  year={2023},
  doi={10.1561/2200000101}
}

@inproceedings{quach2024cp,
  title={Conformal language modeling},
  author={Quach, Victor and Fisch, Adam and Schuster, Tal and Yala, Adam and Sohn, Jae Ho and Jaakkola, Tommi S and Barzilay, Regina},
  booktitle={International Conference on Learning Representations},
  year={2024}
}

@inproceedings{mohri2024factuality,
  title={Language models with conformal factuality guarantees},
  author={Mohri, Christopher and Hashimoto, Tatsunori},
  booktitle={International Conference on Machine Learning},
  year={2024}
}

@article{shihab2025infolift,
  title={Anytime-Valid Answer Sufficiency Certificates for {LLM} Generation via Sequential Information Lift},
  author={Akter, Sanjeda and Shihab, Ibne Farabi and Sharma, Anuj},
  journal={arXiv preprint arXiv:2510.06478},
  year={2025},
  url={https://arxiv.org/abs/2510.06478}
}

@inproceedings{guo2017calibration,
  title={On calibration of modern neural networks},
  author={Guo, Chuan and Pleiss, Geoff and Sun, Yu and Weinberger, Kilian Q},
  booktitle={International Conference on Machine Learning},
  year={2017}
}

@incollection{platt1999probabilistic,
  title={Probabilistic outputs for support vector machines and comparisons to regularized likelihood methods},
  author={Platt, John},
  booktitle={Advances in Large Margin Classifiers},
  publisher={MIT Press},
  year={1999}
}

@inproceedings{park2020calibrated,
  title={Calibrated prediction with covariate shift via unsupervised domain adaptation},
  author={Park, Sangdon and Bastani, Osbert and Weimer, James and Lee, Insup},
  booktitle={International Conference on Artificial Intelligence and Statistics},
  year={2020}
}

@inproceedings{gibbs2021adaptive,
  title={Adaptive conformal inference under distribution shift},
  author={Gibbs, Isaac and Candes, Emmanuel},
  booktitle={Advances in Neural Information Processing Systems},
  year={2021}
}

@article{angelopoulos2022ltt,
  title={Learn then test: Calibrating predictive algorithms to achieve risk control},
  author={Angelopoulos, Anastasios N and Bates, Stephen and Cand{\`e}s, Emmanuel J and Jordan, Michael I and Lei, Lihua},
  journal={The Annals of Applied Statistics},
  volume={19},
  number={2},
  pages={1641--1662},
  year={2025},
  doi={10.1214/24-AOAS1998}
}

@article{bates2021selective,
  title={Distribution-free, risk-controlling prediction sets},
  author={Bates, Stephen and Angelopoulos, Anastasios and Lei, Lihua and Malik, Jitendra and Jordan, Michael},
  journal={Journal of the ACM},
  volume={68},
  number={6},
  pages={1--34},
  year={2021},
  doi={10.1145/3478535}
}

@inproceedings{li2017lwf,
  title={Learning without forgetting},
  author={Li, Zhizhong and Hoiem, Derek},
  booktitle={European Conference on Computer Vision},
  year={2016}
}

@inproceedings{mallya2018packnet,
  title={{PackNet}: Adding multiple tasks to a single network by iterative pruning},
  author={Mallya, Arun and Lazebnik, Svetlana},
  booktitle={IEEE Conference on Computer Vision and Pattern Recognition},
  year={2018}
}

@inproceedings{wang2022dualprompt,
  title={{DualPrompt}: Complementary prompting for rehearsal-free continual learning},
  author={Wang, Zifeng and Zhang, Zizhao and Ebrahimi, Sayna and Sun, Ruoxi and Zhang, Han and Lee, Chen-Yu and Ren, Xiaoqi and Su, Guolong and Perot, Vincent and Dy, Jennifer and others},
  booktitle={European Conference on Computer Vision},
  year={2022}
}

@inproceedings{sun2024wanda,
  title={A simple and effective pruning approach for large language models},
  author={Sun, Mingjie and Liu, Zhuang and Bair, Anna and Kolter, J Zico},
  booktitle={International Conference on Learning Representations},
  year={2024}
}

@inproceedings{frantar2023sparsegpt,
  title={{SparseGPT}: Massive language models can be accurately pruned in one-shot},
  author={Frantar, Elias and Alistarh, Dan},
  booktitle={International Conference on Machine Learning},
  year={2023}
}

@inproceedings{ma2023llm,
  title={{LLM-Pruner}: On the structural pruning of large language models},
  author={Ma, Xinyin and Fang, Gongfan and Wang, Xinchao},
  booktitle={Advances in Neural Information Processing Systems},
  year={2023}
}

@inproceedings{han2015learning,
  title={Learning both weights and connections for efficient neural networks},
  author={Han, Song and Pool, Jeff and Tran, John and Dally, William},
  booktitle={Advances in Neural Information Processing Systems},
  year={2015}
}

@inproceedings{frantar2025scaling,
  title={Scaling laws for sparsely-connected foundation models},
  author={Frantar, Elias and Riquelme, Carlos and Houlsby, Neil and Alistarh, Dan and Evci, Utku},
  booktitle={International Conference on Learning Representations},
  year={2024}
}

@InProceedings{pmlr-v179-zhao22a,
  title = 	 {Pruning neural networks for inductive conformal prediction},
  author =       {Zhao, Xindi and Bellotti, Anthony},
  booktitle = 	 {Proceedings of the Eleventh Symposium on Conformal and Probabilistic Prediction with Applications},
  series =       {Proceedings of Machine Learning Research},
  volume =       {179},
  pages =        {273--293},
  year = 	 {2022},
}

@article{bayesian_lth_2024,
  title   = {Bayesian Lottery Ticket Hypothesis},
  author  = {Kuhn, Nicholas and Weyrauch, Arvid and Heyen, Lars and Streit, Achim and G\"{o}tz, Markus and Debus, Charlotte},
  journal = {arXiv preprint arXiv:2602.18825},
  year    = {2026}
}

@inproceedings{open_lth_2024,
  title     = {The Open-World Lottery Ticket Hypothesis for {OOD} Intent Classification},
  author    = {Zhou, Yunhua and Wang, Pengyu and Liu, Peiju and Wang, Yuxin and Qiu, Xipeng},
  booktitle = {Proceedings of the 2024 Joint International Conference on Computational Linguistics, Language Resources and Evaluation (LREC-COLING 2024)},
  pages     = {15988--15999},
  year      = {2024},
  url       = {https://aclanthology.org/2024.lrec-main.1390/}
}

@inproceedings{vishwakarma2025prune,
  title     = {Prune 'n Predict: Optimizing {LLM} Decision-making with Conformal Prediction},
  author    = {Vishwakarma, Harit and Mishler, Alan and Cook, Thomas and Dalmasso, Niccol\`{o} and Raman, Natraj and Ganesh, Sumitra},
  booktitle = {Proceedings of the 42nd International Conference on Machine Learning (ICML)},
  year      = {2025}
}

@inproceedings{casanueva2020banking77,
    title = "Efficient Intent Detection with Dual Sentence Encoders",
    author = "Casanueva, I{\~n}igo  and
      Tem{\v{c}}inas, Tadas  and
      Gerz, Daniela  and
      Henderson, Matthew  and
      Vuli{\'c}, Ivan",
    booktitle = "Proceedings of the 2nd Workshop on Natural Language Processing for Conversational AI",
    month = jul,
    year = "2020",
    address = "Online",
    publisher = "Association for Computational Linguistics",
    url = "https://aclanthology.org/2020.nlp4convai-1.5/",
    doi = "10.18653/v1/2020.nlp4convai-1.5",
    pages = "38--45"
}

@inproceedings{larson2019clinc150,
    title = "An Evaluation Dataset for Intent Classification and Out-of-Scope Prediction",
    author = "Larson, Stefan  and
      Mahendran, Anish  and
      Peper, Joseph J.  and
      Clarke, Christopher  and
      Lee, Andrew  and
      Hill, Parker  and
      Kummerfeld, Jonathan K.  and
      Leach, Kevin  and
      Laurenzano, Michael A.  and
      Tang, Lingjia  and
      Mars, Jason",
    booktitle = "Proceedings of the 2019 Conference on Empirical Methods in Natural Language Processing and the 9th International Joint Conference on Natural Language Processing (EMNLP-IJCNLP)",
    month = nov,
    year = "2019",
    address = "Hong Kong, China",
    publisher = "Association for Computational Linguistics",
    url = "https://aclanthology.org/D19-1131/",
    doi = "10.18653/v1/D19-1131",
    pages = "1311--1316"
}

@inproceedings{li2002trec,
    title = "Learning Question Classifiers",
    author = "Li, Xin  and
      Roth, Dan",
    booktitle = "{COLING} 2002: The 19th International Conference on Computational Linguistics",
    year = "2002",
    url = "https://aclanthology.org/C02-1150/"
}

@inproceedings{zhang2015textclassification,
 author = {Zhang, Xiang and Zhao, Junbo and LeCun, Yann},
 booktitle = {Advances in Neural Information Processing Systems},
 editor = {C. Cortes and N. Lawrence and D. Lee and M. Sugiyama and R. Garnett},
 publisher = {Curran Associates, Inc.},
 title = {Character-level Convolutional Networks for Text Classification},
 url = {https://proceedings.neurips.cc/paper_files/paper/2015/file/250cf8b51c773f3f8dc8b4be867a9a02-Paper.pdf},
 volume = {28},
 year = {2015}
}

@article{liu2019roberta,
  title={{RoBERTa}: A Robustly Optimized {BERT} Pretraining Approach},
  author={Liu, Yinhan and Ott, Myle and Goyal, Naman and Du, Jingfei and Joshi, Mandar and Chen, Danqi and Levy, Omer and Lewis, Mike and Zettlemoyer, Luke and Stoyanov, Veselin},
  journal={arXiv preprint arXiv:1907.11692},
  year={2019},
  url={https://arxiv.org/abs/1907.11692}
}

@article{yang2024qwen25,
  title={{Qwen2.5} Technical Report},
  author={Yang, An and Yang, Baosong and Zhang, Beichen and Hui, Binyuan and Zheng, Bo and Yu, Bowen and Li, Chengyuan and Liu, Dayiheng and Huang, Fei and Wei, Haoran and Lin, Huan and Yang, Jian and Tu, Jianhong and Zhang, Jianwei and Yang, Jianxin and Yang, Jiaxi and Zhou, Jingren and Lin, Junyang and Dang, Kai and Lu, Keming and Bao, Keqin and Yang, Kexin and Yu, Le and Li, Mei and Xue, Mingfeng and Zhang, Pei and Zhu, Qin and Men, Rui and Lin, Runji and Li, Tianhao and Tang, Tianyi and Xia, Tingyu and Ren, Xingzhang and Ren, Xuancheng and Fan, Yang and Su, Yang and Zhang, Yichang and Wan, Yu and Liu, Yuqiong and Cui, Zeyu and Zhang, Zhenru and Qiu, Zihan},
  journal={arXiv preprint arXiv:2412.15115},
  year={2024},
  url={https://arxiv.org/abs/2412.15115}
}

@article{dubey2024llama3,
  title={The {Llama 3} Herd of Models},
  author={Dubey, Abhimanyu and Jauhri, Abhinav and Pandey, Abhinav and Kadian, Abhishek and Al-Dahle, Ahmad and Letman, Aiesha and Mathur, Akhil and Schelten, Alan and Yang, Amy and Fan, Angela and others},
  journal={arXiv preprint arXiv:2407.21783},
  year={2024},
  url={https://arxiv.org/abs/2407.21783v1}
}

@inproceedings{merity2017wikitext,
  title={Pointer Sentinel Mixture Models},
  author={Merity, Stephen and Xiong, Caiming and Bradbury, James and Socher, Richard},
  booktitle={International Conference on Learning Representations},
  year={2017},
  url={https://openreview.net/forum?id=Byj72udxe}
}

@inproceedings{lee2019snip,
  title={{SNIP}: Single-shot network pruning based on connection sensitivity},
  author={Lee, Namhoon and Ajanthan, Thalaiyasingam and Torr, Philip H. S.},
  booktitle={International Conference on Learning Representations},
  year={2019}
}

@inproceedings{lecun1989obd,
  title={Optimal brain damage},
  author={LeCun, Yann and Denker, John S. and Solla, Sara A.},
  booktitle={Advances in Neural Information Processing Systems},
  year={1989}
}

@inproceedings{sanh2020movement,
  title={Movement pruning: Adaptive sparsity by fine-tuning},
  author={Sanh, Victor and Wolf, Thomas and Rush, Alexander M.},
  booktitle={Advances in Neural Information Processing Systems},
  year={2020}
}

@article{zhao2025pruningcnn,
  title={Pruning convolutional neural networks for inductive conformal prediction},
  author={Zhao, Xindi and Farjudian, Amin and Bellotti, Anthony},
  journal={Neurocomputing},
  volume={611},
  pages={128704},
  year={2025},
  doi={10.1016/j.neucom.2024.128704}
}

@article{lei2018distribution,
  title={Distribution-free predictive inference for regression},
  author={Lei, Jing and G'Sell, Max and Rinaldo, Alessandro and Tibshirani, Ryan J. and Wasserman, Larry},
  journal={Journal of the American Statistical Association},
  volume={113},
  number={523},
  pages={1094--1111},
  year={2018},
  doi={10.1080/01621459.2017.1307116}
}

@article{sadinle2019lac,
  title={Least ambiguous set-valued classifiers with bounded error levels},
  author={Sadinle, Mauricio and Lei, Jing and Wasserman, Larry},
  journal={Journal of the American Statistical Association},
  volume={114},
  number={525},
  pages={223--234},
  year={2019},
  doi={10.1080/01621459.2017.1395341}
}

@inproceedings{romano2020aps,
  title={Classification with valid and adaptive coverage},
  author={Romano, Yaniv and Sesia, Matteo and Cand{\`e}s, Emmanuel J.},
  booktitle={Advances in Neural Information Processing Systems},
  year={2020}
}

@inproceedings{angelopoulos2021raps,
  title={Uncertainty sets for image classifiers using conformal prediction},
  author={Angelopoulos, Anastasios N. and Bates, Stephen and Malik, Jitendra and Jordan, Michael I.},
  booktitle={International Conference on Learning Representations},
  year={2021}
}

@inproceedings{stutz2022conftr,
  title={Learning optimal conformal classifiers},
  author={Stutz, David and Dvijotham, Krishnamurthy and Cemgil, Ali Taylan and Doucet, Arnaud},
  booktitle={International Conference on Learning Representations},
  year={2022}
}

@article{hooker2019forget,
  title={What do compressed deep neural networks forget?},
  author={Hooker, Sara and Courville, Aaron and Clark, Gregory and Dauphin, Yann and Frome, Andrea},
  journal={arXiv preprint arXiv:1911.05248},
  year={2019},
  url={https://arxiv.org/abs/1911.05248}
}

@inproceedings{hong2024decoding,
  title={Decoding compressed trust: Scrutinizing the trustworthiness of efficient {LLMs} under compression},
  author={Hong, Junyuan and Duan, Jinhao and Zhang, Chenhui and Li, Zhangheng and Xie, Chulin and Lieberman, Kelsey and Diffenderfer, James and Bartoldson, Brian R. and Jaiswal, Ajay and Xu, Kaidi and Kailkhura, Bhavya and Hendrycks, Dan and Song, Dawn and Wang, Zhangyang and Li, Bo},
  booktitle={International Conference on Machine Learning},
  series={Proceedings of Machine Learning Research},
  volume={235},
  pages={18611--18633},
  year={2024}
}
\clearpage
\appendix

\section{Notation}
\label{app:notation}

Table~\ref{tab:notation} summarizes the notation used in the main text and experimental protocol. Lower nonconformity scores indicate greater label compatibility; expectations and probabilities use the conditioning specified in each result.

\begin{table*}[t]
\centering
\scriptsize
\setlength{\tabcolsep}{5pt}
\caption{Notation for prediction, calibration, pruning, and analysis.}
\label{tab:notation}
\begin{tabular}{@{}p{0.24\textwidth}p{0.72\textwidth}@{}}
\toprule
Symbol & Meaning \\
\midrule
\multicolumn{2}{@{}l}{\textbf{Data and Prediction}} \\
\(\X,\Y,K\) & Input space, finite label space, and number of labels \(K=|\Y|\). \\
\((x,y),(X,Y)\) & An input-label pair and its random counterpart. \\
\(\theta,\theta'\) & Dense-model parameters and parameters of a fixed pruned model. \\
\(f_\theta,p_\theta\) & Normalized classifier and its class probabilities; decoder sequence probabilities are specified separately. \\
\(s(x,y;\theta)\) & Nonconformity score for a complete candidate label. \\
\(v(y),T_y\) & Fixed label verbalizer and its token length. \\
\(\ell_\theta(x,y)\) & Teacher-forced verbalizer sequence log-likelihood. \\
\(\PP,\E\) & Probability and expectation. \\
\midrule
\multicolumn{2}{@{}l}{\textbf{Data Splits and Calibration}} \\
\(\D_{\mathrm{prune}}\) & Examples used for gradients and pruning statistics. \\
\(\D_{\mathrm{val\text{-}cal}}\) & Validation-calibration split for provisional quantiles and the preliminary dense threshold. \\
\(\D_{\mathrm{val\text{-}eval}}\) & Validation-evaluation split for candidate and checkpoint selection. \\
\(\D_{\mathrm{conf}}\) & Independent final-calibration sample, accessed after model selection. \\
\(\D_{\mathrm{test}}\) & Test sample used only for evaluation. \\
\(\alpha,n,N\) & Target miscoverage, final-calibration sample size, and test sample size. \\
\(k\) & Calibration order-statistic index in Eq.~\ref{eq:k_order}. \\
\(\hat q_\alpha(\theta)\) & Final conformal quantile for model \(\theta\), set to \(+\infty\) when \(k>n\). \\
\(\tilde q\) & Preliminary dense-model true-label quantile used by Threshold-2. \\
\(C_\alpha(x;\theta)\) & Prediction set defined in Eq.~\ref{eq:conformal_set}. \\
\(\bar C_P(\theta)\) & Expected set size over final calibration and a new test input. \\
\(\varepsilon_{\mathrm{size}}\) & Desired set-size slack in the reliability contract. \\
\midrule
\multicolumn{2}{@{}l}{\textbf{Pruning and Selection}} \\
\(B,\rho\) & Base pruner and target sparsity. \\
\(\mask,P(\mask)\) & Binary pruning mask and indices of removed coordinates. \\
\(\theta_\mask\) & Pure-mask parameters \(\mask\odot\theta\). \\
\(A(x)\) & Dense-model candidate-label set used to compute saliency. \\
\(I_{\mathrm{acc}}\) & Coordinate importance supplied by the base pruner. \\
\(I_{\mathrm{cal}}^A\) & Candidate-label importance in Eq.~\ref{eq:candidate_importance}. \\
\(G_{\mathrm{cal}},I_{\mathrm{cal}}\) & Squared true-label gradient factor and coordinate-displacement importance. \\
\(\widetilde I,I_{\mathrm{CSP}}\) & Module-normalized importance and the mixed CSP importance. \\
\(\lambda,\Lambda\) & Mixing coefficient and its fixed search grid. \\
\(\theta^{\mathrm{ref}}\) & Selected absolute true-label gradient reference (\GradRef{}), including its base-pruner fallback. \\
\(E,\lambda^*,\theta^*\) & Eligible coefficients, selected coefficient, and selected model in Algorithm~\ref{alg:cpp}. \\
\(\hat a_\lambda,\hat c_\lambda\) & Validation accuracy and mean set size of candidate \(\theta_\lambda\). \\
\(L_i,\hat\ell,s_L\) & Paired validation accuracy loss, its sample mean, and its standard deviation. \\
\(R\) & Selection record, including the chosen coefficient or reference fallback. \\
\midrule
\multicolumn{2}{@{}l}{\textbf{Theoretical and Statistical Analysis}} \\
\(\mathcal Z,\epsilon_{\mathcal Z}\) & Input-label domain and maximum score change on that domain. \\
\(\epsilon,M\) & Uniform score-perturbation bound and dense-score density bound in Theorem~\ref{thm:set_size_stability}. \\
\(\D_0,z\) & Finite sample used in the displacement analysis and one evaluated example. \\
\(H_z,H\) & Local Hessian norm bound and its supremum over evaluated examples. \\
\(a_\mask(z)\) & Masked linear term in Proposition~\ref{prop:first_order}. \\
\(\bar\Delta_{\D_0}\) & Empirical mean absolute score displacement. \\
\(p,\widehat{\Cov}_N\) & Conditional coverage of a fitted predictor and empirical coverage on \(N\) test examples. \\
\(S,\gamma\) & Number of paired seeds and error level used for a reported interval. \\
\(d^a,d^C\) & Within-seed accuracy difference and relative set-size difference. \\
\(\bar z,s_z\) & Sample mean and standard deviation of seed-level values. \\
\(p_{\mathrm{eq}},p_C\) & TOST equivalence value and sign-flip sensitivity value. \\
\(\G,R(g),\hat R(g)\) & Candidate grid, population risk, and empirical risk in the optional Learn-Then-Test rule. \\
\bottomrule
\end{tabular}
\end{table*}

Some symbols are reused for different quantities. We define their meaning where they are used.

\FloatBarrier

\section{Extended Related Work and Main-Text Details}
\label{app:main_details}

This appendix supplements the main text with related-work comparisons, decoder-score definitions, feasibility checks, selection and evaluation notes, and the squared true-label gradient factor.

\subsection{Related Work Details}
\label{app:related_extended}

The closest predecessor is \citet{pmlr-v179-zhao22a}. For a matched comparison, we adapt the predecessor's hidden-unit deletion criterion and ratio score to shared coordinate blocks and normalized negative log likelihood. All compared methods use the same mask support and scorer; this adaptation differs from the original benchmark.

Structured pruning with recovery changes both architecture and optimization budget. Saliency criteria estimate deletion effects from local derivatives \citep{lecun1989obd,lee2019snip}, and movement pruning learns encoder masks during fine-tuning \citep{sanh2020movement}. Parameter importance also appears in continual learning \citep{kirkpatrick2017overcoming,zenke2017continual}, alongside pruning \citep{mallya2018packnet}, distillation \citep{li2017lwf}, and parameter isolation \citep{wang2022dualprompt,wang2023orthogonal}; surveys review these methods \citep{wu2024csur,wang2024survey}. Our absolute-gradient and squared true-label controls distinguish the effects of generic importance signals from those of candidate selection.

Conformal classifiers differ in the nonconformity score, which trades set size against adaptivity \citep{sadinle2019lac,romano2020aps,angelopoulos2021raps}, and conformal training optimizes set size during learning \citep{stutz2022conftr}; CSP keeps the score fixed and changes only the pruning mask. Shift-aware and risk-controlling procedures rely on different assumptions \citep{park2020calibrated,gibbs2021adaptive,bates2021selective}. CSP's coverage guarantee follows standard split-conformal theory. Appendix~\ref{app:proof} also gives an optional Learn-Then-Test (LTT) rule \citep{angelopoulos2022ltt}; our validation selector does not use it.

\subsection{Decoder Scorers}
\label{app:decoder_scorer}

For a decoder, each label has a fixed verbalizer \(v(y)=(v_1,\ldots,v_{T_y})\). Its teacher-forced sequence log-likelihood is
\[
\ell_\theta(x,y)=\sum_{t=1}^{T_y}
\log p_\theta(v_t\mid x,v_{<t}).
\]
The zero-shot decoder uses the negative sequence sum with fixed verbalizers. These sequence likelihoods need not sum to one across labels. Normalizing them would define a different scorer and could change both saliency and prediction sets. The primary encoder and fine-tuned decoder use normalized class probabilities instead. Every scorer evaluates complete labels rather than individual verbalizer tokens and is frozen before final calibration. Appendix~\ref{app:experimental} specifies the model settings.

\subsection{Recalibration Feasibility Check}
\label{app:feasibility}

Section~\ref{sec:prune_then_recal_fails} states the feasibility check on set size and coverage. Table~\ref{tab:qwen_dbpedia_sparsity_backing} reports decoder means at 50\% sparsity under the separate aggregate-control protocol.

\subsection{Selection and Evaluation Notes}
\label{app:eval_details}

The powered recovery study uses 200 seeds; the secondary sparsity and direct-search panels report their smaller counts. The validation tolerance governs candidate eligibility, while the test margin assesses the mean paired accuracy difference across seeds.

Within each paired seed, methods share the dense task checkpoint, pruning and validation splits, calibration sample, and test cases. Test denominators are 5,000, 3,080, and 4,500 for DBpedia-14, Banking77, and in-scope CLINC150. Appendix~\ref{app:paired_protocol} specifies data roles and task-training allocations. The Banking77 70\% rows show why set size alone is insufficient: their paired accuracy intervals do not establish equivalence within \(\pm0.005\).

To evaluate the proxy directly, we compare Top-2 block saliency with the signed finite-deletion criterion of Zhao and Bellotti. Both use the same class score, 32 coordinate blocks per matrix, and 50\% block sparsity. Each candidate deletion receives a refitted provisional quantile, and final calibration is independent. This matched transformer adaptation differs from the published hidden-unit benchmark; direct search is not an optimal-mask reference.

The first 16 encoder records form a prefix of the 200-record study. Table~\ref{tab:recovery_headline} reports both summaries, but the prefix is not an independent replication.

The controller monitors elapsed time and reserves an allowance for final evaluation. Reused checkpoints must produce identical predictions. Tables~\ref{tab:equal_time_headline} (Section~\ref{sec:experiments}) and \ref{tab:equal_time_recovery} report mean completed-update counts.

\subsection{Squared True-label Gradient Factor}
\label{app:sqtl_details}

The gradient factor of the squared true-label control in Eq.~\ref{eq:cal_importance} is
\begin{equation}
G_{\mathrm{cal}}(\theta_j)
=\frac{1}{|\D_{\mathrm{prune}}|}
\sum_{(x,y)\in\D_{\mathrm{prune}}}
\left(\frac{\partial s(x,y;\theta)}{\partial\theta_j}\right)^2.
\label{eq:cal_gradient_factor}
\end{equation}
For normalized negative log likelihood, \(G_{\mathrm{cal}}\) has the form of a diagonal empirical Fisher estimate \citep{kirkpatrick2017overcoming,zenke2017continual}; for a general score, it measures sensitivity. We compute gradients per example, since squaring a batch mean introduces cross-example terms.

\FloatBarrier

\section{Proof Details and Remarks}
\label{app:proof}

This appendix proves the results in Section~\ref{sec:theory} and adds two statements: a decomposition of calibration and finite-test variability, and an optional LTT rule. The reported model-selection procedure does not use the latter.

To express the set-size stability condition, define the maximum score change over a domain \(\mathcal Z\subseteq\X\times\Y\):
\begin{equation}
\epsilon_{\mathcal Z}(\theta,\theta')
=\sup_{(x,y)\in\mathcal Z}
|s(x,y;\theta')-s(x,y;\theta)|.
\label{eq:epsilon_domain}
\end{equation}

\begin{assumption}[Local smoothness along the pruning path]
\label{assump:smoothness}
For every evaluated example \(z\), \(s(z;\theta)\) is twice differentiable on the segment between dense and pruned parameters, with Hessian operator norm at most \(H_z\); let \(H=\sup_zH_z<\infty\).
\end{assumption}

\begin{proposition}[First-order score displacement]
\label{prop:first_order}
Let \(\theta_\mask=\mask\odot\theta\) be a pure binary-mask perturbation, let \(P(\mask)\) be its removed coordinates, let \(\D_0\) be a finite nonempty evaluation sample, and define
\[
I_{\mathrm{cal},\D_0}(\theta_j)
=\frac{\theta_j^2}{|\D_0|}
\sum_{z\in\D_0}
\left(\frac{\partial s(z;\theta)}{\partial\theta_j}\right)^2.
\]
Under Assumption~\ref{assump:smoothness},
define the masked linear term
\[
a_\mask(z)=\sum_{j\in P(\mask)}
\theta_j\frac{\partial s(z;\theta)}{\partial\theta_j}.
\]
Also write the mean score displacement as
\[
\bar\Delta_{\D_0}(\mask)
=\frac1{|\D_0|}\sum_{z\in\D_0}
|s(z;\theta_\mask)-s(z;\theta)|.
\]
Then
\begin{align}
\bar\Delta_{\D_0}(\mask)
&\leq
\left[
\frac1{|\D_0|}
\sum_{z\in\D_0}
 a_\mask(z)^2
\right]^{1/2}\nonumber\\
&\quad+\frac H2\|\theta_\mask-\theta\|_2^2.
\label{eq:first_order_exact}
\end{align}
If \(m=|P(\mask)|\), the first term is at most
\[
\left[
m\sum_{j\in P(\mask)}
I_{\mathrm{cal},\D_0}(\theta_j)
\right]^{1/2}.
\]
For \(\D_0=\D_{\mathrm{prune}}\), this is the saliency in Eq.~\ref{eq:cal_importance}.
\end{proposition}

\begin{theorem}[Empirical coverage concentration]
\label{thm:empirical_coverage}
Condition on a fitted conformal predictor, including its realized calibration sample. Let \(p\) be its conditional coverage on an independent test example, and let \(\widehat{\Cov}_N\) be coverage on \(N\) independent and identically distributed (i.i.d.) test examples from that distribution. For every \(\delta\in(0,1)\),
\begin{equation}
\PP\!\left\{|\widehat{\Cov}_N-p|>
\sqrt{\frac{\log(2/\delta)}{2N}}\right\}\leq\delta.
\label{eq:coverage_concentration}
\end{equation}
This statement does not imply \(p\geq1-\alpha\) for the realized calibration sample. If, additionally, the \(n\) calibration examples and test examples are i.i.d.\ and independent of model selection, then for \(k\leq n\) and every \(\eta\in(0,1)\),
\begin{equation}
\PP\!\left\{p<\frac{k}{n}
-\sqrt{\frac{\log(1/\eta)}{2n}}\right\}\leq\eta.
\label{eq:conditional_calibration_tail}
\end{equation}
Consequently, with probability at least \(1-\eta-\delta\) over both calibration and testing,
\begin{equation}
\begin{aligned}
\widehat{\Cov}_N\geq \frac{k}{n}
&-\sqrt{\frac{\log(1/\eta)}{2n}}\\
&-\sqrt{\frac{\log(2/\delta)}{2N}}.
\end{aligned}
\label{eq:joint_coverage_concentration}
\end{equation}
For \(k>n\), the prediction set is \(\Y\) and coverage is one. The extra i.i.d.\ assumption in this theorem is stronger than the exchangeability needed for Theorem~\ref{thm:coverage}.
\end{theorem}

\begin{theorem}[Optional finite-grid Learn-Then-Test rule]
\label{thm:ltt}
Condition on all data used to construct a finite nonempty grid \(\G\) and its candidate loss functions. Let an independent i.i.d.\ validation sample \(\D_{\mathrm{val}}\) be used only to evaluate bounded losses \(\ell_g\in[0,1]\), with population risk \(R(g)\), empirical risk \(\hat R(g)\), target \(r\in[0,1]\), and error level \(\delta\in(0,1)\). Define
\begin{equation}
p_g=\begin{cases}
\exp\{-2|\D_{\mathrm{val}}|(r-\hat R(g))^2\},&\hat R(g)<r,\\
1,&\hat R(g)\geq r.
\end{cases}
\label{eq:ltt_pvalue}
\end{equation}
Accepting only candidates with $p_g\leq\delta/|\G|$ ensures, with probability at least $1-\delta$, that every accepted candidate has $R(g)\leq r$.
\end{theorem}

Theorem~\ref{thm:ltt} is an optional finite-grid alternative. Its testing sample cannot also be used to fit candidate-specific quantiles, temperatures, or score functions. Those objects must be fixed on other data before the bounded losses are evaluated. Algorithm~\ref{alg:cpp} uses the accuracy-matching rule in Section~\ref{sec:selection}; the reported ablations use the base-pruner guardrail in Appendix~\ref{app:reported_ablations}. Neither is assigned an LTT guarantee.

\begin{proof}[Proof of Theorem~\ref{thm:coverage}]
Condition on every pruning and selection decision. Independence in Assumption~\ref{assump:exchangeability} implies that the \(n+1\) scores \(S_i=s(X_i,Y_i;\theta')\) remain exchangeable. If \(k>n\), the infinite quantile includes every label and the claim is immediate. Otherwise, break score ties using independent continuous auxiliary random variables solely for the rank argument. The test score has a uniform rank among \(n+1\) exchangeable scores under this tie-breaking. Rank at most \(k\) implies \(S_{n+1}\leq S_{(k)}\), where \(S_{(k)}\) is the calibration order statistic. Including all boundary ties can only increase coverage. Hence
\[
\PP\{S_{n+1}\leq S_{(k)}\}\geq\frac{k}{n+1}\geq1-\alpha.
\]
Without score ties, the first inequality is equality, and
\(k/(n+1)<1-\alpha+1/(n+1)\). Averaging over the pruning and selection data proves the marginal result. The calibration sample remains random in this probability.
\end{proof}

\begin{proof}[Proof of Theorem~\ref{thm:empirical_coverage}]
Conditional on the fitted predictor and its calibration sample, the test indicators \(\mathbf{1}\{Y_i\in C_\alpha(X_i;\theta')\}\) are independent Bernoulli variables with mean \(p\). Hoeffding's inequality gives
\[
\PP\{|\widehat{\Cov}_N-p|>t\}\leq2\exp(-2Nt^2).
\]
Substituting \(t=\sqrt{\log(2/\delta)/(2N)}\) proves Eq.~\ref{eq:coverage_concentration}. This step alone gives no nominal-level lower bound because \(p\) is conditional on the calibration sample.

For the calibration bound, condition only on the independently selected model and let \(F\) be the cumulative distribution function (CDF) of its true-label score. For \(k\leq n\), \(p=F(S_{(k)})\). Put \(a=k/n-\sqrt{\log(1/\eta)/(2n)}\). If \(a\leq0\), the lower-tail event is empty. Otherwise, for each calibration score, \(\PP\{F(S_i)<a\}\leq a\); this inequality holds for CDFs with atoms as well as continuous ones. The event \(F(S_{(k)})<a\) implies that at least \(k\) of the independent indicators \(\mathbf{1}\{F(S_i)<a\}\) equal one. Their means are at most \(a\), so Hoeffding's inequality yields
\[
\PP\{p<a\}\leq\exp\{-2n(k/n-a)^2\}=\eta.
\]
This proves Eq.~\ref{eq:conditional_calibration_tail}. Applying the conditional test bound, averaging over calibration, and taking a union bound gives Eq.~\ref{eq:joint_coverage_concentration}. No independence between these two failure events is required. If \(k>n\), all labels are included.
\end{proof}

\begin{proof}[Proof of Lemma~\ref{lem:order_stat_stability}]
For every $i$,
\[
a_i-\epsilon\leq b_i\leq a_i+\epsilon.
\]
At least $k$ of the $a_i$'s are no larger than $a_{(k)}$, so at least $k$ of the $b_i$'s are no larger than $a_{(k)}+\epsilon$, giving $b_{(k)}\leq a_{(k)}+\epsilon$. Reversing the roles of $a$ and $b$ yields $a_{(k)}\leq b_{(k)}+\epsilon$.
\end{proof}

\begin{proof}[Proof of Theorem~\ref{thm:set_size_stability}]
All probabilities and expectations below condition on the two models and their common calibration sample. By Lemma~\ref{lem:order_stat_stability}, the calibration quantiles satisfy $\hat q_\alpha(\theta')\leq \hat q_\alpha(\theta)+\epsilon$. If $y\in C_\alpha(x;\theta')$, then $s(x,y;\theta')\leq \hat q_\alpha(\theta')$, and Eq.~\ref{eq:uniform_score_bound} gives $s(x,y;\theta)\leq s(x,y;\theta')+\epsilon$, hence $s(x,y;\theta)\leq \hat q_\alpha(\theta)+2\epsilon$. This proves the containment in Eq.~\ref{eq:set_containment}. Define the per-example set-size inflation as
\[
\Delta_C(x)=|C_\alpha(x;\theta')|-|C_\alpha(x;\theta)|.
\]
Subtracting the dense set leaves only labels whose dense scores lie in the band above the dense threshold, so

\[
\begin{aligned}
\Delta_C(x)
&\leq
\sum_{y \in \mathcal{Y}}
\mathbf{1}\!\left\{
\hat q_\alpha(\theta)
< 
s(x,y;\theta)
 \right.\\
&\qquad\qquad\left.
\leq
\hat q_\alpha(\theta) + 2\epsilon
\right\}.
\end{aligned}
\]
Taking expectations and using the density bound for each label gives Eq.~\ref{eq:set_size_bound}.
\end{proof}

\begin{proof}[Proof of Proposition~\ref{prop:first_order}]
For a binary pruning mask $\mask$, let $\theta_\mask=\mask\odot\theta$; the perturbation $\delta=\theta_\mask-\theta$ satisfies $\delta_j=-\theta_j$ for $j\in P(\mask)$ and zero otherwise. Taylor's theorem gives
\[
s(z;\theta_\mask)-s(z;\theta) = \nabla_\theta s(z;\theta)^\top(\theta_\mask-\theta)+r_z,
\]
with the residual bounded by
\[
|r_z|\leq \tfrac{1}{2}H_z\|\theta_\mask-\theta\|_2^2.
\]
Because $\theta_\mask-\theta$ is nonzero only on pruned coordinates, its linear term is
\[
-\sum_{j\in P(\mask)}\theta_j
\frac{\partial s(z;\theta)}{\partial\theta_j}.
\]
Averaging absolute values bounds the mean displacement by \(|\D_0|^{-1}\sum_z|a_\mask(z)|+(H/2)\|\theta_\mask-\theta\|_2^2\). Cauchy-Schwarz over examples then yields Eq.~\ref{eq:first_order_exact}. For the diagonal bound, apply the inequality
\[
\left(\sum_{j=1}^m a_j\right)^2\leq m\sum_{j=1}^m a_j^2
\]
for each $z$, average over $\D_0$, and substitute the definition of $I_{\mathrm{cal},\D_0}$.
\end{proof}

\begin{proof}[Proof of Theorem~\ref{thm:ltt}]
Condition on the candidate-construction data. For a fixed $g$ with $R(g)>r$, independent validation losses and Hoeffding's inequality imply
\[
\PP\{\hat R(g)\leq r-t\}\leq \exp(-2|\D_{\mathrm{val}}|t^2)
\]
for all $t>0$, showing that $p_g$ is a valid conservative p-value for the null hypothesis $R(g)>r$. A union bound over the finite grid $\G$ shows that, with probability at least $1-\delta$, no candidate with $R(g)>r$ is accepted at the Bonferroni threshold $\delta/|\G|$. Any data-dependent choice among accepted candidates therefore satisfies $R(g)\leq r$ on the same event.
\end{proof}

\begin{remark}[Actual perturbations and SparseGPT compensation]
For any pruned output \(\theta'=\theta+\delta\), write
\[
\bar\Delta_{\D_0}(\delta)
=\frac1{|\D_0|}\sum_{z\in\D_0}|s(z;\theta')-s(z;\theta)|.
\]
The same Taylor argument gives
\begin{equation}
\begin{aligned}
\bar\Delta_{\D_0}(\delta)
&\leq\left[\frac1{|\D_0|}\sum_{z\in\D_0}
  (\nabla s(z;\theta)^\top\delta)^2\right]^{1/2}\\
&\quad+\frac H2\|\delta\|_2^2.
\end{aligned}
\label{eq:actual_update_bound}
\end{equation}
The smoothness bound must hold on the actual segment from \(\theta\) to \(\theta'\). If \(\delta=-\mathbf{1}_{P(\mask)}\odot\theta+c\), where \(c\) represents compensation on retained coordinates, Minkowski's inequality bounds the first term by
\[
\begin{aligned}
&\left[\frac1{|\D_0|}\sum_z a_\mask(z)^2\right]^{1/2}\\
&\qquad+\left[\frac1{|\D_0|}\sum_z(\nabla s(z;\theta)^\top c)^2\right]^{1/2}.
\end{aligned}
\]
Equation~\ref{eq:actual_update_bound} follows by applying Taylor's theorem with \(\delta\), then the same triangle and Cauchy-Schwarz inequalities as above. SparseGPT changes retained weights through reconstruction updates \citep{frantar2023sparsegpt}; the second term cannot be dropped without a bound on those updates. No empirical uniform-score certificate is asserted for the evaluated checkpoints.
\end{remark}

\begin{remark}[Softmax score smoothness]
For normalized logits \(u\), \(s_y(u)=-u_y+\log\sum_c\exp u_c\) has gradient \(\operatorname{softmax}(u)-e_y\), which is bounded in logit space even when the true-label probability is small. The derivative of \(-\log p_y\) with respect to the probability coordinate is unbounded near zero, but this alone does not prove a lack of smoothness in logits or network parameters. Parameter-space control additionally depends on the network Jacobian and Hessian along the pruning path. Proposition~\ref{prop:first_order} therefore states a local smoothness assumption explicitly; finiteness of an empirical sample alone does not verify it. Score perturbations measured on validation examples can diagnose changes, but do not establish a supremum over unseen candidate-label pairs.
\end{remark}

\begin{remark}[Why a magnitude-threshold bound is not used]
A threshold on CSP importance is not a simple threshold on weight magnitude. A statement such as
\[
\|\theta-\theta'\|_2\leq \sqrt{sd}\tau
\]
is only valid if $\tau$ is an explicit magnitude threshold on pruned weights. The analysis avoids this conflation by using either the actual perturbation $\|\theta_\mask-\theta\|_2$ or the first-order score displacement shown in Eq.~\ref{eq:first_order_exact}.
\end{remark}

\begin{remark}[Weighted candidate averages]
The Taylor and diagonal bounds hold for any fixed nonnegative weights \(w_z\) summing to one, with every empirical average replaced by \(\sum_z w_z(\cdot)\). Weighted Cauchy-Schwarz gives \(\sum_z w_z|a_z|\leq(\sum_z w_z a_z^2)^{1/2}\), and the remaining proof is unchanged. CSP uses \(w_{xy}=1/(|\D_{\mathrm{prune}}||A(x)|)\) for retained labels. This matches Eq.~\ref{eq:candidate_importance}, while still providing only an empirical average bound rather than a uniform guarantee on unseen pairs.
\end{remark}

\section{Experimental Protocol}
\label{app:experimental}
\label{app:paired_protocol}

\subsection{Data and Classifiers}

Each paired seed reserves 256 training examples for pruning, 512 for provisional calibration, 2,048 for validation evaluation, and 2,048 for final-calibration prefixes. Sampling is without replacement, and source identities do not overlap across roles. Default final calibration uses the first 1,024 examples in its reserved pool.

The test sets are fixed across seeds: 5,000 DBpedia-14 examples, 3,080 Banking77 examples, and 4,500 in-scope CLINC150 examples. Redrawn seeds measure variation conditional on these corpora, not variation across independent task populations.

The encoder is \nolinkurl{FacebookAI/roberta-base}~\citep{liu2019roberta}. We fine-tune its backbone and normalized classification head jointly on a disjoint 16,000-example DBpedia subset, the remaining 5,139 Banking77 training examples, or the remaining 10,136 CLINC150 training examples. The official CLINC validation split is unused.

Training uses three AdamW epochs, learning rate \(2\times10^{-5}\), batch size 32, weight decay 0.01, warmup fraction 0.06, and maximum input length 256. Every pruning arm within a seed starts from the same trained checkpoint. Reproduction requires immutable model, tokenizer, software, and dataset revisions.

The fine-tuned decoder (D) starts from \nolinkurl{Qwen/Qwen2.5-1.5B}~\citep{yang2024qwen25}. We truncate the input, append an end-of-sequence token, and attach a dense \(K\)-class linear head to the final nonpadding state. The score is negative log softmax at temperature one.

We train the full backbone and head on the same task allocation for three epochs. Settings are learning rate \(10^{-5}\), effective batch size 32, microbatch four, eight accumulation steps, activation checkpointing, and the encoder's remaining optimizer settings. This is a classification comparison; training cost is not matched to RoBERTa.

The zero-shot decoder (Q) uses the Qwen base checkpoint and prompt \texttt{Classify the text by <task>. Text: <input>. Label:}. The task strings are \texttt{encyclopedic entity type}, \texttt{banking intent}, and \texttt{assistant intent}. We left-truncate the input while preserving the instruction, keeping the prefix within 256 tokens.

Each complete label verbalizer starts with a space. We sum its teacher-forced log probabilities without length normalization, across-label normalization, or an appended end-of-sequence token. Intent verbalizers preserve canonical spelling, case, and punctuation, replacing underscores with spaces. DBpedia uses its canonical fourteen class names in fixed lexical order. Q requires no task training, and its dense accuracy on the fixed test set is unchanged by calibration redraws.

\subsection{Mask Construction and Reference Selection}

The base pruner determines eligible attention and MLP matrices, layerwise sparsity allocation, and reconstruction. All arms share module normalization and the mixing grid \(\{0,0.1,0.3,0.5\}\).

We first screen the supervised-gradient reference against the unaugmented base checkpoint using Eq.~\ref{eq:validation_buffer}. Selection minimizes provisional set size, with CSP's 0.02-label tie window and accuracy/coefficient tie rules. If no candidate qualifies, we retain the base. CSP and SqTL then screen their candidates against this selected reference. Reuse of validation data makes this an operational selection rule rather than a post-selection guarantee.

Top-2, Threshold-2, and Random-2 each use two distinct labels per pruning example. We fix candidate labels before sweeping coefficients. Random-2 draws without replacement from a separate seed stream and reuses those draws across candidates. Threshold-2 uses the dense true-label quantile from provisional calibration and never accesses final-calibration examples.

SqTL evaluates the true-label gradient once. Repeating the same derivative would leave its squared importance unchanged, so duplicated evaluations are not charged as necessary work.

\section{Recovery and Resource Comparisons}
\label{app:powered_protocol}

\subsection{Shared Recovery Trajectories}

The powered study uses 200 paired seeds in six encoder task/base cells at 50\% sparsity. \GradRef{}, SqTL, Top, and Thr share the checkpoint, data roles, minibatch order, and optimizer random stream. After selecting the initial mask, we reset optimizer state and train the surviving backbone weights and dense head with masked AdamW. Settings are learning rate \(10^{-5}\), batch size 32, weight decay 0.01, and no warmup.

Removed coordinates, their gradients, and their optimizer moments remain zero. We reapply the mask after every update, without adapters or weight regrowth.

We save checkpoints at updates 0, 25, 50, and every multiple of 50 thereafter. Fixed budgets stop at 100, 250, 500, or 1,000 updates; deadline runs can continue to 5,000. A longer budget extends the same trajectory and includes all earlier checkpoints.

At each saved checkpoint, we refit the provisional quantile and screen against that arm's update-zero model using Eq.~\ref{eq:validation_buffer}. Among eligible checkpoints within 0.02 labels of the smallest validation set, we select the highest accuracy, then the earliest update. If no saved checkpoint is eligible, we retain the update-zero model. Final calibration follows checkpoint selection. Reused checkpoints must have identical accuracy, coverage, and set-size outputs. All scheduled work counts toward cost.

The 16-seed encoder recovery panel uses the first 16 records of the 200-seed study and adds no hypothesis tests. Table~\ref{tab:recovery_headline} displays both summaries. Decoder recovery is separate: it uses the same optimizer and only checkpoints 0, 25, 50, and 100. When an arm falls back entirely to its reference, it shares the reference's recovery trajectory, validation choice, and predictions but retains its own construction cost.

Restricted-data recovery starts from the same task-fine-tuned checkpoint but excludes the original training corpus. \GradRef{} and SqTL may cycle the retained 256 labeled pruning examples for 1,000 updates. Validation and calibration examples are never used for gradients. Checkpoint screening and selection follow the shared rules. Top performs no recovery and reuses its powered-study outputs. This condition permits recovery on the retained labels while withholding the larger training pool.

\subsection{Timing Boundary and Deadlines}

Timing uses one A100 80GB, bfloat16 (BF16) weights and forward computation, 32-bit floating-point (FP32) saliency and reconstruction accumulation, input length 256, pruning microbatch one, and evaluation batch 16. We synchronize device work at phase boundaries.

\GradRef{} evaluates four coefficient candidates. SqTL and candidate-label arms evaluate the reference plus their own grid, constructing the shared zero-mixing model once for seven distinct candidates. Total time includes checkpoint loading, dense label scoring, gradients, mask construction, reconstruction, provisional calibration, selection, optimizer setup and updates, checkpoint evaluations, and one final calibration/test pass. Shared task fitting is reported separately and added equally for full model-development comparisons.

Each cell has three fixed deadlines, listed in Table~\ref{tab:equal_time_recovery}. Every method receives the same deadline, so cheaper mask construction leaves more recovery time. A hardware preflight with matching tensor shapes reserves time for final evaluation without inspecting task quality. The controller starts no operation that would consume this reserve.

We report completed updates and elapsed times, without idle padding or performance interpolation. Timeouts stay in the attempted-run denominator and force nonrejection of the joint claim. Reaching the 5,000-update limit with unused time also prevents a matched-elapsed-time claim.

We account for recovery time using
\[
t_{iab}=t_{ia0}+o_i+b u_i+M(b)v_i\quad(b>0),
\]
where \(t_{ia0}\) includes initial selection and one terminal calibration/test pass, \(o_i\) is optimizer setup, \(u_i\) is per-update time, and \(v_i\) is time per saved-checkpoint evaluation. \(M(b)\) counts positive checkpoints at or below \(b\). This decomposition separates initial pruning cost from optimizer setup, updates, and checkpoint evaluation; phase timings are reported separately. Lower pruning cost does not establish sparse-kernel throughput, energy savings, or downstream latency.

\subsection{Paired Inference and Power}

For \(S\) paired seeds, use the Student-\(t\) interval
\begin{equation}
\bar z\ \pm\ t_{1-\gamma/2,S-1}s_z/\sqrt S.
\label{eq:seed_interval}
\end{equation}
Metric and size intervals use \(\gamma=0.05\); paired accuracy and elapsed-time intervals use \(\gamma=0.10\).

Within each seed, define \(d^a=a_A-a_B\) and \(d^C=(c_A-c_B)/c_B\). Relative-size estimates average these within-seed ratios rather than taking a ratio of means.

A zero denominator or degenerate estimated SD makes the relevant parametric test indeterminate; the claim remains a nonrejection. The approximation assumes an adequate normal model for independent seed differences conditional on the fixed corpora.

Write \(T(z;b)=(\bar z-b)/(s_z/\sqrt S)\). For margin \(m\), TOST yields the equivalence value
\[
p_{\rm eq}(z;m)=\max\!\left\{
\begin{aligned}
&1-F_{t_{S-1}}(T(z;-m)),\\
&F_{t_{S-1}}(T(z;m))
\end{aligned}\right\}.
\]
Accuracy uses \(m=0.005\). Meaningful size benefit tests \(\E d^C<-0.02\), with value \(F_{t_{199}}(T(d^C;-0.02))\); practical size equivalence uses \(m=0.02\). The joint value is the maximum of the accuracy and size values. Equal-time claims additionally include elapsed-time TOST with margin 1\% of the deadline. Failure to reject benefit is not evidence of equivalence.

Holm correction covers 270 primary joint claims: 180 fixed-update decisions (six cells, five budgets, three contrasts, two endpoints), 72 equal-time decisions (six cells, three deadlines, two contrasts, two endpoints), 12 restricted-data benefit claims, and six cross-budget cost claims.

Fixed-update pairs are Top versus \GradRef{}, Top versus SqTL, and Thr versus Top; deadline pairs are Top versus \GradRef{} and Top versus SqTL. Restricted-data claims compare Top-0 with each recovered baseline; their zero-update comparisons are already counted. The supplementary Top-100/\GradRef{}-1000 cost test instead requires accuracy equivalence, relative-size noninferiority within 5\%, and lower full time.

At ordered rank \(j\), report \(\min\{1,\max_{k\leq j}(271-k)p_{(k)}\}\), retaining all unfavorable and inconclusive cases.

The separate 16-seed family contains 96 joint claims: 42 Threshold/reference pairs, 24 candidate replacements, 12 decoder recovery pairs, and 18 block comparisons. For an absolute size difference \(d\), combine accuracy TOST, the one-sided paired size \(t\)-test, and the exhaustive sign-flip sensitivity value
\begin{equation}
p_C=2^{-S}\sum_{\sigma\in\{-1,1\}^{S}}
\mathbf1\{|S^{-1}\textstyle\sum_i\sigma_i d_i|\geq|\bar d|\}.
\label{eq:paired_exact_test}
\end{equation}
Set the raw joint value to one for a nonnegative mean difference; otherwise take the maximum of the three values and apply Holm over all 96. The sign-flip interpretation requires within-pair sign symmetry; it is not an assumption-free test of equal means. No family-wise guarantee is asserted across the two separately corrected studies. The 16-record encoder prefix adds no tests.

Power planning uses the conservative level \(\alpha_*=0.05/270\). For a true 4\% benefit against the 2\% minimum, compute power from the noncentral \(t\) distribution. For equivalence, let \(V\sim\chi^2_{S-1}\), \(h(V)=t_{1-\alpha_*,S-1}\sigma\sqrt{V/(S-1)}/\sqrt S\), and integrate
\[
\E_V\!\left[\left\{
\begin{aligned}
&\Phi\!\left(\frac{0.02-h(V)-\mu}{\sigma/\sqrt S}\right)\\
&\quad-\Phi\!\left(\frac{-0.02+h(V)-\mu}{\sigma/\sqrt S}\right)
\end{aligned}
\right\}_+\right].
\]
Table~\ref{tab:recovery_power} gives the sensitivity to \(S\) and \(\sigma\), with \(\mu=0.004\). These are size-endpoint calculations; joint power also depends on accuracy matching and deadline completion.

Planning calculations quantify the sensitivity of these decisions to seed variation. With paired relative-size SD 0.05, 200 seeds give 97.8\% analytical power for a true 4\% benefit against the 2\% minimum, and 81.3\% for equivalence at a true 0.4\% difference. Power decreases as seed variation grows. Table~\ref{tab:recovery_power} reports the sensitivity; these size-endpoint calculations exclude the additional accuracy and elapsed-time requirements of joint decisions.

\section{Candidate, Calibration, and Search Controls}
\label{app:additional_controls}
\label{app:csp_extended}

\subsection{Calibration and Gradient Budgets}

For a fixed model with continuous true-label score CDF \(F\), i.i.d. calibration gives \(P=F(S_{(k)})\sim\operatorname{Beta}(k,n+1-k)\), with \(p_0=k/(n+1)\). Conditional on \(P\), the coverage count on \(N\) independent test cases is binomial. Total variance therefore gives
\begin{equation}
\operatorname{Var}(\widehat{\Cov}_N)
=p_0(1-p_0)\left\{N^{-1}+\frac{1-N^{-1}}{n+2}\right\}.
\label{eq:coverage_variability}
\end{equation}
At \(n=1024\), \(N=5000\), and nominal 90\% coverage, the calculated SD is about 0.0103; at \(n=2048\), it is about 0.0079. These are scale diagnostics under the stated model, not lower bounds on observed seed variation.

The calibration sweep fixes each seed's score matrix and uses nested prefixes of 256, 512, 1,024, and 2,048 examples from an independent calibration pool. The model, prompt, and test scores do not change. Coverage and set size are threshold step functions, so their changes with calibration size need not be monotone. The 1,024-prefix results in Figure~\ref{fig:calibration_size} match the Q1W50 Threshold-2 panel. Figure~\ref{fig:threshold_plateaus} displays adjacent plateaus.

The encoder budget study uses 32 seeds and varies two budgets separately. One sweep uses nested pruning prefixes of 64, 128, and 256 examples with two labels per example. The other uses 256 examples with one, two, four, or eight top labels. Their shared 256/2 setting is the first 32 powered zero-update Top records. Examples keep their data roles. Figure~\ref{fig:runtime_phases} separates dense scoring, gradients, selected-mask construction, remaining grid candidates and validation, and final evaluation.

\subsection{Direct Efficiency Search}

The direct-search control adapts Zhao and Bellotti's signed and absolute hidden-unit deletion criteria \citep{pmlr-v179-zhao22a} to normalized class negative log likelihood and 32 equal, seed-fixed coordinate blocks per eligible matrix. Each matrix must be divisible into these blocks. Every arm removes 16 blocks at 50\% sparsity. Norms, embeddings, biases, and class heads stay dense; retained weights are unchanged. No arm uses reconstruction or recovery.

We fit the dense model's provisional quantile on 512 validation-calibration examples and measure its mean set size \(c_0\) on the 256 pruning examples. We then delete each block \(b\) independently, refit the quantile, measure \(c_b\), and restore dense weights before evaluating the next block.

ZB-sign ranks \(c_b-c_0\); ZB-abs ranks its absolute value. Each removes the 16 lowest-ranked blocks per matrix, breaking ties by block index. Deletions are evaluated individually and applied simultaneously, without reranking after removal.

Magnitude-block sums absolute weights within each block. Top-2 and Threshold-block sum coordinate saliencies, normalize within each matrix, and mix with magnitude over the same four coefficients. Validation uses magnitude-block as the reference; a rejected assembled ZB mask falls back to magnitude. We freeze the selected model before independent final calibration.

RoBERTa requires 2,304 candidate deletions across 72 matrices; Qwen requires 6,272 across 196. Each deletion scores all 768 provisional-calibration/pruning examples. Standalone ZB time includes the full search, even when deletion scores can be shared between criteria. This matched-support comparison measures proxy cost and regret; it neither reproduces the native hidden-unit benchmark nor identifies an optimal mask.

\subsection{Reported True-label Ablations}
\label{app:reported_ablations}
\label{app:extended_results}

The aggregate controls use separate protocols and report descriptive means. TL-CEP denotes the true-label saliency variants in these controls, using the squared displacement importance in Eq.~\ref{eq:cal_importance}. The top-3, Band-CEP, and all-label variants change the label set. Band-CEP retains candidates within the 20th-percentile distance from a preliminary threshold, caps the set at eight, and falls back to the most compatible label when empty. Its variable gradient budget differs from Threshold-2.

These controls use 512 pruning examples and 1,024 calibration examples. Coefficient selection permits a 0.02 validation-accuracy loss and requires coverage at least 0.88. This protocol is distinct from the powered study.

The five-seed DBpedia TL/Wanda accuracy of 0.295 and the three-seed gradient-control accuracy of 0.297 come from different aggregates. The SparseGPT candidate panel combines a five-seed TL reference with three-seed variants, and its cost ratios measure saliency work rather than the complete pipeline. Split redraws form a separate five-seed partition diagnostic.

Repeated backing values and directional win counts add no observations. The fixed dense Llama perplexity of 6.12 has no seed interval. All these aggregates are descriptive means without additional inferential intervals.

Reproducing the reported results requires code, immutable revisions, disjoint split identifiers, selected checkpoints, seed-level predictions, and measured phase timings.

\FloatBarrier

\section{Supporting Comparisons}
\label{app:supporting}

The supporting tables group recovery and data access, candidate choice and direct search, and calibration and encoder cost. They keep the default Top-2 rule, the Threshold-2 alternative, and true-label controls distinct. R/D/Q denote the encoder, fine-tuned decoder, and zero-shot decoder; task digits 1/2/3 denote DBpedia-14, Banking77, and CLINC150. W/S denote Wanda and SparseGPT, followed by sparsity. Smaller sets or favorable mean differences alone do not establish matched accuracy.

\begin{table*}[!tbp]
\centering
\scriptsize
\setlength{\tabcolsep}{3pt}
\caption{Top-2 versus the selected absolute-gradient reference (\GradRef{}) across encoder sparsities, with 16 paired seeds. Entries list \GradRef{}/Top-2 means. Accuracy and set-size differences have paired 90\% and 95\% Student-\(t\) confidence intervals, respectively. R1/R2/R3 denote RoBERTa on DBpedia-14/Banking77/CLINC150; W/S denote Wanda/SparseGPT. The separate powered study uses 200 seeds at 50\% sparsity. Bold indicates the best value per row and metric; \(\uparrow\)/\(\downarrow\) indicate higher/lower is better. Coverage is not ranked.}
\label{tab:csp_top_grid}
\CSPfitTable{\textwidth}{%
\begin{tabular}{lccccc}
\toprule
Cell & Accuracy \(\uparrow\) & Coverage & Set size \(\downarrow\) & \(\Delta a\) [90\%] \(\uparrow\) & \(\Delta|C|\) [95\%] \(\downarrow\)\\
\midrule
R2S30 & \({\mathbf{0.9290}/0.9277}\) & \({0.8985/0.8980}\) & \({3.02/\mathbf{2.59}}\) & \({-0.0013\;[-0.0042,0.0017]}\) & \({-0.434\;[-0.517,-0.351]}\)\\
R2W30 & \({\mathbf{0.9230}/0.9211}\) & \({0.8996/0.8990}\) & \({3.25/\mathbf{2.82}}\) & \({-0.0019\;[-0.0047,0.0010]}\) & \({-0.425\;[-0.512,-0.338]}\)\\
R2S50 & \({0.9160/\mathbf{0.9172}}\) & \({0.9012/0.9013}\) & \({3.57/\mathbf{3.09}}\) & \({0.0012\;[-0.0019,0.0043]}\) & \({-0.488\;[-0.661,-0.315]}\)\\
R2W50 & \({0.9060/\mathbf{0.9071}}\) & \({0.9022/0.9025}\) & \({3.95/\mathbf{3.44}}\) & \({0.0012\;[-0.0016,0.0039]}\) & \({-0.511\;[-0.671,-0.351]}\)\\
R2S70 & \({\mathbf{0.8580}/0.8525}\) & \({0.9003/0.9000}\) & \({6.34/\mathbf{5.46}}\) & \({-0.0055\;[-0.0089,-0.0022]}\) & \({-0.881\;[-1.177,-0.585]}\)\\
R2W70 & \({\mathbf{0.8340}/0.8309}\) & \({0.9049/0.9070}\) & \({7.19/\mathbf{6.41}}\) & \({-0.0031\;[-0.0061,-0.0001]}\) & \({-0.783\;[-1.057,-0.509]}\)\\
R3S30 & \({\mathbf{0.9480}/0.9472}\) & \({0.9023/0.9029}\) & \({2.46/\mathbf{2.11}}\) & \({-0.0008\;[-0.0033,0.0017]}\) & \({-0.356\;[-0.415,-0.296]}\)\\
R3W30 & \({\mathbf{0.9440}/0.9438}\) & \({0.9035/0.9051}\) & \({2.63/\mathbf{2.28}}\) & \({-0.0002\;[-0.0028,0.0025]}\) & \({-0.348\;[-0.415,-0.282]}\)\\
R3S50 & \({\mathbf{0.9380}/0.9378}\) & \({0.8984/0.8985}\) & \({2.93/\mathbf{2.45}}\) & \({-0.0002\;[-0.0027,0.0023]}\) & \({-0.483\;[-0.630,-0.335]}\)\\
R3W50 & \({\mathbf{0.9300}/0.9297}\) & \({0.9037/0.9038}\) & \({3.30/\mathbf{2.82}}\) & \({-0.0003\;[-0.0034,0.0028]}\) & \({-0.478\;[-0.629,-0.326]}\)\\
R3S70 & \({\mathbf{0.8860}/0.8840}\) & \({0.8972/0.8971}\) & \({5.21/\mathbf{4.51}}\) & \({-0.0020\;[-0.0055,0.0015]}\) & \({-0.702\;[-0.894,-0.510]}\)\\
R3W70 & \({0.8660/\mathbf{0.8669}}\) & \({0.8964/0.8960}\) & \({6.01/\mathbf{5.20}}\) & \({0.0009\;[-0.0019,0.0037]}\) & \({-0.809\;[-1.036,-0.581]}\)\\
R1S30 & \({\mathbf{0.9620}/0.9611}\) & \({0.9028/0.9035}\) & \({1.63/\mathbf{1.39}}\) & \({-0.0008\;[-0.0035,0.0018]}\) & \({-0.238\;[-0.281,-0.195]}\)\\
R1W30 & \({\mathbf{0.9600}/0.9585}\) & \({0.8997/0.8986}\) & \({1.70/\mathbf{1.48}}\) & \({-0.0015\;[-0.0036,0.0006]}\) & \({-0.216\;[-0.262,-0.171]}\)\\
R1S50 & \({0.9540/\mathbf{0.9544}}\) & \({0.8990/0.8991}\) & \({1.97/\mathbf{1.60}}\) & \({0.0004\;[-0.0024,0.0032]}\) & \({-0.377\;[-0.473,-0.280]}\)\\
R1W50 & \({0.9510/\mathbf{0.9511}}\) & \({0.8981/0.8979}\) & \({2.06/\mathbf{1.69}}\) & \({0.0001\;[-0.0028,0.0029]}\) & \({-0.367\;[-0.466,-0.267]}\)\\
R1S70 & \({\mathbf{0.9250}/0.9246}\) & \({0.9022/0.9018}\) & \({2.93/\mathbf{2.51}}\) & \({-0.0004\;[-0.0023,0.0015]}\) & \({-0.428\;[-0.552,-0.303]}\)\\
R1W70 & \({\mathbf{0.9110}/0.9100}\) & \({0.9035/0.9041}\) & \({3.36/\mathbf{2.96}}\) & \({-0.0010\;[-0.0037,0.0017]}\) & \({-0.408\;[-0.531,-0.285]}\)\\
\bottomrule
\end{tabular}}
\end{table*}


\subsection{Recovery and Access to Training Data}

Zero-update, fixed-update, and deadline panels share one set of 200-seed records. The 16-seed prefix is descriptive. Under the common timing boundary, cheaper reference construction permits additional recovery updates, so fixed-update and equal-time comparisons answer different questions.

\begin{table*}[!tbp]
\centering\scriptsize
\setlength{\tabcolsep}{3pt}
\caption{Encoder results without recovery, with 200 paired seeds. Top, \GradRef{}, and SqTL denote Top-2, the absolute-gradient reference, and the squared true-label control. Top accuracy and coverage are means with 95\% Student-\(t\) confidence-interval half-widths. Relative size compares Top with SqTL and has a 95\% confidence interval. The joint benefit test combines accuracy equivalence and the size endpoint, with Holm correction over 270 claims. Bold indicates the smallest set size per row; \(\uparrow\)/\(\downarrow\) indicate higher/lower is better. Coverage is not ranked.}
\label{tab:powered_zero}
\CSPfitTable{\textwidth}{%
\begin{tabular}{lccccc}
\toprule
Cell & Top accuracy \(\uparrow\) & Top coverage & \GradRef{}/SqTL/Top size \(\downarrow\) & Relative size [95\%] \(\downarrow\) & Joint \(p\)\\
\midrule
R1W50 & \({0.9515\pm0.0010}\) & \({0.8969\pm0.0015}\) & \({2.060/1.920/\mathbf{1.684}}\) & \({-0.121\;[-0.130,-0.113]}\) & \({1.18\times10^{-22}}\)\\
R1S50 & \({0.9538\pm0.0010}\) & \({0.8988\pm0.0016}\) & \({1.969/1.836/\mathbf{1.591}}\) & \({-0.132\;[-0.140,-0.124]}\) & \({6.71\times10^{-21}}\)\\
R2W50 & \({0.9060\pm0.0010}\) & \({0.9027\pm0.0014}\) & \({3.968/3.745/\mathbf{3.446}}\) & \({-0.078\;[-0.087,-0.070]}\) & \({3.21\times10^{-21}}\)\\
R2S50 & \({0.9155\pm0.0010}\) & \({0.9002\pm0.0014}\) & \({3.574/3.343/\mathbf{3.081}}\) & \({-0.077\;[-0.086,-0.068]}\) & \({8.65\times10^{-23}}\)\\
R3W50 & \({0.9303\pm0.0010}\) & \({0.9036\pm0.0016}\) & \({3.298/3.100/\mathbf{2.817}}\) & \({-0.090\;[-0.100,-0.080]}\) & \({1.66\times10^{-23}}\)\\
R3S50 & \({0.9385\pm0.0009}\) & \({0.8985\pm0.0016}\) & \({2.924/2.762/\mathbf{2.448}}\) & \({-0.112\;[-0.121,-0.104]}\) & \({1.37\times10^{-19}}\)\\
\bottomrule
\end{tabular}}
\end{table*}
\begin{table*}[!tbp]
\centering\scriptsize
\setlength{\tabcolsep}{3pt}
\caption{Encoder results across recovery budgets at 50\% sparsity, with 200 paired seeds. Set-size triples list \GradRef{}/SqTL/Top means: the absolute-gradient reference, squared true-label control, and Top-2. Relative size compares Top with SqTL, with a 95\% Student-\(t\) confidence interval. Meaningful benefit requires a reduction greater than 2\%; practical equivalence requires a difference within \(\pm2\%\). Both require accuracy equivalence within \(\pm0.005\). Joint \(p\)-values use Holm correction over all 270 powered claims. Bold indicates the best value per row; \(\downarrow\) indicates lower is better.}
\label{tab:powered_recovery}
\CSPfitTable{\textwidth}{%
\begin{tabular}{lrccccc}
\toprule
Cell & Updates & \GradRef{}/SqTL/Top size \(\downarrow\) & Relative size [95\%] \(\downarrow\) & Benefit \(p\) & Equiv. \(p\) & Decision\\
\midrule
R1W50 & 100 & \({1.683/1.656/\mathbf{1.597}}\) & \({-0.035\;[-0.040,-0.030]}\) & \({2.17\times10^{-6}}\) & \({1.0000}\) & \({\mathrm{benefit}}\)\\
R1W50 & 250 & \({1.629/1.610/\mathbf{1.569}}\) & \({-0.025\;[-0.031,-0.020]}\) & \({1.0000}\) & \({1.0000}\) & \({\mathrm{unresolved}}\)\\
R1W50 & 500 & \({1.597/1.586/\mathbf{1.559}}\) & \({-0.017\;[-0.022,-0.012]}\) & \({1.0000}\) & \({1.0000}\) & \({\mathrm{unresolved}}\)\\
R1W50 & 1000 & \({1.566/1.570/\mathbf{1.559}}\) & \({-0.007\;[-0.012,-0.001]}\) & \({1.0000}\) & \({8.10\times10^{-4}}\) & \({\mathrm{equivalent}}\)\\
R1S50 & 100 & \({1.605/1.584/\mathbf{1.532}}\) & \({-0.032\;[-0.037,-0.028]}\) & \({3.93\times10^{-5}}\) & \({1.0000}\) & \({\mathrm{benefit}}\)\\
R1S50 & 250 & \({1.563/1.545/\mathbf{1.506}}\) & \({-0.024\;[-0.030,-0.019]}\) & \({1.0000}\) & \({1.0000}\) & \({\mathrm{unresolved}}\)\\
R1S50 & 500 & \({1.534/1.525/\mathbf{1.494}}\) & \({-0.020\;[-0.025,-0.015]}\) & \({1.0000}\) & \({1.0000}\) & \({\mathrm{unresolved}}\)\\
R1S50 & 1000 & \({1.511/1.505/\mathbf{1.492}}\) & \({-0.008\;[-0.013,-0.003]}\) & \({1.0000}\) & \({4.48\times10^{-4}}\) & \({\mathrm{equivalent}}\)\\
R2W50 & 100 & \({3.282/3.247/\mathbf{3.195}}\) & \({-0.016\;[-0.021,-0.011]}\) & \({1.0000}\) & \({1.0000}\) & \({\mathrm{unresolved}}\)\\
R2W50 & 250 & \({3.198/3.170/\mathbf{3.113}}\) & \({-0.017\;[-0.022,-0.013]}\) & \({1.0000}\) & \({1.0000}\) & \({\mathrm{unresolved}}\)\\
R2W50 & 500 & \({3.143/3.125/\mathbf{3.080}}\) & \({-0.014\;[-0.019,-0.009]}\) & \({1.0000}\) & \({1.0000}\) & \({\mathrm{unresolved}}\)\\
R2W50 & 1000 & \({3.098/3.091/\mathbf{3.072}}\) & \({-0.006\;[-0.011,-0.001]}\) & \({1.0000}\) & \({3.73\times10^{-5}}\) & \({\mathrm{equivalent}}\)\\
R2S50 & 100 & \({2.949/2.922/\mathbf{2.866}}\) & \({-0.019\;[-0.023,-0.014]}\) & \({1.0000}\) & \({1.0000}\) & \({\mathrm{unresolved}}\)\\
R2S50 & 250 & \({2.878/2.851/\mathbf{2.796}}\) & \({-0.019\;[-0.024,-0.014]}\) & \({1.0000}\) & \({1.0000}\) & \({\mathrm{unresolved}}\)\\
R2S50 & 500 & \({2.827/2.815/\mathbf{2.771}}\) & \({-0.015\;[-0.020,-0.010]}\) & \({1.0000}\) & \({1.0000}\) & \({\mathrm{unresolved}}\)\\
R2S50 & 1000 & \({2.787/2.777/\mathbf{2.769}}\) & \({-0.003\;[-0.007,0.002]}\) & \({1.0000}\) & \({3.09\times10^{-9}}\) & \({\mathrm{equivalent}}\)\\
R3W50 & 100 & \({2.790/2.747/\mathbf{2.682}}\) & \({-0.023\;[-0.029,-0.018]}\) & \({1.0000}\) & \({1.0000}\) & \({\mathrm{unresolved}}\)\\
R3W50 & 250 & \({2.720/2.691/\mathbf{2.637}}\) & \({-0.020\;[-0.025,-0.014]}\) & \({1.0000}\) & \({1.0000}\) & \({\mathrm{unresolved}}\)\\
R3W50 & 500 & \({2.676/2.668/\mathbf{2.623}}\) & \({-0.016\;[-0.022,-0.011]}\) & \({1.0000}\) & \({1.0000}\) & \({\mathrm{unresolved}}\)\\
R3W50 & 1000 & \({2.651/2.639/\mathbf{2.613}}\) & \({-0.009\;[-0.015,-0.004]}\) & \({1.0000}\) & \({0.0105}\) & \({\mathrm{equivalent}}\)\\
R3S50 & 100 & \({2.465/2.441/\mathbf{2.355}}\) & \({-0.035\;[-0.040,-0.030]}\) & \({5.76\times10^{-6}}\) & \({1.0000}\) & \({\mathrm{benefit}}\)\\
R3S50 & 250 & \({2.404/2.383/\mathbf{2.327}}\) & \({-0.023\;[-0.028,-0.018]}\) & \({1.0000}\) & \({1.0000}\) & \({\mathrm{unresolved}}\)\\
R3S50 & 500 & \({2.370/2.350/\mathbf{2.317}}\) & \({-0.014\;[-0.018,-0.009]}\) & \({1.0000}\) & \({0.8244}\) & \({\mathrm{unresolved}}\)\\
R3S50 & 1000 & \({2.331/2.323/\mathbf{2.313}}\) & \({-0.004\;[-0.009,0.001]}\) & \({1.0000}\) & \({4.08\times10^{-7}}\) & \({\mathrm{equivalent}}\)\\
\bottomrule
\end{tabular}}
\end{table*}
\begin{table*}[!tbp]
\centering\scriptsize
\setlength{\tabcolsep}{3pt}
\caption{Encoder results at equal total time, with 200 paired seeds. Triples list \GradRef{}/SqTL/Top: the absolute-gradient reference, squared true-label control, and Top-2. Cap is the deadline in minutes; updates are mean completed counts. Relative size compares Top with SqTL. Joint decisions require accuracy equivalence, the stated size endpoint, and elapsed-time equivalence within 1\% of the cap, with Holm correction over 270 claims. Every run finishes within its cap. Bold indicates the best value per row; \(\downarrow\) indicates lower is better.}
\label{tab:equal_time_recovery}
\CSPfitTable{\textwidth}{%
\begin{tabular}{lrcccccc}
\toprule
Cell & Cap & Completed updates & Set size \(\downarrow\) & Relative size [95\%] \(\downarrow\) & Benefit \(p\) & Equiv. \(p\) & Decision\\
\midrule
R1W50 & 40 & \({683/305/271}\) & \({1.586/1.608/\mathbf{1.569}}\) & \({-0.024\;[-0.029,-0.018]}\) & \({1.0000}\) & \({1.0000}\) & \({\mathrm{unresolved}}\)\\
R1W50 & 55 & \({1167/789/755}\) & \({1.565/1.573/\mathbf{1.557}}\) & \({-0.010\;[-0.016,-0.004]}\) & \({1.0000}\) & \({0.0418}\) & \({\mathrm{equivalent}}\)\\
R1W50 & 75 & \({1812/1434/1401}\) & \({1.563/1.562/\mathbf{1.557}}\) & \({-0.002\;[-0.008,0.004]}\) & \({1.0000}\) & \({1.59\times10^{-6}}\) & \({\mathrm{equivalent}}\)\\
R1S50 & 58 & \({876/281/247}\) & \({1.516/1.545/\mathbf{1.510}}\) & \({-0.022\;[-0.027,-0.017]}\) & \({1.0000}\) & \({1.0000}\) & \({\mathrm{unresolved}}\)\\
R1S50 & 75 & \({1425/830/796}\) & \({1.503/1.514/\mathbf{1.493}}\) & \({-0.013\;[-0.018,-0.008]}\) & \({1.0000}\) & \({0.3706}\) & \({\mathrm{unresolved}}\)\\
R1S50 & 100 & \({2233/1637/1604}\) & \({1.495/1.500/\mathbf{1.491}}\) & \({-0.006\;[-0.011,-0.001]}\) & \({1.0000}\) & \({2.12\times10^{-5}}\) & \({\mathrm{equivalent}}\)\\
R2W50 & 64 & \({793/340/313}\) & \({\mathbf{3.112}/3.144/3.113}\) & \({-0.009\;[-0.014,-0.004]}\) & \({1.0000}\) & \({3.02\times10^{-4}}\) & \({\mathrm{equivalent}}\)\\
R2W50 & 85 & \({1283/831/804}\) & \({3.090/3.105/\mathbf{3.067}}\) & \({-0.011\;[-0.016,-0.006]}\) & \({1.0000}\) & \({0.0722}\) & \({\mathrm{unresolved}}\)\\
R2W50 & 115 & \({1984/1533/1504}\) & \({3.074/3.074/\mathbf{3.064}}\) & \({-0.003\;[-0.008,0.002]}\) & \({1.0000}\) & \({3.63\times10^{-9}}\) & \({\mathrm{equivalent}}\)\\
R2S50 & 90 & \({1026/366/339}\) & \({\mathbf{2.788}/2.831/\mathbf{2.788}}\) & \({-0.015\;[-0.020,-0.010]}\) & \({1.0000}\) & \({1.0000}\) & \({\mathrm{unresolved}}\)\\
R2S50 & 120 & \({1724/1063/1035}\) & \({2.787/2.775/\mathbf{2.768}}\) & \({-0.002\;[-0.007,0.003]}\) & \({1.0000}\) & \({1.51\times10^{-10}}\) & \({\mathrm{equivalent}}\)\\
R2S50 & 160 & \({2652/1991/1966}\) & \({2.772/\mathbf{2.768}/\mathbf{2.768}}\) & \({0.000\;[-0.005,0.005]}\) & \({1.0000}\) & \({6.12\times10^{-11}}\) & \({\mathrm{equivalent}}\)\\
R3W50 & 84 & \({861/360/335}\) & \({2.654/2.673/\mathbf{2.632}}\) & \({-0.015\;[-0.020,-0.010]}\) & \({1.0000}\) & \({1.0000}\) & \({\mathrm{unresolved}}\)\\
R3W50 & 112 & \({1398/896/872}\) & \({2.635/2.640/\mathbf{2.615}}\) & \({-0.009\;[-0.014,-0.003]}\) & \({1.0000}\) & \({0.0094}\) & \({\mathrm{equivalent}}\)\\
R3W50 & 150 & \({2126/1624/1600}\) & \({2.626/2.632/\mathbf{2.607}}\) & \({-0.009\;[-0.015,-0.004]}\) & \({1.0000}\) & \({0.0139}\) & \({\mathrm{equivalent}}\)\\
R3S50 & 125 & \({1208/442/418}\) & \({2.325/2.363/\mathbf{2.313}}\) & \({-0.020\;[-0.026,-0.015]}\) & \({1.0000}\) & \({1.0000}\) & \({\mathrm{unresolved}}\)\\
R3S50 & 165 & \({1974/1208/1184}\) & \({2.316/2.319/\mathbf{2.310}}\) & \({-0.003\;[-0.009,0.002]}\) & \({1.0000}\) & \({4.97\times10^{-7}}\) & \({\mathrm{equivalent}}\)\\
R3S50 & 220 & \({3027/2261/2237}\) & \({2.308/2.313/\mathbf{2.307}}\) & \({-0.002\;[-0.007,0.004]}\) & \({1.0000}\) & \({7.26\times10^{-8}}\) & \({\mathrm{equivalent}}\)\\
\bottomrule
\end{tabular}}
\end{table*}
\begin{table*}[!tbp]
\centering\scriptsize
\setlength{\tabcolsep}{3pt}
\caption{Accuracy, coverage, set size, and total time for recovery restricted to the pruning pool, with 200 seeds. Entries are means with 95\% Student-\(t\) confidence-interval half-widths. Top-0 reuses the powered study's Top-2 records without recovery; it adds no repetitions. Total time includes all 1,000 scheduled updates and every checkpoint evaluation. Bold indicates the better arm per cell and column; \(\uparrow\)/\(\downarrow\) indicate higher/lower is better. Coverage is not ranked.}
\label{tab:restricted_metrics}
\CSPfitTable{\textwidth}{%
\begin{tabular}{llcccc}
\toprule
Cell & Arm & Accuracy \(\uparrow\) & Coverage & Size \(\downarrow\) & Total min \(\downarrow\)\\
\midrule
R1W50 & \GradRef{} & \({0.9509\pm0.0008}\) & \({0.8971\pm0.0015}\) & \({2.032\pm0.019}\) & \({\mathbf{49.84}\pm0.16}\)\\
R1W50 & SqTL & \({\mathbf{0.9516}\pm0.0009}\) & \({0.8970\pm0.0015}\) & \({\mathbf{1.896}\pm0.022}\) & \({61.53\pm0.19}\)\\
R1S50 & \GradRef{} & \({\mathbf{0.9540}\pm0.0008}\) & \({0.8990\pm0.0015}\) & \({1.932\pm0.020}\) & \({\mathbf{61.82}\pm0.19}\)\\
R1S50 & SqTL & \({0.9539\pm0.0010}\) & \({0.8984\pm0.0016}\) & \({\mathbf{1.805}\pm0.021}\) & \({80.22\pm0.22}\)\\
R2W50 & \GradRef{} & \({0.9062\pm0.0008}\) & \({0.9026\pm0.0015}\) & \({3.911\pm0.038}\) & \({\mathbf{72.99}\pm0.23}\)\\
R2W50 & SqTL & \({\mathbf{0.9063}\pm0.0010}\) & \({0.9027\pm0.0014}\) & \({\mathbf{3.694}\pm0.044}\) & \({92.31\pm0.27}\)\\
R2S50 & \GradRef{} & \({0.9157\pm0.0008}\) & \({0.9004\pm0.0015}\) & \({3.527\pm0.033}\) & \({\mathbf{88.94}\pm0.28}\)\\
R2S50 & SqTL & \({\mathbf{0.9160}\pm0.0010}\) & \({0.9003\pm0.0016}\) & \({\mathbf{3.295}\pm0.037}\) & \({117.31\pm0.34}\)\\
R3W50 & \GradRef{} & \({\mathbf{0.9308}\pm0.0008}\) & \({0.9037\pm0.0015}\) & \({3.248\pm0.033}\) & \({\mathbf{91.41}\pm0.30}\)\\
R3W50 & SqTL & \({\mathbf{0.9308}\pm0.0010}\) & \({0.9042\pm0.0015}\) & \({\mathbf{3.043}\pm0.035}\) & \({117.52\pm0.33}\)\\
R3S50 & \GradRef{} & \({0.9381\pm0.0008}\) & \({0.8983\pm0.0015}\) & \({2.879\pm0.028}\) & \({\mathbf{114.25}\pm0.36}\)\\
R3S50 & SqTL & \({\mathbf{0.9382}\pm0.0010}\) & \({0.8981\pm0.0016}\) & \({\mathbf{2.720}\pm0.031}\) & \({154.18\pm0.48}\)\\
\bottomrule
\end{tabular}}
\end{table*}
\begin{table*}[!tbp]
\centering\scriptsize
\setlength{\tabcolsep}{3pt}
\caption{Mean set size after 100 recovery updates over nested encoder seed sets. Each entry lists the mean over the first 16 records followed by the mean over all 200. The 16-record prefix is descriptive and adds no independent replication; powered inference uses all 200 records. \GradRef{}/SqTL/Top/Thr denote the absolute-gradient reference, squared true-label control, Top-2, and Threshold-2. Bold indicates the smallest value per row and seed set; lower is better.}
\label{tab:recovery_headline}
\CSPfitTable{\textwidth}{%
\begin{tabular}{lcccc}
\toprule
Cell & \GradRef{}: 16/200 & SqTL: 16/200 & Top: 16/200 & Thr: 16/200\\
\midrule
R1W50 & \({1.676/1.683}\) & \({1.647/1.656}\) & \({1.594/1.597}\) & \({\mathbf{1.578}/\mathbf{1.575}}\)\\
R1S50 & \({1.610/1.605}\) & \({1.586/1.584}\) & \({1.529/1.532}\) & \({\mathbf{1.491}/\mathbf{1.497}}\)\\
R2W50 & \({3.257/3.282}\) & \({3.237/3.247}\) & \({3.183/3.195}\) & \({\mathbf{3.123}/\mathbf{3.154}}\)\\
R2S50 & \({2.941/2.949}\) & \({2.912/2.922}\) & \({2.852/2.866}\) & \({\mathbf{2.795}/\mathbf{2.827}}\)\\
R3W50 & \({2.787/2.790}\) & \({2.767/2.747}\) & \({2.661/2.682}\) & \({\mathbf{2.611}/\mathbf{2.633}}\)\\
R3S50 & \({2.472/2.465}\) & \({2.429/2.441}\) & \({2.334/2.355}\) & \({\mathbf{2.273}/\mathbf{2.286}}\)\\
\bottomrule
\end{tabular}}
\end{table*}
\begin{table*}[!tbp]
\centering\scriptsize
\setlength{\tabcolsep}{3pt}
\caption{Top-2 after 100 updates versus the absolute-gradient reference after 1,000 updates, with 200 paired seeds. Entries list Top-100/\GradRef{}-1000 values. This comparison supplements the equal-time study. Joint success requires accuracy equivalence within \(\pm0.005\), relative set-size noninferiority within 5\%, and lower complete time, using Holm correction over the same 270 claims. Bold indicates the best value per row and metric; \(\uparrow\)/\(\downarrow\) indicate higher/lower is better.}
\label{tab:cross_budget_recovery}
\CSPfitTable{\textwidth}{%
\begin{tabular}{lccccc}
\toprule
Cell & Accuracy \(\uparrow\) & Size \(\downarrow\) & Total min \(\downarrow\) & Joint \(p\) & Success\\
\midrule
R1W50 & \({0.9531/\mathbf{0.9576}}\) & \({1.597/\mathbf{1.566}}\) & \({\mathbf{34.71}/49.84}\) & \({1.0000}\) & \({\mathrm{no}}\)\\
R1S50 & \({0.9553/\mathbf{0.9600}}\) & \({1.532/\mathbf{1.511}}\) & \({\mathbf{53.41}/61.82}\) & \({1.0000}\) & \({\mathrm{no}}\)\\
R2W50 & \({0.9079/\mathbf{0.9122}}\) & \({3.195/\mathbf{3.098}}\) & \({\mathbf{54.96}/72.99}\) & \({1.0000}\) & \({\mathrm{no}}\)\\
R2S50 & \({0.9174/\mathbf{0.9216}}\) & \({2.866/\mathbf{2.787}}\) & \({\mathbf{79.76}/88.94}\) & \({1.0000}\) & \({\mathrm{no}}\)\\
R3W50 & \({0.9320/\mathbf{0.9363}}\) & \({2.682/\mathbf{2.651}}\) & \({\mathbf{71.83}/91.41}\) & \({1.0000}\) & \({\mathrm{no}}\)\\
R3S50 & \({0.9404/\mathbf{0.9445}}\) & \({2.355/\mathbf{2.331}}\) & \({\mathbf{108.48}/114.25}\) & \({1.0000}\) & \({\mathrm{no}}\)\\
\bottomrule
\end{tabular}}
\end{table*}
\begin{table*}[!tbp]\centering\begin{minipage}{0.48\textwidth}
\centering\scriptsize
\setlength{\tabcolsep}{3pt}
\caption{Analytical planning under independent normal paired differences, not model measurements. Conservative one-sided level \(0.05/270\). Benefit power assumes a true 4\% reduction against the 2\% minimum; equivalence power assumes a true 0.4\% difference and a \(\pm2\%\) margin. Powers use noncentral-\(t\) probabilities and chi-square integration over the estimated standard deviation (SD).}
\label{tab:recovery_power}
\CSPfitTable{\linewidth}{%
\begin{tabular}{rrrr}
\toprule
Seeds & Paired SD & Benefit power & Equiv. power\\
\midrule
128 & 0.04 & 0.975 & 0.802\\
128 & 0.05 & 0.803 & 0.447\\
128 & 0.07 & 0.342 & 0.002\\
200 & 0.04 & 1.000 & 0.978\\
200 & 0.05 & 0.978 & 0.813\\
200 & 0.07 & 0.661 & 0.240\\
320 & 0.04 & 1.000 & 1.000\\
320 & 0.05 & 1.000 & 0.982\\
320 & 0.07 & 0.933 & 0.681\\
\bottomrule
\end{tabular}}
\end{minipage}\end{table*}

\subsection{Candidate Rules, Transfer, and Direct Search}

Candidate replacements hold both example count and gradient count fixed. Dense references establish the accuracy and set-size scale. The full Threshold-2 grid includes unsuccessful sparsity and zero-shot cells, while paired tests compare Threshold-2 directly with Top-2 and Random-2. Direct search uses the common block support defined in Appendix~\ref{app:additional_controls}.

\begin{table*}[!tbp]
\centering
\scriptsize
\setlength{\tabcolsep}{3pt}
\caption{Dense-model reference results over 16 seeds. R/D/Q denote RoBERTa, the fine-tuned Qwen classifier, and the fixed zero-shot Qwen scorer. Entries are means with 95\% Student-\(t\) confidence-interval half-widths. Q accuracy is deterministic on the fixed test set and has no seed interval. Singleton rate is the fraction of all test examples with a one-label prediction set.}
\label{tab:tcep_dense_summary}
\CSPfitTable{\textwidth}{%
\begin{tabular}{lcccc}
\toprule
Regime & Accuracy & Coverage & Set size & Singleton rate\\
\midrule
R1 & \({0.9640\pm0.0021}\) & \({0.9005\pm0.0058}\) & \({1.430\pm0.044}\) & \({0.770\pm0.011}\)\\
R2 & \({0.9320\pm0.0021}\) & \({0.9005\pm0.0057}\) & \({2.750\pm0.093}\) & \({0.590\pm0.011}\)\\
R3 & \({0.9520\pm0.0021}\) & \({0.9005\pm0.0059}\) & \({2.180\pm0.061}\) & \({0.670\pm0.011}\)\\
D1 & \({0.9750\pm0.0019}\) & \({0.9001\pm0.0057}\) & \({1.300\pm0.056}\) & \({0.810\pm0.012}\)\\
D2 & \({0.9490\pm0.0019}\) & \({0.9016\pm0.0057}\) & \({2.373\pm0.090}\) & \({0.640\pm0.012}\)\\
D3 & \({0.9630\pm0.0019}\) & \({0.9026\pm0.0058}\) & \({1.955\pm0.116}\) & \({0.720\pm0.012}\)\\
Q1 & \({0.6830}\) & \({0.9005\pm0.0059}\) & \({6.330\pm0.244}\) & \({0.220\pm0.004}\)\\
Q2 & \({0.6081}\) & \({0.9005\pm0.0057}\) & \({20.299\pm0.661}\) & \({0.070\pm0.003}\)\\
Q3 & \({0.6540}\) & \({0.9005\pm0.0065}\) & \({35.699\pm0.991}\) & \({0.040\pm0.003}\)\\
\bottomrule
\end{tabular}}
\end{table*}
\begin{table*}[!tbp]
\centering
\scriptsize
\setlength{\tabcolsep}{3pt}
\caption{Candidate-label replacement on RoBERTa at 50\% sparsity. All rules use 256 pruning examples, two label gradients per example, and the same validation selector. Entries report mean set size and joint \(p\)-values for Threshold-2 versus Random-2 or Top-2. Joint tests require accuracy equivalence and use Holm correction over 96 comparisons. Table~\ref{tab:candidate_paired_tests} reports paired confidence intervals and decoder controls. Bold indicates the smallest set size per row; lower is better.}
\label{tab:candidate_paired_headline}
\CSPfitTable{\textwidth}{%
\begin{tabular}{lccccc}
\toprule
Cell & Random-2 & Top-2 & Threshold-2 & vs. Random \(p_J\) & vs. Top \(p_J\)\\
\midrule
R2S50 & \({3.326}\) & \({3.086}\) & \({\mathbf{2.946}}\) & \({0.0136}\) & \({0.0445}\)\\
R2W50 & \({3.681}\) & \({3.441}\) & \({\mathbf{3.301}}\) & \({0.0135}\) & \({0.0494}\)\\
R3S50 & \({2.666}\) & \({2.446}\) & \({\mathbf{2.326}}\) & \({0.0214}\) & \({0.4991}\)\\
R3W50 & \({3.041}\) & \({2.821}\) & \({\mathbf{2.701}}\) & \({0.0216}\) & \({0.4929}\)\\
R1S50 & \({1.753}\) & \({1.598}\) & \({\mathbf{1.543}}\) & \({0.0077}\) & \({1.0000}\)\\
R1W50 & \({1.849}\) & \({1.694}\) & \({\mathbf{1.639}}\) & \({0.0072}\) & \({1.0000}\)\\
\bottomrule
\end{tabular}}
\end{table*}
\begin{table*}[!tbp]
\centering
\scriptsize
\setlength{\tabcolsep}{3pt}
\caption{Paired differences for Threshold-2 minus Top-2 or Random-2. Each row uses the same 16 checkpoints and split seeds. Accuracy and set-size differences have 90\% and 95\% confidence intervals, respectively. Two one-sided tests (TOST) assess accuracy equivalence; Exact is the sign-flip sensitivity \(p\)-value; Joint Holm is the adjusted joint \(p\)-value; J indicates rejection of the joint null. These 24 comparisons belong to the 96-comparison Holm family. \(\uparrow\)/\(\downarrow\) indicate higher/lower is better.}
\label{tab:candidate_paired_tests}
\CSPfitTable{\textwidth}{%
\begin{tabular}{lcccccc}
\toprule
Comparison & \(\Delta a\) [90\%] \(\uparrow\) & \(\Delta|C|\) [95\%] \(\downarrow\) & TOST \(p\) & Exact \(p\) & Joint Holm & J\\
\midrule
R2S50/Top-2 & \({-0.0007\;[-0.0025,0.0011]}\) & \({-0.140\;[-0.212,-0.068]}\) & \({3.83\times10^{-4}}\) & \({9.46\times10^{-4}}\) & \({0.0445}\) & \({yes}\)\\
R2S50/Random-2 & \({0.0003\;[-0.0015,0.0021]}\) & \({-0.380\;[-0.452,-0.308]}\) & \({1.87\times10^{-4}}\) & \({3.05\times10^{-5}}\) & \({0.0136}\) & \({yes}\)\\
R2W50/Top-2 & \({-0.0007\;[-0.0025,0.0011]}\) & \({-0.140\;[-0.212,-0.068]}\) & \({3.92\times10^{-4}}\) & \({0.0011}\) & \({0.0494}\) & \({yes}\)\\
R2W50/Random-2 & \({0.0003\;[-0.0015,0.0021]}\) & \({-0.380\;[-0.452,-0.308]}\) & \({1.83\times10^{-4}}\) & \({3.05\times10^{-5}}\) & \({0.0135}\) & \({yes}\)\\
R3S50/Top-2 & \({-0.0007\;[-0.0026,0.0012]}\) & \({-0.120\;[-0.213,-0.027]}\) & \({6.75\times10^{-4}}\) & \({0.0172}\) & \({0.4991}\) & \({no}\)\\
R3S50/Random-2 & \({0.0003\;[-0.0016,0.0022]}\) & \({-0.340\;[-0.433,-0.247]}\) & \({3.15\times10^{-4}}\) & \({3.05\times10^{-5}}\) & \({0.0214}\) & \({yes}\)\\
R3W50/Top-2 & \({-0.0007\;[-0.0026,0.0012]}\) & \({-0.120\;[-0.213,-0.027]}\) & \({7.40\times10^{-4}}\) & \({0.0159}\) & \({0.4929}\) & \({no}\)\\
R3W50/Random-2 & \({0.0003\;[-0.0017,0.0022]}\) & \({-0.340\;[-0.433,-0.247]}\) & \({3.22\times10^{-4}}\) & \({3.05\times10^{-5}}\) & \({0.0216}\) & \({yes}\)\\
R1S50/Top-2 & \({-0.0007\;[-0.0024,0.0010]}\) & \({-0.055\;[-0.106,-0.004]}\) & \({2.06\times10^{-4}}\) & \({0.0376}\) & \({1.0000}\) & \({no}\)\\
R1S50/Random-2 & \({0.0003\;[-0.0014,0.0020]}\) & \({-0.210\;[-0.261,-0.159]}\) & \({8.83\times10^{-5}}\) & \({3.05\times10^{-5}}\) & \({0.0077}\) & \({yes}\)\\
R1W50/Top-2 & \({-0.0007\;[-0.0024,0.0010]}\) & \({-0.055\;[-0.106,-0.004]}\) & \({2.04\times10^{-4}}\) & \({0.0370}\) & \({1.0000}\) & \({no}\)\\
R1W50/Random-2 & \({0.0003\;[-0.0014,0.0020]}\) & \({-0.210\;[-0.261,-0.159]}\) & \({8.14\times10^{-5}}\) & \({3.05\times10^{-5}}\) & \({0.0072}\) & \({yes}\)\\
D2S50/Top-2 & \({-0.0007\;[-0.0025,0.0011]}\) & \({-0.182\;[-0.254,-0.110]}\) & \({4.06\times10^{-4}}\) & \({1.53\times10^{-4}}\) & \({0.0247}\) & \({yes}\)\\
D2S50/Random-2 & \({0.0003\;[-0.0015,0.0021]}\) & \({-0.494\;[-0.566,-0.422]}\) & \({1.68\times10^{-4}}\) & \({3.05\times10^{-5}}\) & \({0.0129}\) & \({yes}\)\\
D3S50/Top-2 & \({-0.0007\;[-0.0026,0.0012]}\) & \({-0.156\;[-0.249,-0.063]}\) & \({6.73\times10^{-4}}\) & \({0.0028}\) & \({0.1151}\) & \({no}\)\\
D3S50/Random-2 & \({0.0003\;[-0.0017,0.0022]}\) & \({-0.442\;[-0.535,-0.349]}\) & \({3.27\times10^{-4}}\) & \({3.05\times10^{-5}}\) & \({0.0216}\) & \({yes}\)\\
D1S50/Top-2 & \({-0.0007\;[-0.0024,0.0010]}\) & \({-0.071\;[-0.122,-0.021]}\) & \({2.04\times10^{-4}}\) & \({0.0096}\) & \({0.3279}\) & \({no}\)\\
D1S50/Random-2 & \({0.0003\;[-0.0014,0.0020]}\) & \({-0.273\;[-0.324,-0.222]}\) & \({8.05\times10^{-5}}\) & \({3.05\times10^{-5}}\) & \({0.0072}\) & \({yes}\)\\
Q2S50/Top-2 & \({-0.0001\;[-0.0016,0.0015]}\) & \({-1.098\;[-1.488,-0.708]}\) & \({2.28\times10^{-5}}\) & \({6.10\times10^{-5}}\) & \({0.0057}\) & \({yes}\)\\
Q2S50/Random-2 & \({-0.0000\;[-0.0016,0.0015]}\) & \({-2.163\;[-2.553,-1.773]}\) & \({2.20\times10^{-5}}\) & \({3.05\times10^{-5}}\) & \({0.0029}\) & \({yes}\)\\
Q3S50/Top-2 & \({0.0000\;[-0.0015,0.0016]}\) & \({-2.039\;[-2.763,-1.314]}\) & \({2.10\times10^{-5}}\) & \({1.53\times10^{-4}}\) & \({0.0119}\) & \({yes}\)\\
Q3S50/Random-2 & \({-0.0000\;[-0.0016,0.0015]}\) & \({-4.015\;[-4.740,-3.291]}\) & \({2.03\times10^{-5}}\) & \({3.05\times10^{-5}}\) & \({0.0029}\) & \({yes}\)\\
Q1S50/Top-2 & \({0.0000\;[-0.0015,0.0015]}\) & \({-0.288\;[-0.390,-0.185]}\) & \({2.16\times10^{-5}}\) & \({6.10\times10^{-5}}\) & \({0.0057}\) & \({yes}\)\\
Q1S50/Random-2 & \({0.0000\;[-0.0015,0.0015]}\) & \({-0.566\;[-0.669,-0.464]}\) & \({1.94\times10^{-5}}\) & \({3.05\times10^{-5}}\) & \({0.0029}\) & \({yes}\)\\
\bottomrule
\end{tabular}}
\end{table*}
\begin{table*}[!tbp]
\centering
\scriptsize
\setlength{\tabcolsep}{3pt}
\caption{Results across all 42 main cells: 18 encoder, six fine-tuned decoder, and 18 zero-shot decoder stress cells. R/D/Q denote RoBERTa, fine-tuned Qwen, and zero-shot Qwen. Q uses a fixed sequence scorer; R and D vary task fine-tuning by seed. The Q rows and R rows at 50\% sparsity identify the fixed 16-seed scorer comparisons. Joint decisions use the 96-comparison family. Bold indicates the best value per row and metric; \(\uparrow\)/\(\downarrow\) indicate higher/lower is better. Coverage is not ranked.}
\label{tab:tcep_full_grid}
\CSPfitTable{\textwidth}{%
\begin{tabular}{lcccccc}
\toprule
Cell & Accuracy \(\uparrow\) & Coverage & Set size \(\downarrow\) & Gain (\%) \(\uparrow\) & F & J\\
\midrule
R2S30 & \({\mathbf{0.9290}/0.9288}\) & \({0.8985/0.8984}\) & \({3.02/\mathbf{2.49}}\) & \({17.6}\) & \({0}\) & \({yes}\)\\
R2W30 & \({\mathbf{0.9230}/0.9226}\) & \({0.8996/0.8998}\) & \({3.25/\mathbf{2.73}}\) & \({16.0}\) & \({0}\) & \({yes}\)\\
R2S50 & \({0.9160/\mathbf{0.9165}}\) & \({0.9012/0.9013}\) & \({3.57/\mathbf{2.95}}\) & \({17.6}\) & \({2}\) & \({yes}\)\\
R2W50 & \({0.9060/\mathbf{0.9065}}\) & \({0.9022/0.9025}\) & \({3.95/\mathbf{3.30}}\) & \({16.5}\) & \({2}\) & \({yes}\)\\
R2S70 & \({\mathbf{0.8580}/0.8539}\) & \({0.9003/0.9003}\) & \({6.34/\mathbf{5.38}}\) & \({15.2}\) & \({3}\) & \({no}\)\\
R2W70 & \({\mathbf{0.8340}/0.8300}\) & \({0.9049/0.9068}\) & \({7.19/\mathbf{6.32}}\) & \({12.1}\) & \({3}\) & \({no}\)\\
R3S30 & \({\mathbf{0.9480}/0.9478}\) & \({0.9023/0.9019}\) & \({2.46/\mathbf{2.03}}\) & \({17.4}\) & \({0}\) & \({yes}\)\\
R3W30 & \({\mathbf{0.9440}/0.9436}\) & \({0.9035/0.9039}\) & \({2.63/\mathbf{2.21}}\) & \({15.9}\) & \({0}\) & \({yes}\)\\
R3S50 & \({\mathbf{0.9380}/0.9371}\) & \({0.8984/0.8985}\) & \({2.93/\mathbf{2.33}}\) & \({20.6}\) & \({1}\) & \({yes}\)\\
R3W50 & \({\mathbf{0.9300}/0.9290}\) & \({0.9037/0.9038}\) & \({3.30/\mathbf{2.70}}\) & \({18.1}\) & \({1}\) & \({yes}\)\\
R3S70 & \({\mathbf{0.8860}/0.8848}\) & \({0.8972/0.8974}\) & \({5.21/\mathbf{4.43}}\) & \({14.9}\) & \({2}\) & \({no}\)\\
R3W70 & \({\mathbf{0.8660}/0.8649}\) & \({0.8964/0.8968}\) & \({6.01/\mathbf{5.12}}\) & \({14.7}\) & \({2}\) & \({no}\)\\
R1S30 & \({\mathbf{0.9620}/0.9618}\) & \({0.9028/0.9032}\) & \({1.63/\mathbf{1.35}}\) & \({17.4}\) & \({0}\) & \({yes}\)\\
R1W30 & \({\mathbf{0.9600}/0.9596}\) & \({0.8997/0.8989}\) & \({1.70/\mathbf{1.43}}\) & \({16.0}\) & \({0}\) & \({yes}\)\\
R1S50 & \({\mathbf{0.9540}/0.9537}\) & \({0.8990/0.8991}\) & \({1.97/\mathbf{1.54}}\) & \({21.9}\) & \({1}\) & \({yes}\)\\
R1W50 & \({\mathbf{0.9510}/0.9504}\) & \({0.8981/0.8979}\) & \({2.06/\mathbf{1.64}}\) & \({20.5}\) & \({1}\) & \({yes}\)\\
R1S70 & \({\mathbf{0.9250}/0.9247}\) & \({0.9022/0.9021}\) & \({2.93/\mathbf{2.48}}\) & \({15.5}\) & \({2}\) & \({yes}\)\\
R1W70 & \({\mathbf{0.9110}/0.9097}\) & \({0.9035/0.9040}\) & \({3.36/\mathbf{2.90}}\) & \({13.7}\) & \({2}\) & \({no}\)\\
D2S50 & \({0.9290/\mathbf{0.9293}}\) & \({0.9006/0.9011}\) & \({3.40/\mathbf{2.84}}\) & \({16.7}\) & \({2}\) & \({yes}\)\\
D2W50 & \({\mathbf{0.9170}/0.9166}\) & \({0.9013/0.9006}\) & \({3.90/\mathbf{3.31}}\) & \({15.1}\) & \({2}\) & \({yes}\)\\
D3S50 & \({0.9450/\mathbf{0.9453}}\) & \({0.8979/0.8982}\) & \({2.74/\mathbf{2.22}}\) & \({18.8}\) & \({1}\) & \({yes}\)\\
D3W50 & \({\mathbf{0.9380}/0.9377}\) & \({0.8996/0.9000}\) & \({3.12/\mathbf{2.59}}\) & \({16.9}\) & \({1}\) & \({yes}\)\\
D1S50 & \({\mathbf{0.9630}/0.9628}\) & \({0.9012/0.9021}\) & \({1.81/\mathbf{1.48}}\) & \({18.3}\) & \({1}\) & \({yes}\)\\
D1W50 & \({\mathbf{0.9560}/\mathbf{0.9560}}\) & \({0.8989/0.8990}\) & \({1.98/\mathbf{1.64}}\) & \({17.1}\) & \({1}\) & \({yes}\)\\
Q2S30 & \({\mathbf{0.4930}/0.4924}\) & \({0.9014/0.9014}\) & \({25.14/\mathbf{23.41}}\) & \({6.9}\) & \({0}\) & \({yes}\)\\
Q2W30 & \({\mathbf{0.4710}/0.4699}\) & \({0.9021/0.9020}\) & \({26.98/\mathbf{25.71}}\) & \({4.7}\) & \({1}\) & \({no}\)\\
Q2S50 & \({\mathbf{0.3590}/0.3586}\) & \({0.9000/0.8999}\) & \({33.27/\mathbf{30.24}}\) & \({9.1}\) & \({2}\) & \({yes}\)\\
Q2W50 & \({\mathbf{0.3380}/0.3369}\) & \({0.9023/0.9026}\) & \({35.31/\mathbf{33.31}}\) & \({5.7}\) & \({3}\) & \({yes}\)\\
Q2S70 & \({\mathbf{0.1980}/0.1934}\) & \({0.9003/0.9005}\) & \({48.58/\mathbf{45.56}}\) & \({6.2}\) & \({5}\) & \({no}\)\\
Q2W70 & \({\mathbf{0.1740}/0.1699}\) & \({0.8995/0.9000}\) & \({51.11/\mathbf{48.83}}\) & \({4.5}\) & \({6}\) & \({no}\)\\
Q3S30 & \({\mathbf{0.4710}/0.4702}\) & \({0.9041/0.9041}\) & \({45.09/\mathbf{41.60}}\) & \({7.7}\) & \({0}\) & \({no}\)\\
Q3W30 & \({\mathbf{0.4430}/0.4417}\) & \({0.9028/0.9028}\) & \({48.40/\mathbf{45.84}}\) & \({5.3}\) & \({1}\) & \({no}\)\\
Q3S50 & \({\mathbf{0.3440}/0.3434}\) & \({0.9003/0.9006}\) & \({61.78/\mathbf{56.43}}\) & \({8.7}\) & \({3}\) & \({yes}\)\\
Q3W50 & \({\mathbf{0.3000}/0.2992}\) & \({0.9007/0.9005}\) & \({67.13/\mathbf{63.32}}\) & \({5.7}\) & \({4}\) & \({yes}\)\\
Q3S70 & \({\mathbf{0.1800}/0.1754}\) & \({0.9003/0.9004}\) & \({94.66/\mathbf{88.96}}\) & \({6.0}\) & \({6}\) & \({no}\)\\
Q3W70 & \({\mathbf{0.1530}/0.1491}\) & \({0.9005/0.9011}\) & \({102.41/\mathbf{97.75}}\) & \({4.5}\) & \({7}\) & \({no}\)\\
Q1S30 & \({\mathbf{0.5610}/0.5606}\) & \({0.8986/0.8986}\) & \({6.83/\mathbf{6.42}}\) & \({6.0}\) & \({0}\) & \({yes}\)\\
Q1W30 & \({\mathbf{0.5480}/0.5476}\) & \({0.9020/0.9017}\) & \({7.27/\mathbf{7.01}}\) & \({3.6}\) & \({1}\) & \({yes}\)\\
Q1S50 & \({\mathbf{0.3730}/0.3724}\) & \({0.9046/0.9045}\) & \({8.71/\mathbf{7.98}}\) & \({8.4}\) & \({1}\) & \({yes}\)\\
Q1W50 & \({\mathbf{0.3520}/0.3513}\) & \({0.8979/0.8983}\) & \({9.20/\mathbf{8.78}}\) & \({4.6}\) & \({2}\) & \({yes}\)\\
Q1S70 & \({\mathbf{0.2010}/0.1976}\) & \({0.9033/0.9033}\) & \({10.95/\mathbf{10.33}}\) & \({5.7}\) & \({4}\) & \({no}\)\\
Q1W70 & \({\mathbf{0.1780}/0.1733}\) & \({0.9016/0.9012}\) & \({11.22/\mathbf{10.81}}\) & \({3.7}\) & \({5}\) & \({no}\)\\
\bottomrule
\end{tabular}}
\end{table*}
\begin{table*}[!tbp]
\centering\scriptsize
\setlength{\tabcolsep}{3pt}
\caption{Threshold-2 minus Top-2 with SparseGPT at 50\% sparsity, with 16 paired seeds. R/D/Q denote RoBERTa, fine-tuned Qwen, and zero-shot Qwen. Set-size and accuracy differences have paired 95\% and 90\% confidence intervals, respectively. Joint \(p\)-values use Holm correction over the separate 96 claims. Larger absolute savings in Q do not establish a causal interaction with set size. \(\uparrow\)/\(\downarrow\) indicate higher/lower is better.}
\label{tab:threshold_regime}
\CSPfitTable{\textwidth}{%
\begin{tabular}{lcccc}
\toprule
Cell & Top-2 size & \(\Delta|C|\) [95\%] \(\downarrow\) & \(\Delta a\) [90\%] \(\uparrow\) & Joint \(p\)\\
\midrule
R2S50 & \({3.086}\) & \({-0.140\;[-0.212,-0.068]}\) & \({-0.0007\;[-0.0025,0.0011]}\) & \({0.0445}\)\\
R3S50 & \({2.446}\) & \({-0.120\;[-0.213,-0.027]}\) & \({-0.0007\;[-0.0026,0.0012]}\) & \({0.4991}\)\\
R1S50 & \({1.598}\) & \({-0.055\;[-0.106,-0.004]}\) & \({-0.0007\;[-0.0024,0.0010]}\) & \({1.0000}\)\\
D2S50 & \({3.019}\) & \({-0.182\;[-0.254,-0.110]}\) & \({-0.0007\;[-0.0025,0.0011]}\) & \({0.0247}\)\\
D3S50 & \({2.380}\) & \({-0.156\;[-0.249,-0.063]}\) & \({-0.0007\;[-0.0026,0.0012]}\) & \({0.1151}\)\\
D1S50 & \({1.553}\) & \({-0.071\;[-0.122,-0.021]}\) & \({-0.0007\;[-0.0024,0.0010]}\) & \({0.3279}\)\\
Q2S50 & \({31.337}\) & \({-1.098\;[-1.488,-0.708]}\) & \({-0.0001\;[-0.0016,0.0015]}\) & \({0.0057}\)\\
Q3S50 & \({58.468}\) & \({-2.039\;[-2.763,-1.314]}\) & \({0.0000\;[-0.0015,0.0016]}\) & \({0.0119}\)\\
Q1S50 & \({8.268}\) & \({-0.288\;[-0.390,-0.185]}\) & \({0.0000\;[-0.0015,0.0015]}\) & \({0.0057}\)\\
\bottomrule
\end{tabular}}
\end{table*}
\begin{table*}[!tbp]
\centering
\scriptsize
\setlength{\tabcolsep}{3pt}
\caption{Finite-deletion criteria of Zhao and Bellotti versus Threshold-2 block saliency at 50\% sparsity. Entries are means with 95\% Student-\(t\) confidence-interval half-widths over 16 paired seeds. All methods use the same 32 fixed coordinate blocks per eligible matrix, remove 16 whole blocks, and apply neither reconstruction nor recovery. ZB-sign/ZB-abs adapt the signed/absolute efficiency-loss criteria to this setting; they do not reproduce the original hidden-unit experiments. Threshold-block sums Threshold-2 coordinate importance over the same blocks. Times include all deletion scoring and selection. Bold indicates the best method per cell and column; \(\uparrow\)/\(\downarrow\) indicate higher/lower is better. Coverage is not ranked.}
\label{tab:zb_block_comparison}
\CSPfitTable{\textwidth}{%
\begin{tabular}{llcccc}
\toprule
Cell & Method & Accuracy \(\uparrow\) & Coverage & Set size \(\downarrow\) & Total min \(\downarrow\)\\
\midrule
R1B50 & Mag-block & \({0.9420\pm0.0016}\) & \({0.8981\pm0.0047}\) & \({2.308\pm0.105}\) & \({\mathbf{9.0}\pm0.1}\)\\
R1B50 & Threshold-block & \({0.9422\pm0.0021}\) & \({0.8981\pm0.0042}\) & \({1.916\pm0.118}\) & \({27.0\pm0.4}\)\\
R1B50 & ZB-sign & \({\mathbf{0.9424}\pm0.0025}\) & \({0.8981\pm0.0042}\) & \({\mathbf{1.886}\pm0.126}\) & \({315.0\pm5.0}\)\\
R1B50 & ZB-abs & \({\mathbf{0.9424}\pm0.0030}\) & \({0.8981\pm0.0044}\) & \({1.976\pm0.134}\) & \({315.0\pm5.0}\)\\
R2B50 & Mag-block & \({0.8970\pm0.0024}\) & \({0.9023\pm0.0061}\) & \({4.426\pm0.128}\) & \({\mathbf{12.0}\pm0.2}\)\\
R2B50 & Threshold-block & \({0.8972\pm0.0029}\) & \({0.9022\pm0.0060}\) & \({3.674\pm0.129}\) & \({36.0\pm0.6}\)\\
R2B50 & ZB-sign & \({\mathbf{0.8974}\pm0.0036}\) & \({0.9022\pm0.0061}\) & \({\mathbf{3.584}\pm0.142}\) & \({387.0\pm6.2}\)\\
R2B50 & ZB-abs & \({\mathbf{0.8974}\pm0.0033}\) & \({0.9023\pm0.0065}\) & \({3.814\pm0.146}\) & \({387.0\pm6.2}\)\\
R3B50 & Mag-block & \({0.9210\pm0.0023}\) & \({0.9037\pm0.0065}\) & \({3.695\pm0.145}\) & \({\mathbf{15.0}\pm0.2}\)\\
R3B50 & Threshold-block & \({0.9212\pm0.0036}\) & \({0.9036\pm0.0064}\) & \({\mathbf{3.066}\pm0.131}\) & \({45.0\pm0.7}\)\\
R3B50 & ZB-sign & \({0.9214\pm0.0033}\) & \({0.9037\pm0.0072}\) & \({3.106\pm0.179}\) & \({459.0\pm7.3}\)\\
R3B50 & ZB-abs & \({\mathbf{0.9215}\pm0.0048}\) & \({0.9036\pm0.0070}\) & \({3.176\pm0.117}\) & \({459.0\pm7.3}\)\\
D1B50 & Mag-block & \({0.9470\pm0.0029}\) & \({0.8989\pm0.0064}\) & \({2.213\pm0.101}\) & \({\mathbf{34.2}\pm0.5}\)\\
D1B50 & Threshold-block & \({0.9472\pm0.0041}\) & \({0.8989\pm0.0060}\) & \({1.837\pm0.103}\) & \({102.6\pm1.6}\)\\
D1B50 & ZB-sign & \({\mathbf{0.9474}\pm0.0046}\) & \({0.8989\pm0.0066}\) & \({\mathbf{1.807}\pm0.122}\) & \({3780.0\pm60.4}\)\\
D1B50 & ZB-abs & \({\mathbf{0.9474}\pm0.0043}\) & \({0.8989\pm0.0057}\) & \({1.897\pm0.121}\) & \({3780.0\pm60.4}\)\\
D2B50 & Mag-block & \({0.9080\pm0.0032}\) & \({0.9013\pm0.0066}\) & \({4.368\pm0.174}\) & \({\mathbf{45.6}\pm0.7}\)\\
D2B50 & Threshold-block & \({0.9082\pm0.0037}\) & \({0.9013\pm0.0068}\) & \({3.626\pm0.134}\) & \({136.8\pm2.2}\)\\
D2B50 & ZB-sign & \({0.9083\pm0.0036}\) & \({0.9013\pm0.0073}\) & \({\mathbf{3.536}\pm0.132}\) & \({4644.0\pm74.2}\)\\
D2B50 & ZB-abs & \({\mathbf{0.9084}\pm0.0046}\) & \({0.9014\pm0.0067}\) & \({3.766\pm0.145}\) & \({4644.0\pm74.2}\)\\
D3B50 & Mag-block & \({0.9290\pm0.0033}\) & \({0.8996\pm0.0067}\) & \({3.492\pm0.146}\) & \({\mathbf{57.0}\pm0.9}\)\\
D3B50 & Threshold-block & \({0.9292\pm0.0037}\) & \({0.8996\pm0.0066}\) & \({\mathbf{2.898}\pm0.130}\) & \({171.0\pm2.7}\)\\
D3B50 & ZB-sign & \({\mathbf{0.9294}\pm0.0051}\) & \({0.8996\pm0.0073}\) & \({2.938\pm0.185}\) & \({5508.0\pm88.1}\)\\
D3B50 & ZB-abs & \({\mathbf{0.9294}\pm0.0048}\) & \({0.8996\pm0.0066}\) & \({3.008\pm0.172}\) & \({5508.0\pm88.1}\)\\
\bottomrule
\end{tabular}}
\end{table*}
\begin{table*}[!tbp]
\centering
\scriptsize
\setlength{\tabcolsep}{3pt}
\caption{Top-2 block saliency versus ZB-sign direct search, with 16 seeds and identical 32-block support at 50\% sparsity. ZB-sign uses Zhao and Bellotti's signed finite-deletion criterion. Accuracy, set size, and complete offline time list Top-2/ZB-sign means. Regret is the mean within-seed relative set-size excess, with a paired 95\% Student-\(t\) confidence interval; it is not an optimality gap. Time ratio divides mean ZB-sign time by mean Top-2 time. Accuracy is reported separately, and the comparison establishes no universal equivalence or superiority. Bold indicates the best value per row and metric; \(\uparrow\)/\(\downarrow\) indicate higher/lower is better.}
\label{tab:proxy_efficiency_cost}
\CSPfitTable{\textwidth}{%
\begin{tabular}{lccccc}
\toprule
Cell & Accuracy \(\uparrow\) & Set size \(\downarrow\) & Regret (\%) [95\%] \(\downarrow\) & Total min \(\downarrow\) & Time ratio\\
\midrule
R1B50 & \({\mathbf{0.9426}/0.9424}\) & \({1.958/\mathbf{1.886}}\) & \({4.22\;[0.68,7.76]}\) & \({\mathbf{26.7}/315.0}\) & \({11.8}\)\\
R2B50 & \({\mathbf{0.8983}/0.8974}\) & \({3.692/\mathbf{3.584}}\) & \({3.16\;[0.68,5.63]}\) & \({\mathbf{35.6}/387.0}\) & \({10.9}\)\\
R3B50 & \({\mathbf{0.9217}/0.9214}\) & \({\mathbf{3.062}/3.106}\) & \({-1.01\;[-4.35,2.34]}\) & \({\mathbf{44.5}/459.0}\) & \({10.3}\)\\
D1B50 & \({\mathbf{0.9477}/0.9474}\) & \({1.866/\mathbf{1.807}}\) & \({4.01\;[-0.65,8.67]}\) & \({\mathbf{101.6}/3780.0}\) & \({37.2}\)\\
D2B50 & \({0.9076/\mathbf{0.9083}}\) & \({3.639/\mathbf{3.536}}\) & \({3.06\;[0.37,5.74]}\) & \({\mathbf{135.4}/4644.0}\) & \({34.3}\)\\
D3B50 & \({\mathbf{0.9298}/0.9294}\) & \({\mathbf{2.931}/2.938}\) & \({0.28\;[-3.33,3.88]}\) & \({\mathbf{169.3}/5508.0}\) & \({32.5}\)\\
\bottomrule
\end{tabular}}
\end{table*}

\subsection{Calibration and Encoder Cost}

Calibration-size comparisons hold scores fixed and use nested independent pools. Encoder comparisons vary the gradient budget and account for every runtime phase. Shared task fitting is reported separately from incremental compression cost; full comparisons across architectures require both costs. Figure~\ref{fig:calibration_size} shows the coverage and set-size responses to calibration-pool size, and Figure~\ref{fig:threshold_plateaus} isolates the local plateaus on one score matrix. Figure~\ref{fig:runtime_phases} shows that grid construction and validation selection form the largest runtime phase in every encoder cell.


\begin{figure*}[!tbp]
\centering
\includegraphics[width=\textwidth]{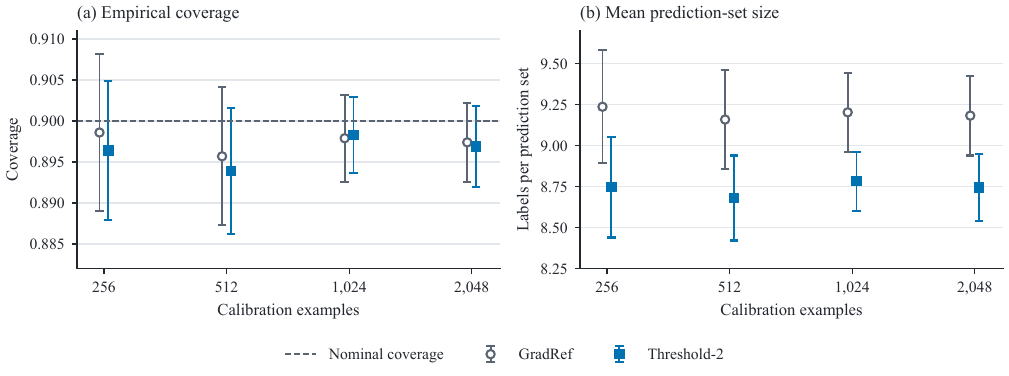}
\caption{Calibration-size sensitivity for the fixed zero-shot Qwen scorer on DBpedia-14 with Wanda at 50\% sparsity (Q1W50), using 16 fixed score matrices and nested independent calibration pools. Points show mean coverage (a) and mean set size (b), with 95\% Student-\(t\) confidence intervals. \GradRef{} is the absolute-gradient reference; Threshold-2 uses two candidate labels. Horizontal offsets separate methods at the same pool size. The dashed line marks nominal 90\% coverage. The 1,024-example operating point reuses the Q1W50 aggregate; it is not an additional experiment. Only the calibration prefix changes, so the sweep measures threshold sensitivity rather than model changes.}
\label{fig:calibration_size}
\end{figure*}

\begin{figure}[!tbp]
\centering
\includegraphics[width=\columnwidth]{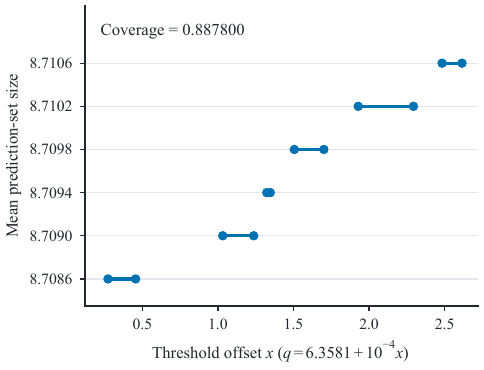}
\caption{Local threshold plateaus on one fixed 5,000-example reference score matrix. Each horizontal segment joins two observed thresholds within the same score gap; gaps between segments leave the unmeasured transition locations unspecified. Mean set size rises in increments of \(0.0004=2/5000\) across the reported plateaus, while coverage stays at 0.887800. The horizontal axis expresses threshold \(q\) through the offset \(x\), where \(q=6.3581+10^{-4}x\). These are fixed-matrix readouts, with no inferential intervals.}
\label{fig:threshold_plateaus}
\end{figure}
\begin{table*}[!tbp]
\centering
\scriptsize
\setlength{\tabcolsep}{3pt}
\caption{Encoder saliency-budget sweep at 50\% sparsity without recovery, with 32 paired seeds. Top-\(k\) selects the \(k\) lowest-scoring labels. Example and label counts vary separately, with the 256-example/two-label default shared by both sweeps. Settings share a frozen task checkpoint and use nested subsets of the 256-example pruning pool. Entries are means with 95\% Student-\(t\) confidence-interval half-widths. Bold indicates the best setting per cell and column; \(\uparrow\)/\(\downarrow\) indicate higher/lower is better. Coverage is not ranked.}
\label{tab:encoder_saliency_budget}
\CSPfitTable{\textwidth}{%
\begin{tabular}{lrrccccc}
\toprule
Cell & Examples & Labels & Accuracy \(\uparrow\) & Coverage & Size \(\downarrow\) & Gradient min \(\downarrow\) & Total min \(\downarrow\)\\
\midrule
R1W50 & 64 & 2 & \({\mathbf{0.9509}\pm0.0027}\) & \({0.8995\pm0.0037}\) & \({1.862\pm0.092}\) & \({\mathbf{0.53}\pm0.01}\) & \({\mathbf{29.2}\pm0.3}\)\\
R1W50 & 128 & 2 & \({0.9497\pm0.0023}\) & \({0.8999\pm0.0039}\) & \({1.774\pm0.085}\) & \({1.05\pm0.02}\) & \({29.7\pm0.3}\)\\
R1W50 & 256 & 1 & \({0.9508\pm0.0024}\) & \({0.8989\pm0.0037}\) & \({1.815\pm0.087}\) & \({1.05\pm0.02}\) & \({29.6\pm0.3}\)\\
R1W50 & 256 & 2 & \({0.9503\pm0.0022}\) & \({0.8999\pm0.0034}\) & \({1.703\pm0.081}\) & \({2.11\pm0.05}\) & \({30.7\pm0.3}\)\\
R1W50 & 256 & 4 & \({0.9499\pm0.0024}\) & \({0.8993\pm0.0037}\) & \({1.673\pm0.078}\) & \({4.20\pm0.08}\) & \({33.0\pm0.3}\)\\
R1W50 & 256 & 8 & \({0.9505\pm0.0021}\) & \({0.9004\pm0.0038}\) & \({\mathbf{1.669}\pm0.082}\) & \({8.39\pm0.18}\) & \({37.6\pm0.3}\)\\
R1S50 & 64 & 2 & \({0.9534\pm0.0027}\) & \({0.8987\pm0.0043}\) & \({1.706\pm0.065}\) & \({\mathbf{0.53}\pm0.01}\) & \({\mathbf{47.1}\pm0.4}\)\\
R1S50 & 128 & 2 & \({\mathbf{0.9535}\pm0.0026}\) & \({0.8995\pm0.0039}\) & \({1.654\pm0.064}\) & \({1.05\pm0.02}\) & \({47.6\pm0.4}\)\\
R1S50 & 256 & 1 & \({0.9533\pm0.0026}\) & \({0.8996\pm0.0042}\) & \({1.678\pm0.063}\) & \({1.06\pm0.02}\) & \({47.5\pm0.4}\)\\
R1S50 & 256 & 2 & \({0.9530\pm0.0025}\) & \({0.8989\pm0.0039}\) & \({1.579\pm0.059}\) & \({2.07\pm0.03}\) & \({48.7\pm0.4}\)\\
R1S50 & 256 & 4 & \({0.9525\pm0.0025}\) & \({0.8984\pm0.0041}\) & \({1.549\pm0.057}\) & \({4.20\pm0.06}\) & \({51.0\pm0.4}\)\\
R1S50 & 256 & 8 & \({0.9534\pm0.0024}\) & \({0.8997\pm0.0043}\) & \({\mathbf{1.541}\pm0.061}\) & \({8.50\pm0.14}\) & \({55.6\pm0.4}\)\\
R2W50 & 64 & 2 & \({0.9072\pm0.0030}\) & \({0.9024\pm0.0038}\) & \({3.763\pm0.131}\) & \({\mathbf{0.58}\pm0.01}\) & \({\mathbf{47.9}\pm0.4}\)\\
R2W50 & 128 & 2 & \({0.9073\pm0.0027}\) & \({0.9012\pm0.0036}\) & \({3.622\pm0.124}\) & \({1.18\pm0.02}\) & \({48.4\pm0.4}\)\\
R2W50 & 256 & 1 & \({\mathbf{0.9079}\pm0.0028}\) & \({0.9009\pm0.0040}\) & \({3.696\pm0.129}\) & \({1.15\pm0.02}\) & \({48.3\pm0.4}\)\\
R2W50 & 256 & 2 & \({0.9075\pm0.0026}\) & \({0.9014\pm0.0036}\) & \({3.471\pm0.114}\) & \({2.27\pm0.04}\) & \({49.6\pm0.4}\)\\
R2W50 & 256 & 4 & \({0.9078\pm0.0026}\) & \({0.9014\pm0.0039}\) & \({3.420\pm0.113}\) & \({4.58\pm0.09}\) & \({52.1\pm0.4}\)\\
R2W50 & 256 & 8 & \({0.9071\pm0.0024}\) & \({0.9013\pm0.0034}\) & \({\mathbf{3.390}\pm0.110}\) & \({9.19\pm0.15}\) & \({57.1\pm0.4}\)\\
R2S50 & 64 & 2 & \({0.9139\pm0.0029}\) & \({0.9025\pm0.0037}\) & \({3.371\pm0.155}\) & \({\mathbf{0.59}\pm0.01}\) & \({\mathbf{72.8}\pm0.7}\)\\
R2S50 & 128 & 2 & \({0.9150\pm0.0028}\) & \({0.9011\pm0.0032}\) & \({3.209\pm0.151}\) & \({1.13\pm0.02}\) & \({73.4\pm0.7}\)\\
R2S50 & 256 & 1 & \({\mathbf{0.9152}\pm0.0028}\) & \({0.9014\pm0.0035}\) & \({3.273\pm0.154}\) & \({1.14\pm0.02}\) & \({73.3\pm0.7}\)\\
R2S50 & 256 & 2 & \({0.9147\pm0.0028}\) & \({0.9016\pm0.0035}\) & \({3.076\pm0.138}\) & \({2.31\pm0.03}\) & \({74.6\pm0.7}\)\\
R2S50 & 256 & 4 & \({0.9141\pm0.0026}\) & \({0.9017\pm0.0037}\) & \({3.025\pm0.132}\) & \({4.53\pm0.06}\) & \({77.1\pm0.7}\)\\
R2S50 & 256 & 8 & \({0.9144\pm0.0028}\) & \({0.9018\pm0.0035}\) & \({\mathbf{2.998}\pm0.137}\) & \({9.08\pm0.15}\) & \({82.1\pm0.7}\)\\
R3W50 & 64 & 2 & \({0.9304\pm0.0035}\) & \({0.9031\pm0.0042}\) & \({3.059\pm0.129}\) & \({\mathbf{0.62}\pm0.01}\) & \({\mathbf{63.6}\pm0.5}\)\\
R3W50 & 128 & 2 & \({0.9301\pm0.0032}\) & \({0.9028\pm0.0041}\) & \({2.927\pm0.127}\) & \({1.25\pm0.02}\) & \({64.3\pm0.5}\)\\
R3W50 & 256 & 1 & \({0.9306\pm0.0030}\) & \({0.9031\pm0.0045}\) & \({3.000\pm0.122}\) & \({1.26\pm0.02}\) & \({64.2\pm0.5}\)\\
R3W50 & 256 & 2 & \({0.9301\pm0.0029}\) & \({0.9028\pm0.0041}\) & \({2.810\pm0.118}\) & \({2.49\pm0.04}\) & \({65.5\pm0.5}\)\\
R3W50 & 256 & 4 & \({0.9310\pm0.0032}\) & \({0.9018\pm0.0040}\) & \({2.777\pm0.117}\) & \({5.01\pm0.12}\) & \({68.2\pm0.5}\)\\
R3W50 & 256 & 8 & \({\mathbf{0.9311}\pm0.0034}\) & \({0.9025\pm0.0040}\) & \({\mathbf{2.746}\pm0.121}\) & \({9.99\pm0.20}\) & \({73.6\pm0.5}\)\\
R3S50 & 64 & 2 & \({0.9379\pm0.0033}\) & \({0.9005\pm0.0048}\) & \({2.636\pm0.116}\) & \({\mathbf{0.63}\pm0.01}\) & \({\mathbf{100.0}\pm1.0}\)\\
R3S50 & 128 & 2 & \({0.9374\pm0.0028}\) & \({0.8994\pm0.0044}\) & \({2.537\pm0.119}\) & \({1.25\pm0.02}\) & \({100.6\pm1.0}\)\\
R3S50 & 256 & 1 & \({0.9374\pm0.0030}\) & \({0.9000\pm0.0045}\) & \({2.594\pm0.117}\) & \({1.26\pm0.02}\) & \({100.5\pm1.0}\)\\
R3S50 & 256 & 2 & \({\mathbf{0.9381}\pm0.0027}\) & \({0.8997\pm0.0043}\) & \({2.423\pm0.112}\) & \({2.46\pm0.04}\) & \({101.8\pm1.0}\)\\
R3S50 & 256 & 4 & \({0.9372\pm0.0025}\) & \({0.8997\pm0.0045}\) & \({2.381\pm0.109}\) & \({5.03\pm0.09}\) & \({104.5\pm1.0}\)\\
R3S50 & 256 & 8 & \({0.9376\pm0.0032}\) & \({0.9001\pm0.0042}\) & \({\mathbf{2.359}\pm0.109}\) & \({10.11\pm0.17}\) & \({109.9\pm1.0}\)\\
\bottomrule
\end{tabular}}
\end{table*}

\begin{figure*}[!tbp]
\centering
\includegraphics[width=\textwidth]{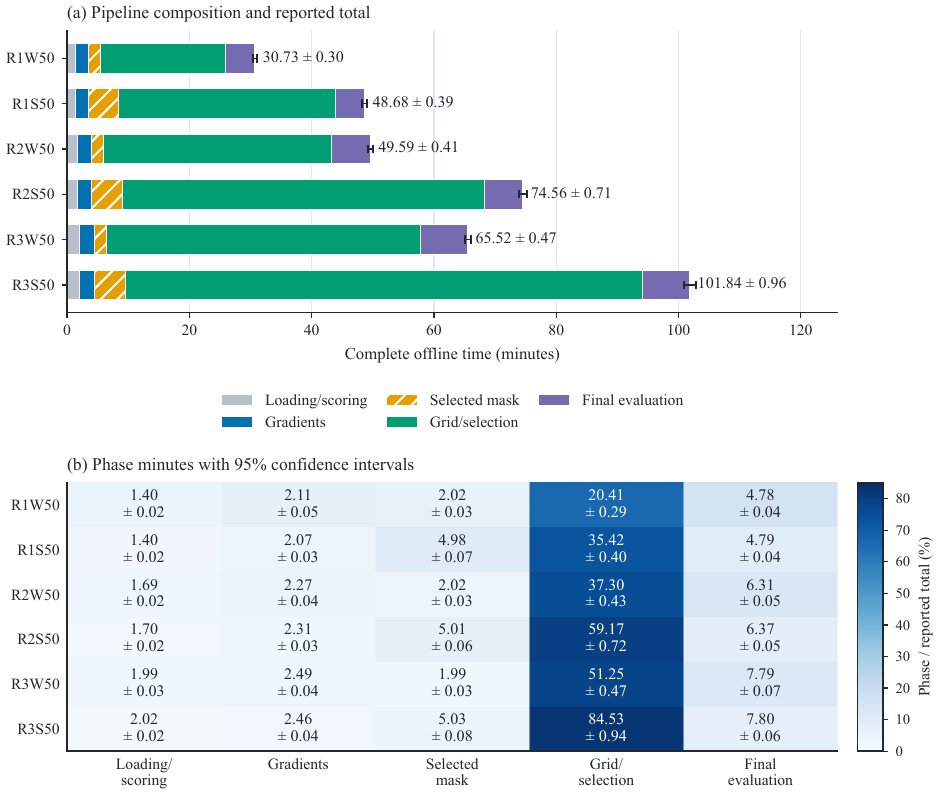}
\caption{Top-2 encoder runtime by phase at 50\% sparsity, with 256 pruning examples and 32 seeds. (a) Stacked phase means and separately reported total minutes with 95\% Student-\(t\) confidence intervals. (b) Each cell reports phase minutes and its 95\% interval half-width; shading shows the phase mean as a percentage of the reported total. Loading/scoring includes checkpoint loading and dense scoring; selected mask constructs the chosen sparse model; grid/selection includes remaining candidates, provisional calibration, validation, and selection; final evaluation includes independent calibration and testing. Phases partition the pipeline; rounding can make their displayed sum differ slightly from the reported total. R1/R2/R3 denote DBpedia-14/Banking77/CLINC150 on RoBERTa; W/S denote Wanda/SparseGPT. Shared dense task fitting is reported separately.}
\label{fig:runtime_phases}
\end{figure*}
\begin{table*}[!tbp]\centering\begin{minipage}{0.48\textwidth}
\centering
\scriptsize
\setlength{\tabcolsep}{3pt}
\caption{Dense task-fitting time in minutes over 16 seeds on one A100 80GB. Full model-development cost adds this task-specific time to the incremental pruning pipeline. It is shared only when methods reuse the same trained checkpoint.}
\label{tab:task_training_cost}
\begin{tabular}{lc}
\toprule
Regime & Training min\\
\midrule
R1 & \({11.00\pm0.15}\)\\
R2 & \({4.30\pm0.06}\)\\
R3 & \({7.90\pm0.11}\)\\
D1 & \({84.00\pm1.12}\)\\
D2 & \({31.00\pm0.41}\)\\
D3 & \({59.00\pm0.79}\)\\
\bottomrule
\end{tabular}
\end{minipage}\end{table*}

\section{Reported Aggregate Controls}
\label{app:aggregate_controls}

The following tables report descriptive means and sample counts from the protocols in Appendix~\ref{app:reported_ablations}. They evaluate true-label or variable-budget saliency rather than powered Top-2. Repeated summaries and cell-direction counts provide descriptive accounting, not independent replication or paired significance tests.

\begin{table*}[!tbp]
\centering
\footnotesize
\setlength{\tabcolsep}{4.2pt}
\caption{Qwen2.5-1.5B results at 50\% sparsity, reported as descriptive means over five seeds. Coverage is the validity constraint; set sizes should be compared at matched coverage. Bold indicates the better value within each base/TL-CEP pair.}
\label{tab:main_required}
\begin{tabular}{llccc}
\toprule
Dataset & Method & Accuracy \(\uparrow\) & Coverage & Avg.\ set size \(\downarrow\)\\
\midrule
\multirow{4}{*}{DBpedia-14}
& Wanda & \(\mathbf{0.310}\) & \(0.901\) & \(11.2\)\\
& TL-CEP-Wanda & \(0.295\) & \(0.902\) & \(\mathbf{9.0}\)\\
& SparseGPT & \(0.347\) & \(0.901\) & \(10.1\)\\
& TL-CEP-SparseGPT & \(\mathbf{0.366}\) & \(0.902\) & \(\mathbf{8.6}\)\\
\midrule
\multirow{4}{*}{Banking77}
& Wanda & \(0.282\) & \(0.900\) & \(39.8\)\\
& TL-CEP-Wanda & \(\mathbf{0.316}\) & \(0.902\) & \(\mathbf{34.2}\)\\
& SparseGPT & \(0.327\) & \(0.901\) & \(36.7\)\\
& TL-CEP-SparseGPT & \(\mathbf{0.351}\) & \(0.902\) & \(\mathbf{32.1}\)\\
\midrule
\multirow{4}{*}{CLINC150}
& Wanda & \(0.247\) & \(0.900\) & \(73.4\)\\
& TL-CEP-Wanda & \(\mathbf{0.286}\) & \(0.901\) & \(\mathbf{64.8}\)\\
& SparseGPT & \(0.292\) & \(0.901\) & \(68.9\)\\
& TL-CEP-SparseGPT & \(\mathbf{0.326}\) & \(0.902\) & \(\mathbf{61.7}\)\\
\bottomrule
\end{tabular}
\end{table*}
\begin{table*}[!tbp]\centering
\begin{minipage}[t]{0.48\textwidth}\vspace{0pt}
\centering
\scriptsize
\setlength{\tabcolsep}{2.8pt}
\caption{Gradient-information controls, reported as descriptive means over three seeds. Acc., Cov., and \(|C|\) denote accuracy, coverage, and mean set size. SNIP-style controls are inspired by single-shot network pruning (SNIP). Bold indicates the best value per column.}
\label{tab:gradient_controls}
\CSPfitTable{\linewidth}{%
\begin{tabular}{lccc}
\toprule
Method & Acc.\(\uparrow\) & Cov. & \(|C|\downarrow\)\\
\midrule
Wanda & \(0.310\) & \(0.901\) & \(11.2\)\\
Wanda++ & \(0.335\) & \(0.900\) & \(9.8\)\\
Gradient-only & \(0.323\) & \(0.899\) & \(10.1\)\\
SNIP-style & \(0.341\) & \(0.901\) & \(9.6\)\\
Wanda\(\oplus\)SNIP & \(\mathbf{0.352}\) & \(0.901\) & \(9.2\)\\
TL-CEP-Wanda & \(0.297\) & \(0.902\) & \(9.0\)\\
Band-CEP-Wanda & \(0.319\) & \(0.901\) & \(\mathbf{8.4}\)\\
\bottomrule
\end{tabular}}
\end{minipage}\hfill
\begin{minipage}[t]{0.48\textwidth}\vspace{0pt}
\centering
\scriptsize
\setlength{\tabcolsep}{2.2pt}
\caption{Candidate-label saliency with SparseGPT on DBpedia-14 at 50\% sparsity. Entries are descriptive means; Seeds gives the sample count for each arm. Acc., Cov., and \(|C|\) denote accuracy, coverage, and mean set size. Cost is saliency work relative to TL-CEP. Bold indicates the best value per column.}
\label{tab:candidate_saliency}
\CSPfitTable{\linewidth}{%
\begin{tabular}{lcccc}
\toprule
Construction & Seeds & Acc.\(\uparrow\) & Cov. & \(|C|\downarrow\) / cost\(\downarrow\)\\
\midrule
TL-CEP & 5 & \(0.366\) & \(0.902\) & \(8.6\) / \(\mathbf{1.0}\times\)\\
Top-\(3\) & 3 & \(0.369\) & \(0.901\) & \(8.1\) / \(2.4\times\)\\
Band-CEP & 3 & \(\mathbf{0.372}\) & \(0.901\) & \(7.8\) / \(2.7\times\)\\
All labels & 3 & \(0.370\) & \(0.902\) & \(\mathbf{7.7}\) / \(8.9\times\)\\
\bottomrule
\end{tabular}}
\end{minipage}
\end{table*}
\begin{table*}[!tbp]
\centering
\footnotesize
\caption{Offline pruning time and peak memory on one NVIDIA A100 80GB. Cost ratios use each method's own base unless another reference is stated. Bold indicates the lowest value per base group and column; \(\downarrow\) indicates lower is better.}
\label{tab:compute_cost}
\begin{tabular}{lccc}
\toprule
Method & Wall time \(\downarrow\) & Relative cost \(\downarrow\) & Peak memory \(\downarrow\)\\
\midrule
Wanda & \textbf{7.8} min & \(\mathbf{1.0}\times\) & \textbf{9.4} GB\\
Wanda++ & 15.1 min & \(1.9\times\) & 11.2 GB\\
TL-CEP-Wanda & 31.4 min & \(4.0\times\) & 12.6 GB\\
Band-CEP-Wanda & 68.7 min & \(8.8\times\) Wanda; \(2.2\times\) TL-CEP & 14.8 GB\\
SparseGPT & \textbf{24.6} min & \(\mathbf{1.0}\times\) & \textbf{13.8} GB\\
TL-CEP-SparseGPT & 47.9 min & \(1.9\times\) & 15.1 GB\\
Llama-3-8B SparseGPT & \textbf{2.1} h & \(\mathbf{1.0}\times\) & \textbf{37} GB\\
Llama-3-8B TL-CEP-SparseGPT & 4.8 h & \(2.3\times\) & 43 GB\\
\bottomrule
\end{tabular}
\end{table*}
\begin{table*}[!tbp]\centering\begin{minipage}{0.48\textwidth}
\centering
\footnotesize
\caption{Transfer to RoBERTa-base on DBpedia-14 at 50\% sparsity, with three seeds. Bold indicates the better pruned value; Dense provides the unpruned reference.}
\label{tab:roberta_transfer}
\CSPfitTable{\linewidth}{%
\begin{tabular}{lccc}
\toprule
Method & Accuracy \(\uparrow\) & Coverage & Set size \(\downarrow\)\\
\midrule
Dense & \(0.963\) & \(0.902\) & \(1.43\)\\
Magnitude & \(0.934\) & \(0.900\) & \(2.31\)\\
TL-CEP-Magnitude & \(\mathbf{0.946}\) & \(0.901\) & \(\mathbf{1.88}\)\\
\bottomrule
\end{tabular}}
\end{minipage}\end{table*}
\begin{table*}[!tbp]
\centering
\footnotesize
\setlength{\tabcolsep}{3.5pt}
\caption{Qwen2.5-1.5B results on AG News~\citep{zhang2015textclassification} and TREC~\citep{li2002trec}, reported as means over five seeds. Each entry lists empirical coverage followed by mean set size. No inferential interval is assigned. Bold indicates the smaller set size within each base/TL-CEP pair; \(\downarrow\) indicates lower is better. Coverage is not ranked.}
\label{tab:pilot_multi_extended}
\begin{tabular}{llcccc}
\toprule
Dataset & Sparsity & Wanda & TL-CEP-Wanda & SparseGPT & TL-CEP-SparseGPT\\
\midrule
AG News & 30\% & \(0.901,\ 3.71\) & \(0.902,\ \mathbf{3.46}\) & \(0.900,\ 3.62\) & \(0.902,\ \mathbf{3.38}\)\\
AG News & 50\% & \(0.900,\ 3.86\) & \(0.902,\ \mathbf{3.52}\) & \(0.901,\ 3.77\) & \(0.902,\ \mathbf{3.41}\)\\
AG News & 70\% & \(0.901,\ 3.89\) & \(0.901,\ \mathbf{3.74}\) & \(0.900,\ 3.88\) & \(0.901,\ \mathbf{3.67}\)\\
TREC & 30\% & \(0.901,\ 5.34\) & \(0.902,\ \mathbf{4.92}\) & \(0.900,\ 5.17\) & \(0.902,\ \mathbf{4.78}\)\\
TREC & 50\% & \(0.900,\ 5.11\) & \(0.902,\ \mathbf{4.63}\) & \(0.901,\ 4.91\) & \(0.902,\ \mathbf{4.52}\)\\
TREC & 70\% & \(0.900,\ 5.61\) & \(0.901,\ \mathbf{5.33}\) & \(0.900,\ \mathbf{5.42}\) & \(0.901,\ 5.48\)\\
\bottomrule
\end{tabular}
\end{table*}
\begin{table*}[!tbp]\centering
\begin{minipage}[t]{0.48\textwidth}\vspace{0pt}
\centering
\footnotesize
\caption{TL-CEP changes relative to the matched base at 50\% sparsity. Set-size changes are relative percentages; accuracy changes are absolute differences. \(\uparrow\)/\(\downarrow\): higher/lower is better.}
\label{tab:conformal_detail}
\CSPfitTable{\linewidth}{%
\begin{tabular}{llrr}
\toprule
Dataset & Variant & Set-size change \(\downarrow\) & Acc.\ change \(\uparrow\)\\
\midrule
DBpedia & TL-CEP-Wanda & \(-19.6\%\) & \(-0.015\)\\
DBpedia & TL-CEP-SparseGPT & \(-14.9\%\) & \(+0.019\)\\
Banking77 & TL-CEP-Wanda & \(-14.1\%\) & \(+0.034\)\\
Banking77 & TL-CEP-SparseGPT & \(-12.5\%\) & \(+0.024\)\\
CLINC150 & TL-CEP-Wanda & \(-11.7\%\) & \(+0.039\)\\
CLINC150 & TL-CEP-SparseGPT & \(-10.4\%\) & \(+0.034\)\\
\bottomrule
\end{tabular}}
\end{minipage}\hfill
\begin{minipage}[t]{0.48\textwidth}\vspace{0pt}
\centering
\footnotesize
\caption{Qwen results on DBpedia-14 at 50\% sparsity: accuracy, coverage, and mean set size. Bold indicates the better value within each base/TL-CEP pair. Coverage is not ranked.}
\label{tab:qwen_dbpedia_sparsity_backing}
\begin{tabular}{lccc}
\toprule
Method & Accuracy \(\uparrow\) & Coverage & Set size \(\downarrow\)\\
\midrule
Wanda & \(\mathbf{0.310}\) & \(0.901\) & \(11.2\)\\
TL-CEP-Wanda & \(0.295\) & \(0.902\) & \(\mathbf{9.0}\)\\
SparseGPT & \(0.347\) & \(0.901\) & \(10.1\)\\
TL-CEP-SparseGPT & \(\mathbf{0.366}\) & \(0.902\) & \(\mathbf{8.6}\)\\
\bottomrule
\end{tabular}
\end{minipage}
\end{table*}
\begin{table*}[!tbp]\centering
\begin{minipage}[t]{0.48\textwidth}\vspace{0pt}
\centering
\footnotesize
\caption{Directional wins for TL-CEP against its matched base across 15 dataset/sparsity cells. A win means higher accuracy or smaller mean set size. Rates summarize cell directions, not significance tests; \(\uparrow\) indicates higher is better.}
\label{tab:winrate_required}
\begin{tabular}{lcc}
\toprule
Comparison & Wins / 15 \(\uparrow\) & Rate \(\uparrow\)\\
\midrule
TL-CEP-Wanda accuracy & 12 & 80.0\%\\
TL-CEP-SparseGPT accuracy & 11 & 73.3\%\\
TL-CEP-Wanda set size & 14 & 93.3\%\\
TL-CEP-SparseGPT set size & 13 & 86.7\%\\
\bottomrule
\end{tabular}
\end{minipage}\hfill
\begin{minipage}[t]{0.48\textwidth}\vspace{0pt}
\centering
\footnotesize
\caption{Cell counts underlying the win rates in Table~\ref{tab:winrate_required}. Den. counts dataset/sparsity cells, not independent repetitions.}
\label{tab:aggregate_accounting}
\begin{tabular}{lcl}
\toprule
Metric & Den. & Cells\\
\midrule
Accuracy & 15 & Five datasets \(\times\) three sparsities\\
Set size & 15 & Five datasets \(\times\) three sparsities\\
\bottomrule
\end{tabular}
\end{minipage}
\end{table*}
\begin{table*}[!tbp]
\centering
\footnotesize
\caption{Direction of TL-CEP differences from the matched base in Qwen cell means. W/L indicate favorable/unfavorable differences and do not imply significance. In headers, W/SG denote Wanda/SparseGPT, Acc. denotes accuracy, and \(|C|\) denotes mean set size.}
\label{tab:aggregate_cell_signs}
\begin{tabular}{llcccc}
\toprule
Dataset & Sparsity & TL-CEP-W Acc & TL-CEP-SG Acc & TL-CEP-W \(|C|\) & TL-CEP-SG \(|C|\)\\
\midrule
AG News & 30\% & W & L & W & W\\
AG News & 50\% & W & W & W & W\\
AG News & 70\% & L & W & W & W\\
TREC & 30\% & W & L & W & W\\
TREC & 50\% & L & W & W & W\\
TREC & 70\% & W & W & W & L\\
DBpedia-14 & 30\% & W & W & W & W\\
DBpedia-14 & 50\% & L & W & W & W\\
DBpedia-14 & 70\% & W & W & W & W\\
Banking77 & 30\% & W & W & W & W\\
Banking77 & 50\% & W & W & W & W\\
Banking77 & 70\% & W & L & W & L\\
CLINC150 & 30\% & W & W & W & W\\
CLINC150 & 50\% & W & W & W & W\\
CLINC150 & 70\% & W & L & L & W\\
\bottomrule
\end{tabular}
\end{table*}
\begin{table*}[!tbp]
\centering
\footnotesize
\caption{Qwen results on DBpedia-14 at 50\% sparsity under five independent split redraws. Bold indicates the better value per column. Coverage is not ranked.}
\label{tab:split_redraw}

\begin{tabular}{lccc}
\toprule
Method & Accuracy \(\uparrow\) & Coverage & Set size \(\downarrow\)\\
\midrule
SparseGPT & \(0.3492\) & \(0.9004\) & \(10.27\)\\
TL-CEP-SparseGPT & \(\mathbf{0.3638}\) & \(0.9026\) & \(\mathbf{8.79}\)\\
\bottomrule
\end{tabular}
\end{table*}
\begin{table*}[!tbp]
\centering
\footnotesize
\caption{Effect of pruning-pool size using nested subsets. Each entry lists accuracy, coverage, and mean set size, in that order; \(n_{\mathrm{prune}}\) is the number of pruning examples. Bold indicates the better value per row and metric; \(\uparrow\)/\(\downarrow\) indicate higher/lower is better. Coverage is not ranked.}
\label{tab:prune_size_sweep}

\begin{tabular}{rcc}
\toprule
\(n_{\mathrm{prune}}\) & SparseGPT & TL-CEP-SparseGPT\\
\midrule
128 & \(0.341/0.900/10.2\) & \(\mathbf{0.354}/0.901/\mathbf{9.4}\)\\
256 & \(0.344/0.900/10.1\) & \(\mathbf{0.359}/0.901/\mathbf{9.1}\)\\
512 & \(0.346/0.901/10.1\) & \(\mathbf{0.363}/0.902/\mathbf{8.8}\)\\
1024 & \(0.347/0.901/10.1\) & \(\mathbf{0.366}/0.902/\mathbf{8.6}\)\\
\bottomrule
\end{tabular}
\end{table*}
\begin{table*}[!tbp]
\caption{Llama-3-8B~\citep{dubey2024llama3} results at 50\% sparsity. Perplexity (PPL) is measured on WikiText-2~\citep{merity2017wikitext}; accuracy (Acc.), coverage (Cov.), and mean set size (\(|C|\)) are measured on DBpedia-14. Sparse results are descriptive means over three pruning seeds. Dense perplexity is a single point value. Bold indicates the better sparse value; Dense provides the unpruned reference.}
\label{tab:llama_results}
\centering
\footnotesize

\begin{tabular}{lcccc}
\toprule
Method & PPL $\downarrow$ & Acc. $\uparrow$ & Cov. & $|C|$ $\downarrow$ \\
\midrule
Dense & $6.12 $ & $.78 $ & $.902 $ & $4.5 $ \\
SparseGPT & $\mathbf{7.25} $ & $.61 $ & $.901 $ & $8.9 $ \\
TL-CEP-SparseGPT & $7.28 $ & $\mathbf{.65} $ & $.902 $ & $\mathbf{7.2} $ \\
\bottomrule
\end{tabular}
\end{table*}

\end{document}